\documentclass{article}

\usepackage{microtype}
\usepackage{booktabs} 

\usepackage{hyperref}

\usepackage[accepted]{icml2024}

\usepackage{amsmath,amsfonts,bm}

\def\eqref#1{equation~\ref{#1}}

\def\1{\bm{1}}

\newcommand{\maximize}{\mathop{\mathbf{max}}}
\newcommand{\minimize}{\mathop{\mathbf{min}}}
\newcommand{\maximizewrt}[1]{\mathop{\underset{#1}{\maximize}}}
\newcommand{\minimizewrt}[1]{\mathop{\underset{#1}{\minimize}}}

\DeclareMathAlphabet{\mathsfit}{\encodingdefault}{\sfdefault}{m}{sl}
\SetMathAlphabet{\mathsfit}{bold}{\encodingdefault}{\sfdefault}{bx}{n}

\DeclareMathOperator*{\argmax}{arg\,max}

\usepackage[utf8]{inputenc} 
\usepackage[T1]{fontenc}    
\usepackage{hyperref}       
\usepackage{url}            
\usepackage{booktabs}       
\usepackage{amsfonts}       
\usepackage{nicefrac}       
\usepackage{microtype}      
\usepackage{xcolor}         
\usepackage{url}
\usepackage{amsmath}
\usepackage{amssymb}
\usepackage{mathtools}
\usepackage{amsthm}
\usepackage[font={footnotesize}]{subcaption}
\usepackage[font={footnotesize}]{caption}
\usepackage{rotating}
\usepackage{colortbl}
\usepackage{wrapfig}
\usepackage{amsfonts}
\usepackage{pbox}
\usepackage{pifont}
\usepackage{bm}
\usepackage{multirow}
\usepackage{algorithm}
\usepackage{graphicx}

\usepackage[capitalize,noabbrev]{cleveref}

\theoremstyle{plain}
\newtheorem{theo}{Theorem}
\newtheorem{prop}{Proposition}
\newtheorem{lem}{Lemma}

\theoremstyle{definition}
\newtheorem{defi}{Definition}

\theoremstyle{remark}
\newtheorem{rmk}{Remark}

\usepackage[textsize=tiny]{todonotes}

\icmltitlerunning{Absolute Policy Optimization}

\begin{document}
\twocolumn[
\icmltitle{Absolute Policy Optimization: Enhancing Lower Probability Bound of Performance with High Confidence}



\icmlsetsymbol{equal}{*}

\begin{icmlauthorlist}
\icmlauthor{Weiye Zhao}{equal,yyy}
\icmlauthor{Feihan Li}{equal,yyy}
\icmlauthor{Yifan Sun}{yyy}
\icmlauthor{Rui Chen}{yyy}
\icmlauthor{Tianhao Wei}{yyy}
\icmlauthor{Changliu Liu}{yyy}
\end{icmlauthorlist}

\icmlaffiliation{yyy}{Robotics Institute, Carnegie Mellon University, USA}

\icmlcorrespondingauthor{Weiye Zhao}{weiyezha@andrew.cmu.edu}
\icmlcorrespondingauthor{Changliu Liu}{cliu6@andrew.cmu.edu}

\icmlkeywords{Machine Learning, ICML}

\vskip 0.3in
]



\printAffiliationsAndNotice{\icmlEqualContribution} 

\begin{abstract}
In recent years, trust region on-policy reinforcement learning has achieved impressive results in addressing complex control tasks and gaming scenarios. However, contemporary state-of-the-art algorithms within this category primarily emphasize improvement in expected performance, lacking the ability to control over the worst-case performance outcomes.
To address this limitation, we introduce a novel objective function, optimizing which leads to guaranteed monotonic improvement in the 
lower probability bound of performance with high confidence. 
Building upon this groundbreaking theoretical advancement, we further introduce a practical solution called Absolute Policy Optimization (APO). 
Our experiments demonstrate the effectiveness of our approach across challenging continuous control benchmark tasks and extend its applicability to mastering Atari games. Our findings reveal that APO as well as its efficient variation Proximal Absolute Policy Optimization (PAPO) significantly outperforms state-of-the-art policy gradient algorithms, resulting in substantial improvements in worst-case performance, as well as expected performance. 
\end{abstract}

\section{Introduction}

Existing reinforcement learning algorithms focus on improving expected cumulative rewards (referred to as \textbf{performance}). 
Within this framework, trust region-based on-policy reinforcement learning algorithms have achieved the most promising results. However, 
the representative trust-region policy optimization (TRPO)~\citep{schulman2015trust} only ensures the monotonic improvement of the expected performance, it fails to exert control over the worst-case performance.
Simultaneously, the deployment of reinforcement learning policies in real-world scenarios demands a high level of consistency, where the performance distribution must be carefully controlled to ensure that even under worst-case conditions, undesirable outcomes are avoided. This is particularly critical in applications such as autonomous driving and intelligent robot manipulation, where robust performance is essential to guarantee safety, reliability, and adherence to desired behaviors.

\begin{figure}[t]
    
    \raisebox{-\height}{\includegraphics[width=\linewidth]{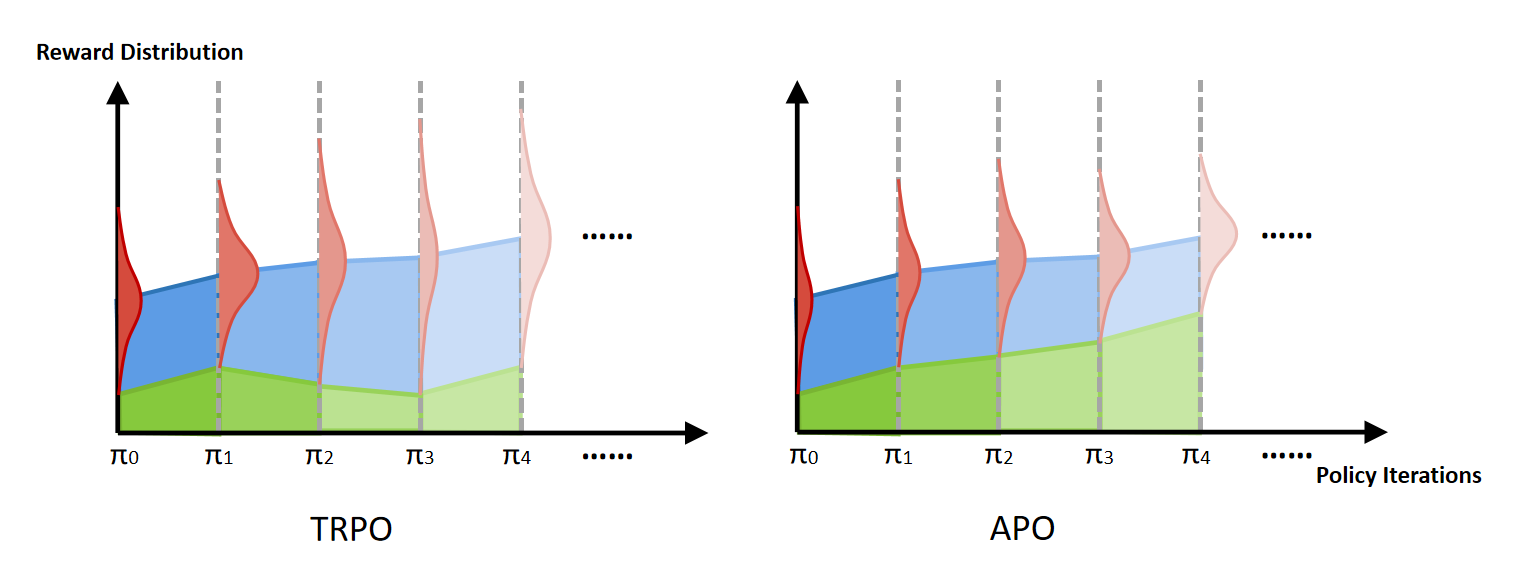}}

    \captionsetup{width=0.8\linewidth}
    \caption{Explanation of APO principles. Here we simplify the distribution of rewards to a Gaussian distribution for displaying. Green and blue represent the worst case and expectation respectively. Y-axis represents the distribution of performance samples introduced in \Cref{sec: absolute bound}. In comparison to TRPO, APO is designed to consistently elevate the lower bound of overall performance.} 
    \label{fig: trpo vs apo}
\vspace{-15pt}
\end{figure}

Distribution control represents a prominent research area across various domains. Existing works in this field can be categorized into three main groups: (i) modeling the entire performance distribution, (ii) controlling worst-case costs, and (iii) mitigating model discrepancies during generalization.
Many \textbf{distributional RL} methods~\citep{marc2017c51, will2017qrdqn} fall into the first category, aiming to model the entire distribution of returns~\cite{will2018iqn}. By incorporating more information about the distribution of rewards into policy gradients, these methods contribute to more stable and efficient learning~\citep{derek2019fqf, Zhou2020NoncrossingQR, kuang2023qemrl, fujimoto2018addressing}. 
While distributional RL is commonly applied in the context of off-policy methods, the convergence characteristics of these methods are not fully understood, typically explored under strict assumptions such as infinite sampling and Q-updates~\citep{fujimoto2018addressing}. It's important to note that distributional RL does not inherently provide direct guarantees for worst-case performance. 
\textbf{Risk-sensitive RL}~\cite{alexander2004comparison}, which falls into the second category, stands as a potent class of safe RL methods designed to ensure high-confidence satisfaction of worst-case costs~\citep{yinlam2015risk,berkenkamp2017safe,chow2018risk,tang2019worst,jain2021variance,chen2023probabilistic,yu2023global}. However, these methods rely on a Lagrangian approach and necessitate intricate assumptions for convergence~\citep{chow2018risk}. 
In the third category, \textbf{robust RL}~\citep{peng2017dr} addresses the challenge of minimizing performance fluctuations caused by applying a policy to different environments with dynamic discrepancies~\cite{aravind2016epopt, jiang2021mrpo, you2022useroriented}. Typically, these methods involve training policies on a diverse set of environments to enhance performance robustness~\citep{peng2017dr, panaganti2022robust}.
Nevertheless, Robust RL does not address the enhancement of performance robustness within a single environment.

Distinct from all the aforementioned categories, 
\textbf{ensuring the improvement of worst-case policy performance for a given environment} remains an uncharted area. In this paper, we address this challenge by introducing a novel trust-region policy optimization method. Unlike distributional RL and risk-sensitive RL, our approach is an on-policy trust-region method that explicitly ensures the monotonic improvement of the worst-case performance without relying on complex assumptions. In contrast to robust RL, our method directly enhances the worst-case performance for a given environment. We delve into existing works and establish connections to our approach in detail in \Cref{sec: related works}.

Specifically, we introduce a novel theoretical breakthrough 
that ensures the monotonic improvement of
the lower probability bound of performance with high confidence.
Subsequently, we implement a series of approximations to transform this theoretically-grounded algorithm into a practical solution, which we refer to as \textbf{A}bsolute \textbf{P}olicy \textbf{O}ptimization (APO). The main idea of APO is illustrated in \cref{fig: trpo vs apo}. Remarkably, APO exhibits scalability and can efficiently optimize nonlinear policies characterized by tens of thousands of parameters. 
Our experimental results underscore the effectiveness of APO and its efficient variation PAPO, demonstrating substantial performance improvements in terms of both worst-case performance and expected performance compared to state-of-the-art policy gradient algorithms. These improvements are evident across challenging continuous control benchmark tasks and extend to the realm of playing Atari games. 
Our code is available on Github.\footnote{\url{https://github.com/intelligent-control-lab/absolute-policy-optimization}}

Our contribution is summarized below:
\begin{itemize}
    \item To the best of the authors' knowledge, the proposed approach is the first to guarantee the monotonic improvement of lower probability bound of performance with high confidence.
\end{itemize}

\section{Problem Formulation}

\subsection{Notations}
Consider an infinite-horizon discounted Markov decision process (MDP) defined by the tuple $(\mathcal{S}, \mathcal{A}, \gamma, \mathcal{R}, P, \mu)$, where $\mathcal{S}$ is the state space, and $\mathcal{A}$ is the control space, $\mathcal{R}: \mathcal{S} \times \mathcal{A} \mapsto \mathbb{R}$ is a bounded reward function, $ 0 \leq \gamma < 1$ is the discount factor, $\mu : \mathcal{S} \mapsto \mathbb{R}$
is the bounded initial state distribution, and $P: \mathcal{S} \times \mathcal{A} \times \mathcal{S} \mapsto \mathbb{R}$
is the transition probability.
$P(s'|s,a)$ is the probability of transitioning to state $s'$ when the agent takes action $a$ at state $s$.
A stationary policy $\pi: \mathcal{S} \mapsto \mathcal{P}(\mathcal{A})$ is a mapping from states to a probability distribution over actions, with $\pi(a|s)$ denoting the probability of selecting action $a$ in state $s$. We denote the set of all stationary policies by $\Pi$. Subsequently, we denote $\pi_\theta$ as the policy that is parameterized by the parameter $\theta$. 

The standard goal for MDP is to learn a policy $\pi$ that maximizes a performance measure
$\mathcal{J}(\pi)$
which is computed via the discounted sum of reward $\mathcal{J}(\pi) = \mathbb{E}_{\tau \sim \pi}\left[\textstyle\sum_{t=0}^\infty \gamma^t R(s_t, a_t, s_{t+1})\right]$,
where $\tau = [s_0, a_0, s_1, \cdots]$, and $\tau \sim \pi$ is shorthand for that the distribution over trajectories depends on $\pi: s_0 \sim \mu, a_t \sim \pi(\cdot | s_t), s_{t+1} \sim P(\cdot|s_t, a_t)$. Mathematically,
\begin{align}
    \label{eq: mdp objective}
    \maximizewrt{\pi \in \Pi} ~ \mathcal{J}(\pi)~.
\end{align}
Additionally, a performance sample is defined here as $R_\pi(s_0) \doteq \sum_{t=0}^\infty \gamma^t R(s_t, a_t, s_{t+1})$, where the state action sequence $\hat\tau = [a_0, s_1, \dots] \sim \pi$ starts with an initial state $s_0$, which follows initial state distribution $\mu$.

The value function is denoted 
as $V_\pi(s) \doteq \mathbb{E}_{\tau \sim \pi}[R_\pi(s) | s_0 = s]$, the action-value function as $Q_\pi(s,a) = \mathbb{E}_{s' \sim P}[Q_\pi(s,a,s')] \doteq \mathbb{E}_{\tau \sim \pi}[R_\pi(s) | s_0 = s, a_0=a]$, and the advantage function as $A_\pi(s,a) = \mathbb{E}_{s' \sim P}[A_\pi(s,a,s')] \doteq Q_\pi(s, a) - V_\pi(s)$. We also define ${\bar A_{\pi',\pi}(s)}$ as the expected advantage of ${\pi'}$ over ${\pi}$ at state ${s}$: $\bar A_{\pi',\pi}(s) \doteq  {\mathbb{E}}_{a \sim \pi'}[A_{\pi}(s,a)]$.

\subsection{Lower Probability Bound of Performance}
\label{sec: absolute bound}

Notice that maximizing the expected reward performance ($\mathcal{J}$), unfortunately, does not provide control over each individual performance sample ($R_\pi(s_0)$) derived from the policy $\pi$. 
In a practical reinforcement learning setting, unexpected poor performance samples can lead to training instability, compromising the reliability of solutions in real-world applications. To tackle this issue, our  insight is that policy optimization should not be solely fixated on enhancing expected performance, but also on improving the
lower bound
originating from the distribution of the variable $R_\pi(s_0)$. For the continuous random variable $R_\pi(s_0)$, the best we can do is to improve the lower probability bound~\cite{pishro2014introduction} of performance with high confidence, defined as:

\begin{defi}[Lower Probability Bound of Performance]
\label{def: abs bound}

    Given a tuple $(\mathcal{B} \in \mathbb{R}, p \in \mathbb{R}^+)$, $\mathcal{B}$ is defined as the lower probability bound of performance with confidence ${p}$. Mathematically:
    \begin{align}
        Pr\big(R_\pi(s_0)\geq \mathcal{B}\big) \geq p,
    \end{align}
    For an unknown performance distribution of policy ${\pi}$, we first define $\mathcal{V}(\pi)$ as the variance of the performance distribution.
    Then, we can leverage 
    the Selberg's inequality theory~\citep{saw1984chebyshev}
    to obtain an lower probability bound of performance as $ \mathcal{B}_k(\pi) \doteq \mathcal{J}(\pi)-k\mathcal{V}(\pi)$, which is guaranteed to satisfy \Cref{def: abs bound} (proved in \Cref{prop: absolute bound definition}) with confidence $p_k^\psi \doteq 1-\frac{1}{k^2 \psi+1} \in (0,1)$. Here $k$ is the probability factor ($k \geq 0$, $k \in \mathbb{R}$) and $\psi = \mathcal{V}_{min} \in \mathbb{R}^+$, where $\mathcal{V}_{min}$ is minima of $\mathcal{V}(\pi)$.
\end{defi}

\begin{rmk}
    \Cref{def: abs bound} shows that more than $p_k^\psi$ of the samples from the distribution of $R_\pi(s_0)$ will be larger than the bound $\mathcal{B}_k(\pi)$. Given a positive constant $\psi$, we can make $p_k^\psi \rightarrow 1$ by setting a large enough $k$, so that $\mathcal{B}_k(\pi)$ represents the lower probability bound of performance with with a confidence level close to 1.
\end{rmk}

\subsection{Problem}
\label{sec:amdp}

In this paper, we focus on improving the lower probability bound of performance in 
 Markov Decision Processes (MDP).
In accordance with \Cref{prop: absolute bound definition}, the overarching objective
is to identify a policy $\pi$ that effectively improves $\mathcal{B}_k(\pi)$.  Mathematically,
\begin{align}
\label{eq: apo optimization original}
     \maximizewrt{\pi\in\Pi} ~ \mathcal{J}(\pi)-k\mathcal{V}(\pi).
\end{align}



\section{Absolute Policy Optimization}

To optimize \eqref{eq: apo optimization original}, we need to evaluate the objective with respect to an unknown $\pi$.
Our main intuition is to find a surrogate function for the objective, such that (i) it represents a tight lower bound of the objective; and (ii) it can be easily estimated from the samples on the most recent policy. To solve large and continuous MDPs, policy search algorithms look for the optimal policy within a set $\Pi_\theta \subset \Pi$ of parametrized policies. Mathematically, APO updates solve the following optimization:
\begin{align}
\label{eq: apo optimization final}
    \pi_{j+1} = \underset{\pi \in \Pi_{\theta}}{\textbf{argmax}} ~ \mathcal{J}^l_{\pi, \pi_j} - k\left({MV}_{\pi,\pi_j} + {VM}_{\pi,\pi_j} \right)
\end{align}
where $\mathcal{J}^l_{\pi, \pi_j}$ represents the lower bound surrogate function for $\mathcal{J}(\pi)$ and $\left({MV}_{\pi,\pi_j} + {VM}_{\pi,\pi_j} \right)$ represents the upper bound surrogate function for $\mathcal{V}(\pi)$ in the (j+1)-th iteration.
\begin{rmk}
     $MV_{\pi,\pi_j}$ reflects the  upper bound of expected variance of the return over different start states. $VM_{\pi,\pi_j}$ reflects the upper bound of variance of the expected return of different start states. The detailed interpretations are shown in \cref{fig:mv and vm} and discussed in \Cref{eq:variance interpretation}, \Cref{lem: bound of MV}, and \Cref{lem: bound of VM}.
\end{rmk}

\begin{rmk}[\textbf{Balance of Exploration and Exploitation}]
   Intuitively, \eqref{eq: apo optimization final} improves performance expectation and minimizes performance variance, where $k$ controls the importance of the two folder objectives. Higher performance variance (decreasing $k$) allows for more exploration, introducing the possibility of bad outcomes, thereby hindering exploitation. Conversely, reducing performance variance (increase $k$) limits exploration, impeding exploration and resulting in convergence to local optima. Hence, a moderate $k$ is desirable, which is further discussed in \Cref{sec: k factor exploration}. \textbf{The objective of TRPO can be viewed as a special case of \eqref{eq: apo optimization final}, where $k=0$.}
\end{rmk}

Here $\mathcal{J}^l_{\pi, \pi_j}, {MV}_{\pi,\pi_j}, {VM}_{\pi,\pi_j}$ are defined as:

\begin{align}
\label{eq: value lower bound}
    &\mathcal{J}^l_{\pi, \pi_j}\doteq\mathcal{J}(\pi_j)\\ \nonumber
    &~~~~~+\frac{1}{1-\gamma} \underset{\substack{\substack{s \sim d^{\pi_j}\\a\sim {\pi}}}}{\mathbb{E}} \bigg[ A_{\pi_j}(s,a) - \frac{2\gamma \epsilon^{\pi}}{1-\gamma} \sqrt{\frac 12 \mathcal{D}_{KL}({\pi} \| \pi_j)[s]} \bigg] \\
\label{eq: mv def}
    &{MV}_{\pi,\pi_j} \doteq \frac{\|\mu^\top\|_\infty}{1-\gamma^2}\underset{s}{\textbf{max}}\Bigg|\underset{\substack{\\a\sim\pi\\s'\sim P}}{\mathbb{E}}\left[A_{\pi_j}(s,a,s')^2\right]\\ \nonumber
    &~~~~~-\underset{\substack{\\a\sim\pi_j\\s'\sim P}}{\mathbb{E}}\left[A_{\pi_j}(s,a,s')^2\right] + |H(s,a,s')|_{max}^2 \\ \nonumber
    &~~~~~+ 2\underset{\substack{a \sim \pi \\ s' \sim P}}{\mathbb{E}}\left[A_{\pi_j}(s,a,s')\right]\cdot|H(s,a,s')|_{max}\Bigg| \\ \nonumber
    &~~~~~+ MV_{\pi_j} + \frac{2\gamma^2\|\mu^\top\|_\infty}{(1-\gamma^2)^2}\sqrt{\frac{1}{2}\mathcal{D}_{KL}^{max}(\pi \| \pi_j)}\cdot\|\Omega_{\pi_j}\|_\infty \\
\label{eq: vm def}
    &{VM}_{\pi,\pi_j} \doteq \|\mu^\top\|_\infty\underset{s}{\textbf{max}}\bigg||\eta(s)|_{max}^2+2|V_{\pi_j}(s)|\cdot|\eta(s)|_{max}\bigg| \\ \nonumber
    &~~~~~-\textbf{min} ~\left( \mathcal{J}(\pi) \right)^2 + \underset{s_0 \sim \mu}{\mathbb{E}} [V_{\pi_j}^2(s_0)]
\end{align}

where $\mathcal{D}_{KL}(\pi \| \pi_j)[s]$ is the KL divergence between $(\pi, \pi_j)$ at state $s$ and $\mathcal{D}_{KL}^{max}(\pi \| \pi_j) = \max_s \mathcal{D}_{KL}(\pi \| \pi_j)[s]$;
$\epsilon^{\pi} \doteq \underset{s}{\textbf{max}}|\mathbb{E}_{a\sim {\pi}}[A_{\pi_j}(s,a)]|$ is the maximum expected advantage;
$d^{\pi_j} \doteq (1-\gamma)\sum_{t=0}^\infty\gamma^t P(s_t=s|{\pi_j})$ is the discounted future state distribution;
$\omega_{\pi_j}(s) \doteq \underset{\substack{a \sim \pi_j\\ s' \sim P}}{\mathbb{E}} \big[Q_{\pi_j}(s,a,s')^2\big] - V_{\pi_j}(s)^2$ is the variance of action value;
$\Omega_{\pi_j} \doteq \begin{bmatrix}
    \omega_{\pi_j}(s^1) &
    \omega_{\pi_j}(s^2) &
    \hdots
\end{bmatrix}^\top$ is the vector of variance of action value;
 $MV_{\pi_j} \doteq \underset{\substack{s_0 \sim \mu}}{\mathbb{E}}[\underset{\substack{{\hat \tau}\sim {\pi_j}}}{\mathbb{V}ar}[R_{\pi_j}(s_0)]$ is the expectation of performance variance over initial state distribution;
 $\textbf{min} ~\left( \mathcal{J}(\pi) \right)^2 \doteq \minimizewrt{\mathcal{J}(\pi) \in [\mathcal{J}^l_{\pi, \pi_j}, \mathcal{J}^u_{\pi, \pi_j}]} \left( \mathcal{J}(\pi) \right)^2 $ is the minimal squared performance,
 where the upper bound of $\mathcal{J}(\pi)$ is defined as $\mathcal{J}^u_{\pi, \pi_j}\doteq\mathcal{J}(\pi_j) + \frac{1}{1-\gamma} \underset{\substack{s \sim d^{\pi_j} \\ a\sim {\pi}}}{\mathbb{E}} \bigg[ A_{\pi_j}(s,a) + \frac{2\gamma \epsilon^{\pi}}{1-\gamma} \sqrt{\frac 12 \mathcal{D}_{KL}({\pi} \| \pi_j)[s]} \bigg]$.
Additionally,
\begin{align}
    \label{eq: H def}
    &|H(s,a,s')|_{max}\doteq\dfrac{2\gamma(1+\gamma)\epsilon}{(1-\gamma)^2}\mathcal{D}_{KL}^{max}(\pi||\pi_j)\\ \nonumber
    &~~~~~+\left|\gamma\underset{\substack{s_0 = s' \\ \hat\tau \sim \pi_j}}{\mathbb{E}}\bigg[\sum_{t=0}^\infty \gamma^t \bar A_{\pi,\pi_j}(s_t)\bigg] - \underset{\substack{s_0 = s \\ \hat\tau \sim \pi_j}}{\mathbb{E}}\bigg[\sum_{t=0}^\infty \gamma^t\bar A_{\pi,\pi_j}(s_t)\bigg] \right| \\ \nonumber
    \label{eq: eta def}
    &|\eta(s)|_{max}\doteq \left| \underset{\substack{s_0 = s \\\hat\tau \sim \pi_j}}{\mathbb{E}}\bigg[\sum_{t=0}^\infty \gamma^t\bar A_{\pi,\pi_j}(s_t)\bigg] \right| \\ \nonumber
    &~~~~~ + \frac{2\gamma \epsilon}{(1-\gamma)^2}\mathcal{D}_{KL}^{max}({\pi} \| \pi_j),
\end{align}
where $\epsilon\doteq \underset{s, a}{\mathbf{max}}|A_{\pi_j}(s,a)|$.

\begin{figure}[t]
    \raisebox{-\height}{\includegraphics[width=\linewidth]{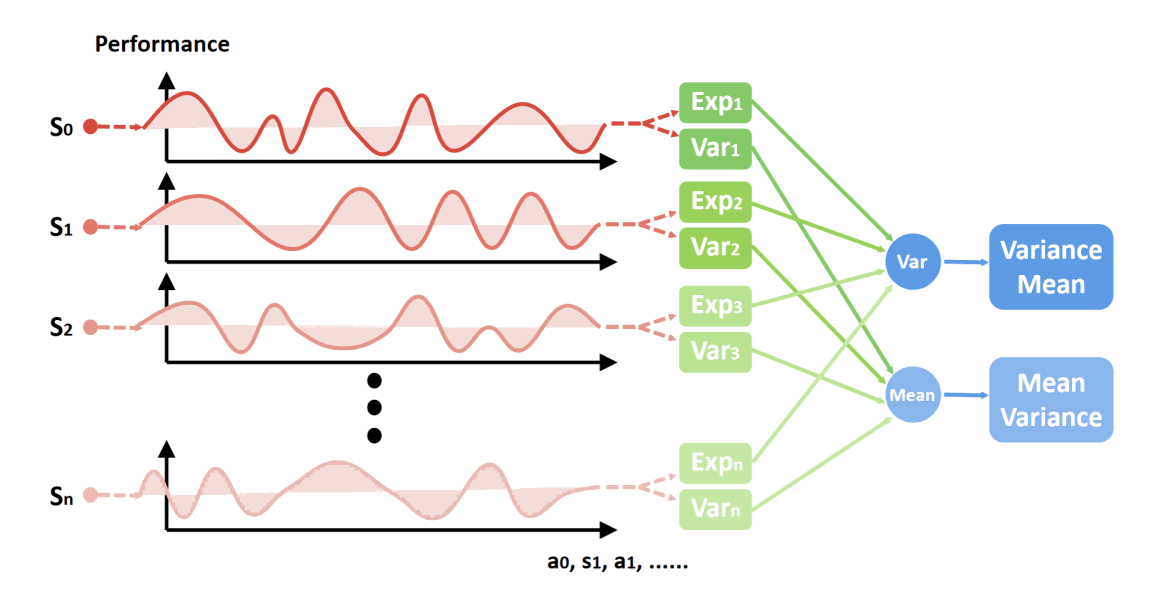}}
    \caption{Explanation of MV and VM. Since performance from different start states belong to a mixture of one-dimensional distributions, the variance of performance can be deconstructed into two components: MeanVariance and VarianceMean.
     } 
    \label{fig:mv and vm}
\end{figure}

\subsection{Theoretical Guarantees for APO}
\label{sec: theory guarantee}

\begin{theo}[Monotonic Improvement of Absolute Performance]
\label{theo: absolute performance improvement} Suppose $\pi, \pi'$ are related by \eqref{eq: apo optimization final}, then absolute performance bound $ \mathcal{B}_k(\pi) = \mathcal{J}(\pi)-k\mathcal{V}(\pi)$ satisfies $\mathcal{B}_k(\pi') \geq \mathcal{B}_k(\pi)$.
\end{theo}

The proof for \Cref{theo: absolute performance improvement} is summarized in \Cref{proof: absolute performance improvement}.
In essence, we introduce the function $\mathcal{M}_k^j(\pi) = \mathcal{J}^l_{\pi, \pi_j} - k\left({MV}_{\pi,\pi_j} + {VM}_{\pi,\pi_j} \right)$ (the right-hand side of \eqref{eq: apo optimization final}). Our demonstration establishes that $\mathcal{M}_k^j(\pi)$ serves as the lower bound for the absolute bound $\mathcal{B}_k(\pi)$ through the application of \Cref{lem: bound of MV}, \Cref{lem: bound of VM}, and \Cref{lem: lower bound of mean}.

Subsequently, leveraging the facts that $\mathcal{B}_k(\pi_j) = \mathcal{M}_k^j({\pi_j})$ (proved in [\Cref{lem: b eq m}, \Cref{sec: additional results}]) and  $\mathcal{B}_k(\pi_{j+1})\geq \mathcal{M}_k^j(\pi_{j+1})$, the following inequality holds:
\begin{align}
    \mathcal{B}_k(\pi_{j+1}) - \mathcal{B}_k(\pi_{j}) \geq \mathcal{M}_k^j(\pi_{j+1}) - \mathcal{M}_k^j(\pi_{j})
\end{align}
Therefore, through the maximization of $\mathcal{M}_k^j$ at each iteration (\eqref{eq: apo optimization final}), we ensure the true absolute performance $\mathcal{B}_k$ is non-decreasing. Furthermore, it is noteworthy that $\pi_j$ is always a feasible solution for \eqref{eq: apo optimization final}.

\section{Practical Implementation}
\label{sec: practical implementation}
In this section, we show how to
(i) encourage larger update steps with trust region constraint, (ii) simplify complex computations. The full APO pseudocode is provided as \Cref{alg:apo_main} in \Cref{append: apo}.


\paragraph{Trust Region Constraint} 
In practice, strictly following the theoretical recommendations for the coefficients of KL divergence terms in \eqref{eq: apo optimization final} often leads to very small step sizes. Instead, a practical approach is to enforce a constraint on the KL divergence between the new and old policies~\citep{schulman2015trust}, commonly known as a trust region constraint. This strategy allows for taking larger steps in a robust way:
\begin{align}
\label{eq: apo optimization weighted sum version}
    \pi_{j+1} &= \underset{\pi \in \Pi_{\theta}}{\textbf{argmax}} ~ \frac{1}{1-\gamma} \underset{\substack{s \sim d^{\pi_j} \\ a\sim {\pi}}}{\mathbb{E}} \left[ A_{\pi_j}(s,a) \right] \\ \nonumber 
    &~~~~- k\left(\overline{MV}_{\pi, \pi_j}+\overline{VM}_{\pi, \pi_j}\right) \\ \nonumber
    &\textbf{s.t.}~ \bar{\mathcal{D}}_{KL}(\pi||\pi_j) \leq\delta
\end{align}
where $\delta$ is the step size, $\overline{MV}_{\pi,\pi_j}\doteq{MV}_{\pi,\pi_j} - MV_{\pi_j} - \frac{2\gamma^2\|\mu^\top\|_\infty}{(1-\gamma^2)^2}\sqrt{\frac{1}{2}\mathcal{D}_{KL}^{max}(\pi \| \pi_j)}\cdot\|\Omega_{\pi_j}\|_\infty$ and $\overline{VM}_{\pi.\pi_j} = {VM}_{\pi,\pi_j}-\underset{s_0 \sim \mu}{\mathbb{E}} [V_{\pi_j}^2(s_0)]$.
The set $\{\pi \in \Pi_\theta: \bar{\mathcal{D}}_{KL}(\pi||\pi_j)=\underset{{s \sim \pi_j}}{\mathbb{E}}[\mathcal{D}_{KL}(\pi \| \pi_j)[s]] \leq \delta\}$ is called \textit{trust region}. Notice that $MV_{\pi_j}$ and $\underset{s_0 \sim \mu}{\mathbb{E}}[V_{\pi_j}^2(s_0)]$ are computable constant.

\paragraph{Special Parameters Handling} 
When implementing \Cref{eq: apo optimization weighted sum version}, we first treat two items as hyperparameters. (i) \bm{$\|\mu^\top\|_\infty$}:~ 
Although the infinity norm of $\mu^\top$ is theoretically equal to 1, we found that treating it as a hyperparameter in $\mathbb{R}^+$ enhances performance in practical implementation.
(ii) \bm{$|H(s,a,s')|_{max}$}: ~ We can either compute $|H(s,a,s')|_{max}$ from the most recent policy with \eqref{eq: H def} or treat it as a hyperparameter since $|H(s,a,s')|_{max}$ is bounded for any system with a bounded reward function. In practice, we found that the hyperparameter option helps increase the performance. This setting will be discussed more detailedly in \Cref{sec: ablation of Hmax}. 
(iii) \bm{$|\eta(s)|_{max}$} and $\underset{\bm s}{\bm{\max}}\bm{|~\cdots~|}$: ~ 
Furthermore, we find that taking the average of the state $s$ instead of the maximum will be more superior and stable in terms of convergence results. Note that a similar trick is also applied in \citep{schulman2015trust} to handle maximum KL divergence.

\section{Experiment}
In our experiments, we want to answer the following questions: \\
\textbf{Q1:} How does APO compare with state-of-the-art on-policy RL algorithms?\\
\textbf{Q2:} What benefits are demonstrated by directly optimizing the absolute performance?\\
\textbf{Q3:} Is treating $H_{max}$ as a hyperparameter necessary?\\
\textbf{Q4:} What are the impacts of different probability factor $k$ choices?\\
\textbf{Q5:} What potential does APO hold?\\
\textbf{Q6:} How does APO compare with state-of-the-art algorithms in terms of computational cost? \\
\textbf{Q7:} What are the scenarios where APO is most beneficial or less effective?

\subsection{Experiment Setup}
\label{sec: experiment setup}
To answer the above, we run experiments on both continuous domain and the discrete domain.
\paragraph{Continuous Tasks} Our continuous experiments are conducted on GUARD~\citep{zhao2023guard}, a challenging robot locomotion benchmark build upon Mujoco~\citep{mujoco} and Gym. Seven different robots are included: (i) \textbf{Point: } (\Cref{fig: Point}) A point-mass robot ($\mathcal{A} \subseteq \mathbb{R}^{2}$) that can move on the ground. (ii) \textbf{Swimmer: } (\Cref{fig: Swimmer}) A three-link robot ($\mathcal{A} \subseteq \mathbb{R}^{2}$) that can move on the ground. (iii) \textbf{Arm3: } (\Cref{fig: Arm3}) A fixed three-joint robot arm($\mathcal{A} \subseteq \mathbb{R}^{3}$) that can move its end effector around with high flexibility. (iv) \textbf{Drone: } (\Cref{fig: Drone}) A quadrotor robot ($\mathcal{A} \subseteq \mathbb{R}^{4}$) that can move in the air. (v) \textbf{Hopper: } (\Cref{fig: Hopper}) A one-legged robot ($\mathcal{A} \subseteq \mathbb{R}^{5}$) that can move on the ground. (vi) \textbf{Ant: } (\Cref{fig: Ant}) A quadrupedal robot ($\mathcal{A} \subseteq \mathbb{R}^{8}$) that can move on the ground. (vii) \textbf{Walker: } (\Cref{fig: Walker}) A bipedal robot ($\mathcal{A} \subseteq \mathbb{R}^{10}$) that can move on the ground. Furthermore, three different types of tasks are considered, including (i) \textbf{Goal: } (\Cref{fig: Goal3D1}) robot navigates towards a series of 2D or 3D goal positions. (ii) \textbf{Push: } (\Cref{fig: PushBall1}) robot pushes a ball toward different goal positions. (iii) \textbf{Chase: } (\Cref{fig: Chase}) robot tracks multiple dynamic targets. Considering these different robots and tasks, we design 8 low-dim test suites and 4 high-dim test suits with 7 types of robots and 3 types of tasks, which are summarized in \Cref{tab: testing suites} in Appendix. We name these test suites as \texttt{\{Task Type\}\_\{Robot\}}. Further details are listed in \Cref{appendix:Environment Settings}.

\begin{figure*}[t]
    \centering
    \begin{subfigure}[b]{0.9\textwidth}
        \raisebox{-\height}{\includegraphics[width=\textwidth]{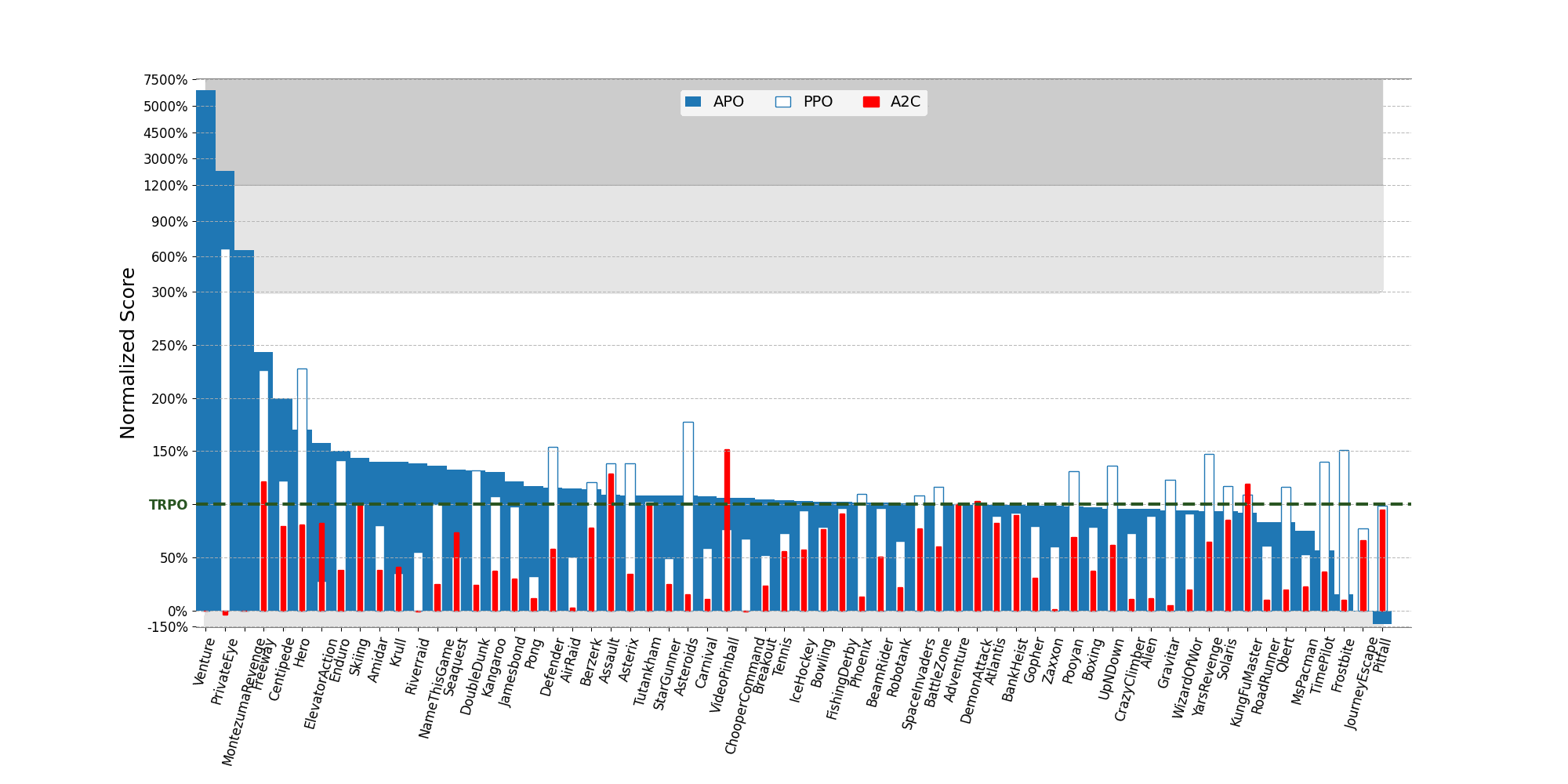}}
    \end{subfigure}
    \caption{Stacked bar chart for all 62 atari games}
    \label{fig:atari hist}
    \vspace{-10pt}
\end{figure*}

\begin{table*}[h]
\begin{center}
\begin{tabular}{cc|ccccc}
\toprule
\multicolumn{1}{c|}{Performance}&Metrics& APO  & PPO & TRPO & A2C & Tie\\
\hline \\[-0.95em]
\multicolumn{1}{c|}{} &(1) average expected reward over all epochs(Atari)& \textbf{26} & 22 & 10 & 3 & 1\\
\multicolumn{1}{c|}{Expected}&(2) average expected reward over last 100 epochs(Atari) & \textbf{29} & 17 & 12 & 3 & 1\\
\multicolumn{1}{c|}{Performance} &(1) average expected reward over all epochs(GUARD)& \textbf{9} & 3 & 0 & 0 & 0\\
\multicolumn{1}{c|}{}&(2) average expected reward over last 20 epochs(GUARD) & \textbf{9} & 2 & 1 & 0 & 0\\
\hline \\[-0.95em]
\multicolumn{1}{c|}{}&(1) average worst reward over all epochs(Atari)& \textbf{26} & 24 & 7 & 3 & 2\\
\multicolumn{1}{c|}{Worst}&(2) average worst reward over last 100 epochs(Atari)& \textbf{27} & 20 & 10 & 3 & 2\\
\multicolumn{1}{c|}{Performance}&(3) average worst reward over all epochs(GUARD)& \textbf{8} & 2 & 1 & 1 & 0\\
\multicolumn{1}{c|}{}&(4) average worst reward over last 20 epochs(GUARD) & \textbf{9} & 1 & 2 & 0 & 0\\
\bottomrule
\end{tabular}
\caption{The number of highest evaluation scores obtained by each algorithm across all test suites}
\label{tab:metric performance}
\end{center}
\vspace{-15pt}
\end{table*}

Additionally, we conduct continuous control experiments on Mujoco Openai Gym \citep{openaigym} and Gymnasium Robotics \citep{matthias2018gymrobotics}. Five high dimensional tasks are considered: (i) \textbf{Humanoid: } (\Cref{fig: Humanoid}) The 3D bipedal robot ($\mathcal{A} \subseteq \mathbb{R}^{17}$) is designed to simulate a human. And the goal of the environment is to walk forward as fast as possible without falling over. (ii) \textbf{Humanoid Standup: } (\Cref{fig: HumanoidStandup}) The robot ($\mathcal{A} \subseteq \mathbb{R}^{17}$) is same with task \textbf{Humanoid}, but the goal is to make the humanoid standup and then keep it standing. These two tasks are also summarized in \Cref{tab: testing suites}. (iii) \textbf{HandReach: } (\Cref{fig: HandReach}) The goal of the task is for the fingertips of the hand ($\mathcal{A} \subseteq \mathbb{R}^{24}$) to reach a predefined target Cartesian position. (iv) \textbf{HandManipulateEgg: } (\Cref{fig: HandManipulateEgg}) The task is to manipulate the egg such that a target pose is achieved (v) \textbf{HandManipulateBlock: } (\Cref{fig: HandManipulateBlock}) The task is to manipulate the block such that a target pose is achieved. Notice that tasks (iii)-(v) employ dense rather than discrete reward setup for learning, see \cite{matthias2018gymrobotics} for details.

\paragraph{Discrete Tasks} 
We also test APO in all 62 Atari environments of \citep{openaigym} which are simulated on the Arcade Learning Environment benchmark \citep{atarigame}.
All experiments are based on `v5' environments and `ram' observation space.

\paragraph{Comparison Group}
We compare APO 
to the state-of-the-art on-policy  RL algorithms: (i) TRPO \citep{schulman2015trust} (ii) Advantage Actor Critic (A2C) \citep{a2c} (iii) PPO by clipping \citep{schulman2017ppo} in both continuous and discrete tasks 
and additionally compare (iv) AlphaPPO \citep{xu2023improving} (v) ESPO \citep{sun2022espo} (vi) V-MPO \citep{song2020vmpo} on continuous tasks.For all experiments, we take the best specific parameters mentioned in the original \cite{xu2023improving} paper and keep the common parameters as the same. In particular, we selected the optimal parameters based on the tuning method of the  paper for each task individually. The policy $\pi$, the value $V^\pi$ are all encoded in feedforward neural networks using two hidden layers of size (64,64) with tanh activations. The full list of parameters of all methods and tasks compared can be found in \Cref{appendix:Policy Settings}. 

\subsection{Comparison to Other Algorithms on the Atari Domain}
\cref{fig: selected atari} shows 12 representative test suites in Atari Domain. All 62 learning curves can be found in \Cref{sec:total experiments}. The hyperparameters for Atari Domain are also provided in \Cref{appendix:Policy Settings}. For the other three algorithms, we used hyperparameters that were tuned to maximize performance on this benchmark. Then we follow the metrics of \citep{schulman2017ppo} to quantitatively evaluate the strengths of APO: (i) average expected reward per episode over \textbf{all epochs of training} (which favors fast learning), and (ii) average expected reward per episode over \textbf{last 100 (20 of GUARD) epochs of training} (which favors final performance). \Cref{tab:metric performance} records the number of highest evaluation scores obtained by each algorithm across all games. To compare the performance of all testing algorithm to the TRPO baseline across games, we slightly change the normalization algorithm proposed by \citep{vanhasselt2015deep} to obtain more reasonable score (See \Cref{appendix:Score Settings} for further explanation) in percent. The score we used is average reward per episode over last 100 (20 of GUARD) epochs of training:
\begin{align}
\label{normlized score}
    &{\Delta}_1 \doteq score_{agent}-score_{random} \\ \nonumber
    &{\Delta}_2 \doteq score_{TRPO}-score_{random} \\ \nonumber
    &score_{normalized} = \frac{\Delta_2}{\Delta_1} ~ if ~ \Delta_1 < 0 ~ and ~ \Delta_2 <0 ~ else ~ \frac{\Delta_1}{\Delta_2}
\end{align}
Then we use a stacked bar chart  in \Cref{fig:atari hist} to visualize APO's capabilities. \Cref{fig:atari hist} show that APO has a superior combination of capabilities compared to other algorithms. The statistics of the algorithm in terms of expected performance improvement are presented in \Cref{tab:metric performance}. So far the above experimental comparison answers \textbf{Q1}.

\begin{figure*}[t]
    \centering
    \begin{subfigure}[t]{0.24\textwidth}
        \begin{subfigure}[t]{1.00\textwidth}
            \raisebox{-\height}{\includegraphics[width=\textwidth]{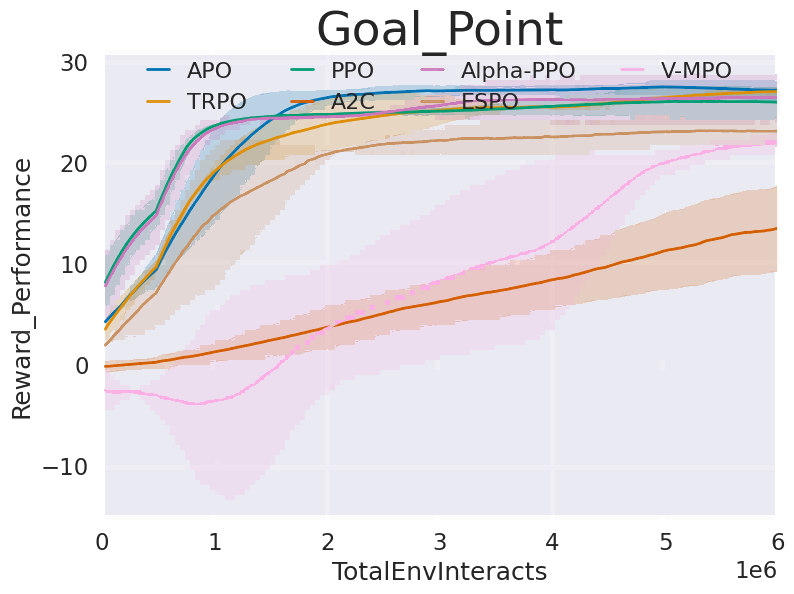}}
            \label{fig:Goal_Point_Reward_Performance}
        \end{subfigure}
    \end{subfigure}
   \begin{subfigure}[t]{0.24\textwidth}
        \begin{subfigure}[t]{1.00\textwidth}
            \raisebox{-\height}{\includegraphics[width=\textwidth]{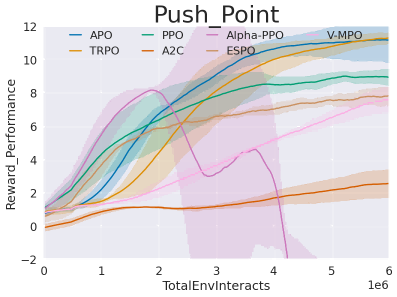}}
        \label{fig:Push_Point_Reward_Performance}
        \end{subfigure}
    \end{subfigure}
    \begin{subfigure}[t]{0.24\textwidth}
        \begin{subfigure}[t]{1.00\textwidth}
            \raisebox{-\height}{\includegraphics[width=\textwidth]{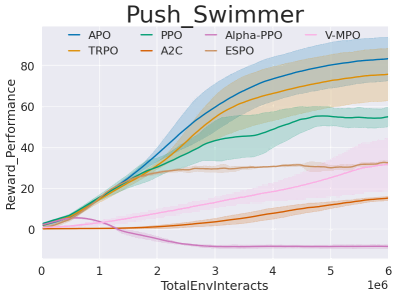}}
            \label{fig:Push_Swimmer_Reward_Performance}
        \end{subfigure}
    \end{subfigure}
    \begin{subfigure}[t]{0.24\textwidth}
        \begin{subfigure}[t]{1.00\textwidth}
            \raisebox{-\height}{\includegraphics[width=\textwidth]{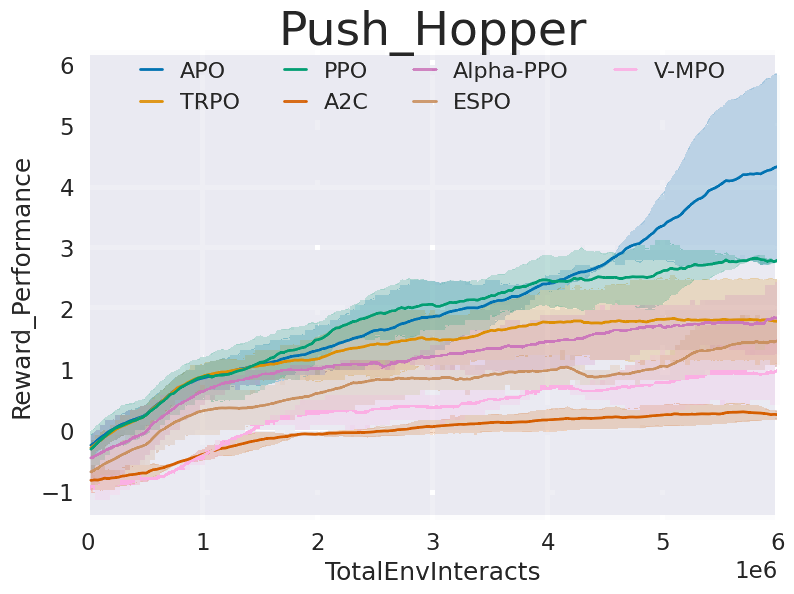}}
        \label{fig:Push_Hopper_Reward_Performance}
        \end{subfigure}
    \end{subfigure}
    \caption{Comparison of results from four representative test suites in low dimensional continuous systems}
    \label{fig:low-dim guard result}
    \vspace{-10pt}
\end{figure*}

\begin{figure*}[t]
    \centering
    \begin{subfigure}[t]{0.24\textwidth}
        \begin{subfigure}[t]{1.00\textwidth}
            \raisebox{-\height}{\includegraphics[width=\textwidth]{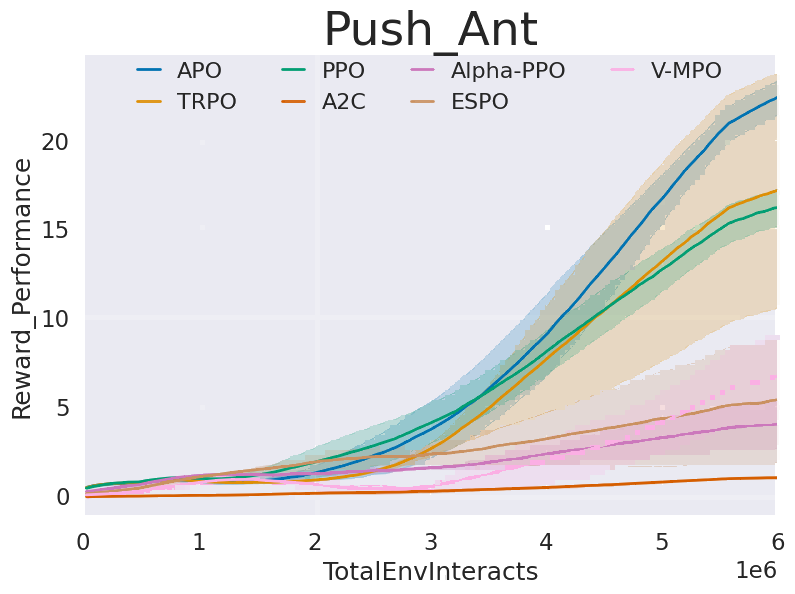}}
            \label{fig:Push_Ant_Reward_Performance}
        \end{subfigure}
    \end{subfigure}
   \begin{subfigure}[t]{0.24\textwidth}
        \begin{subfigure}[t]{1.00\textwidth}
            \raisebox{-\height}{\includegraphics[width=\textwidth]{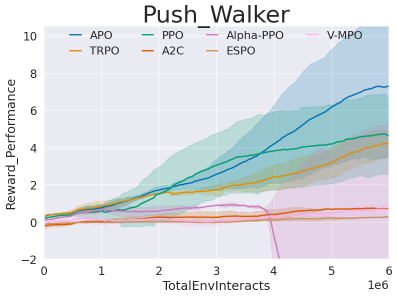}}
        \label{fig:Push_Walker_Reward_Performance}
        \end{subfigure}
    \end{subfigure}
    \begin{subfigure}[t]{0.24\textwidth}
        \begin{subfigure}[t]{1.00\textwidth}
            \raisebox{-\height}{\includegraphics[width=\textwidth]{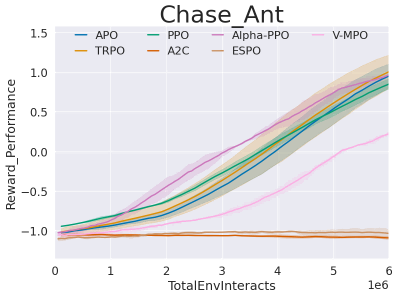}}
            \label{fig:Chase_Ant_Reward_Performance}
        \end{subfigure}
    \end{subfigure}
    \begin{subfigure}[t]{0.24\textwidth}
        \begin{subfigure}[t]{1.00\textwidth}
            \raisebox{-\height}{\includegraphics[width=\textwidth]{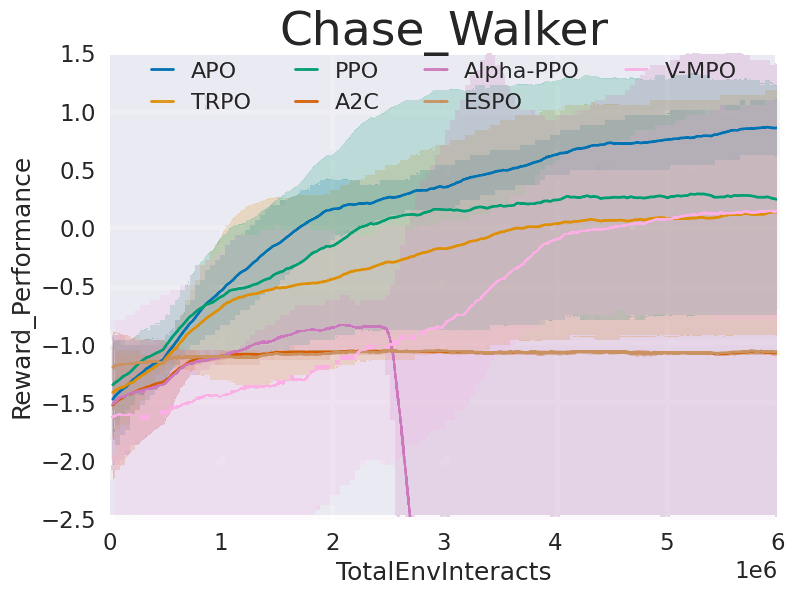}}
        \label{fig:Chase_Walker_Reward_Performance}
        \end{subfigure}
    \end{subfigure}
    \caption{Comparison of results from four representative test suites in high dimensional continuous systems} 
    \label{fig:high-dim guard result}
\end{figure*}

\subsection{Comparison to Other Algorithms in GUARD Continuous Domain}

\paragraph{Low dimension} 
\Cref{fig:low-dim guard result} shows representative comparison results on a low dimensional system (See \Cref{sec:total experiments} for all results). APO is successful at getting a more steady and higher final reward. We notice that PPO only gains faster convergence in part of the simplest task owing to its exploration abilities, the advantage decreases rapidly with more complex tasks such as PUSH. In difficult tasks, APO can perform best at the combined level of convergence speed and final performance.

\paragraph{High dimension}
\Cref{fig:high-dim guard result} reports the comparison results on challenging high-dimensional \{PUSH, CHASE\}\_Ant and \{PUSH, CHASE\}\_Walker tasks, where APO outperforms other baselines in getting higher reward and convergence speed. It is worth noting that we can clearly observe that PPO-related methods perform more erratically in complex tasks (PUSH and CHASE) and are not as good as TRPO in benchmark performance.

We also performed quantitative statistics on the 12 sets of continuous tasks tested. The results are presented in \Cref{tab:metric performance} and metrics we used are presented in the next section.

\subsection{Worst-case Performance Comparison}
We use a large probability factor $k$ in practical implementation, which means we are close to optimizing the lower bound for all samples. Thus we use another two similar metrics to evaluate the effectiveness of algorithms for lower bound lifting: (iii) average worst reward per episode over \textbf{all epochs of training}, and (iv) average worst reward per episode over \textbf{last 100 or 20 epochs of training}. We summarize the worst-case performance of APO in Atari games and GUARD in \Cref{tab:metric performance}, which answers \textbf{Q2}.

\subsection{Ablation on $H_{max}$ Hyperparameter Trick}
\label{sec: ablation of Hmax}

\begin{figure}[h]
    \centering
    \begin{subfigure}[b]{0.49\linewidth}
        \begin{subfigure}[t]{1.00\linewidth}\raisebox{-\height}{\includegraphics[width=\linewidth]{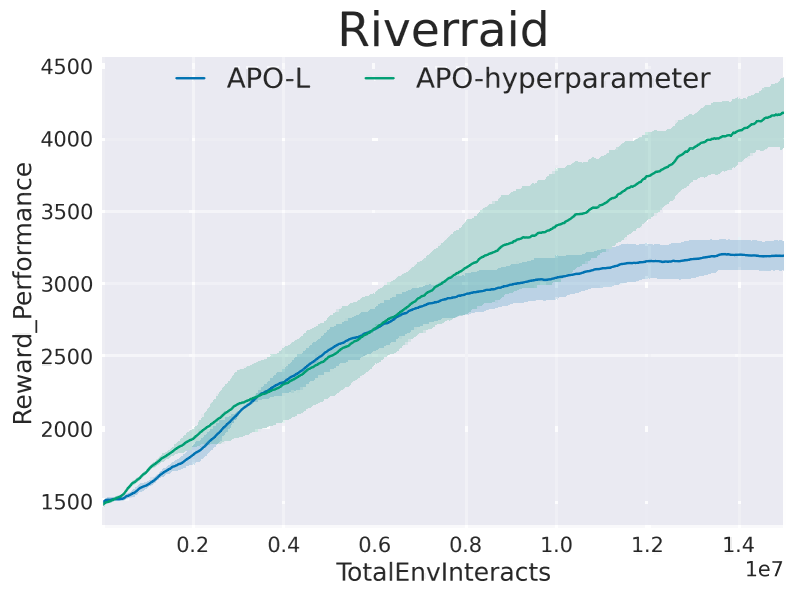}}
        \end{subfigure}
    \end{subfigure}
    \hfill
    \begin{subfigure}[b]{0.49\linewidth}
        \begin{subfigure}[t]{1.00\linewidth}\raisebox{-\height}{\includegraphics[width=\linewidth]{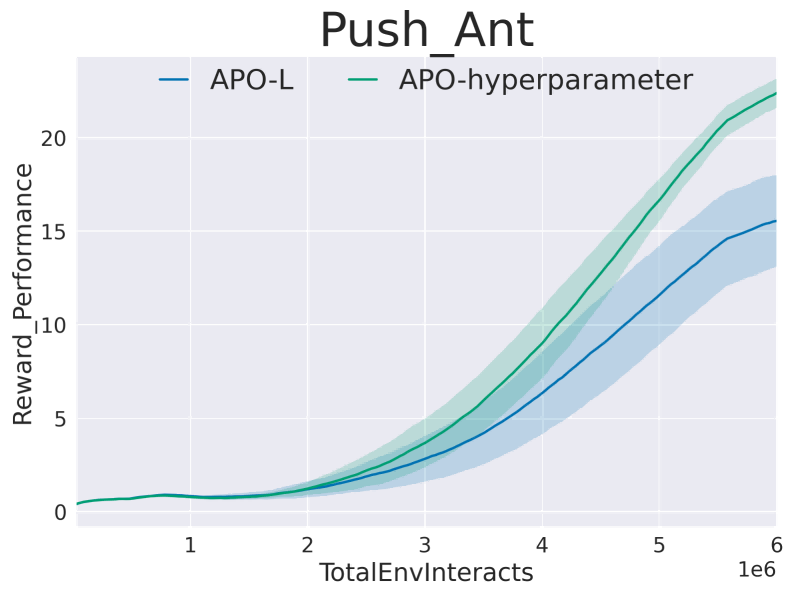}}
        \end{subfigure}
    \end{subfigure}
    \caption{Ablation on $H_{max}$ hyperparameter trick} 
    \label{fig: ablation on hmax}
    \vspace{-10pt}
\end{figure}

We chose Riverraid of discrete tasks and PUSH\_Ant of continuous tasks to perform ablation experiments against the $|H(s,a,s')|_{max}$ implementation. \Cref{fig: ablation on hmax} shows that although both boosts are similar in the early stages of tasks, hyperparameter methods can more consistently converge to a higher reward value. Thus, the figures and description answer \textbf{Q3}.

\subsection{Ablation on Probability Factor $k$}
\label{sec: k factor exploration}
For ablation, we selected Riverraid to investigate the impact of different choices for the probability factor $k$. As illustrated in Figure \ref{fig: k ablation}, when $k$ takes on a very small value, indicating optimization of only a limited portion of performance samples, the effectiveness diminishes. This is attributed to the loss of control over the lower probability bound of performance because of the small confidence level. Conversely, when $k$ becomes excessively large, the optimization shifts its focus towards the most extreme worst-case performance scenarios. This ultra-conservative approach tends to render the overall optimization less effective. Therefore, a moderate choice of $k$ will be favorable to the overall improvement of the effect, which answers \textbf{Q4}.

\begin{figure*}[t]
    \centering
    \begin{subfigure}[t]{0.19\textwidth}
        \begin{subfigure}[t]{1.00\textwidth}
            \raisebox{-\height}{\includegraphics[width=\textwidth]{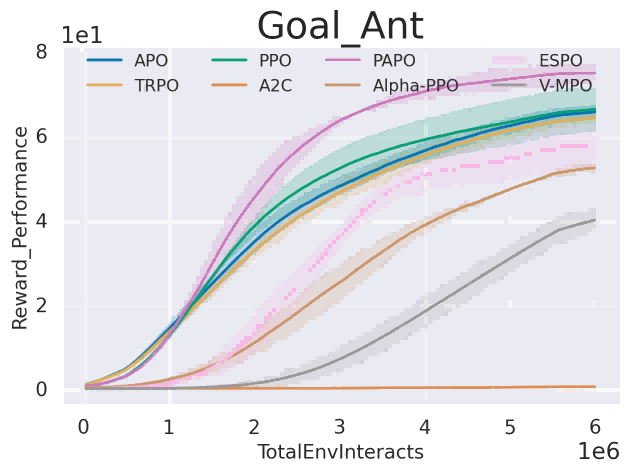}}
        \end{subfigure}
    \end{subfigure}
   \begin{subfigure}[t]{0.19\textwidth}
        \begin{subfigure}[t]{1.00\textwidth}
            \raisebox{-\height}{\includegraphics[width=\textwidth]{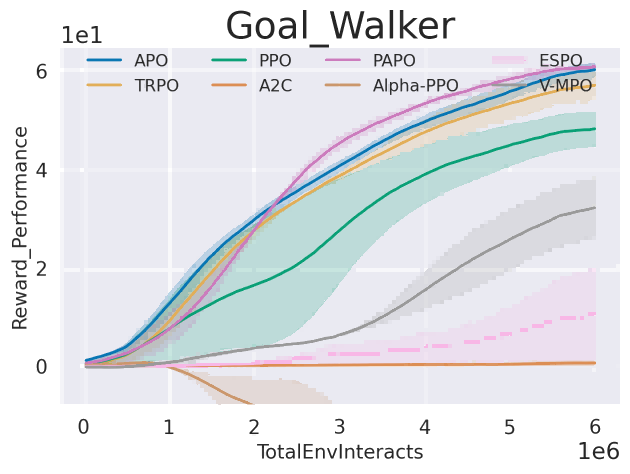}}
        \end{subfigure}
    \end{subfigure}
    \begin{subfigure}[t]{0.19\textwidth}
        \begin{subfigure}[t]{1.00\textwidth}
            \raisebox{-\height}{\includegraphics[width=\textwidth]{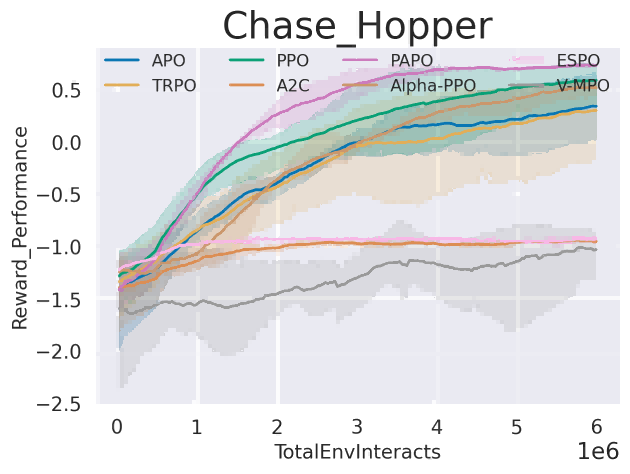}}
        \end{subfigure}
    \end{subfigure}
    \begin{subfigure}[t]{0.19\textwidth}
        \begin{subfigure}[t]{1.00\textwidth}
            \raisebox{-\height}{\includegraphics[width=\textwidth]{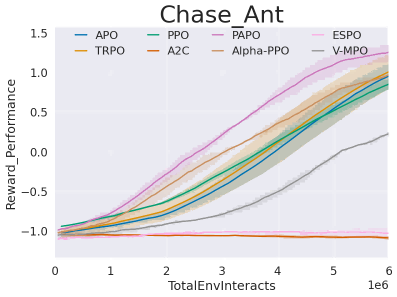}}
        \end{subfigure}
    \end{subfigure}
    \begin{subfigure}[t]{0.19\textwidth}
        \begin{subfigure}[t]{1.00\textwidth}
            \raisebox{-\height}{\includegraphics[width=\textwidth]{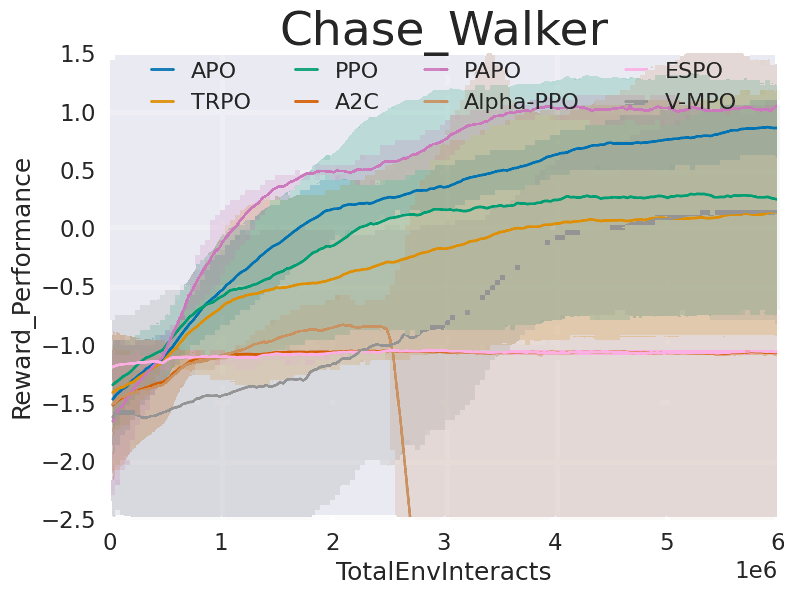}}
        \end{subfigure}
    \end{subfigure}
    \caption{Comparison results of PAPO from five representative test suites in GUARD}
    \label{fig:papo comparison}
\end{figure*}

\begin{figure*}[t]
    \centering
    \begin{subfigure}[t]{0.19\textwidth}
        \begin{subfigure}[t]{1.00\textwidth}
            \raisebox{-\height}{\includegraphics[width=\textwidth]{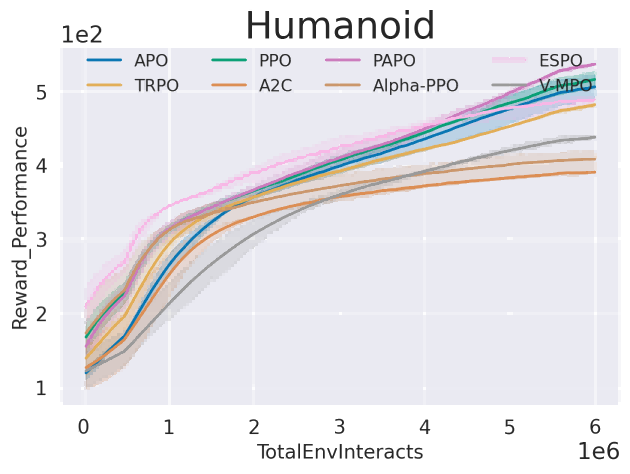}}
        \end{subfigure}
    \end{subfigure}
   \begin{subfigure}[t]{0.19\textwidth}
        \begin{subfigure}[t]{1.00\textwidth}
            \raisebox{-\height}{\includegraphics[width=\textwidth]{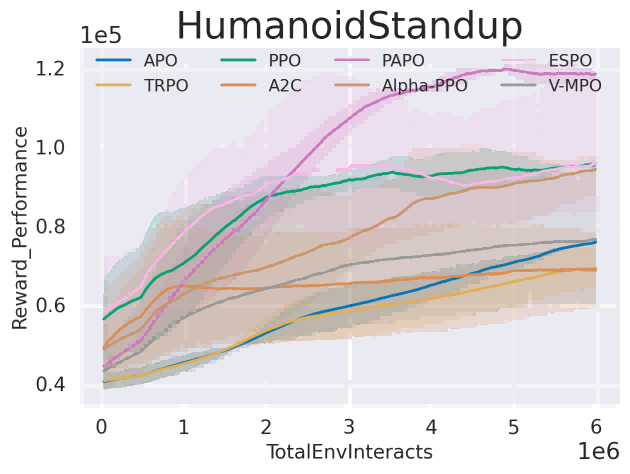}}
        \end{subfigure}
    \end{subfigure}
    \begin{subfigure}[t]{0.19\textwidth}
        \begin{subfigure}[t]{1.00\textwidth}
            \raisebox{-\height}{\includegraphics[width=\textwidth]{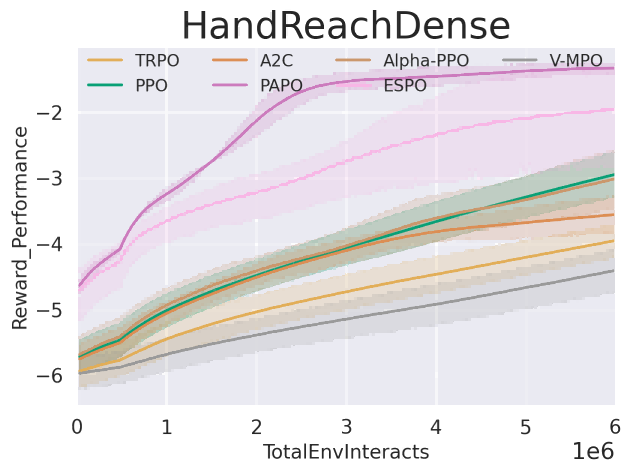}}
        \end{subfigure}
    \end{subfigure}
    \begin{subfigure}[t]{0.19\textwidth}
        \begin{subfigure}[t]{1.00\textwidth}
            \raisebox{-\height}{\includegraphics[width=\textwidth]{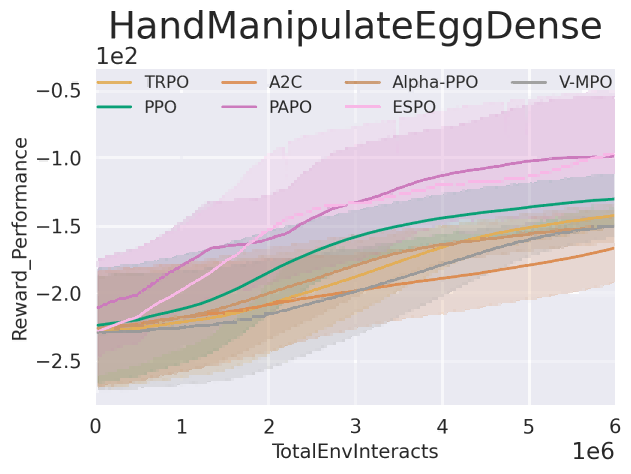}}
        \end{subfigure}
    \end{subfigure}
    \begin{subfigure}[t]{0.19\textwidth}
        \begin{subfigure}[t]{1.00\textwidth}
            \raisebox{-\height}{\includegraphics[width=\textwidth]{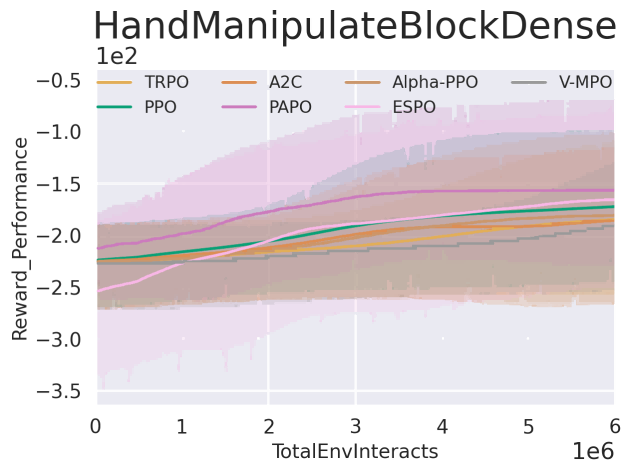}}
        \end{subfigure}
    \end{subfigure}
    \caption{Comparison results of PAPO from five representative test suites in Mujoco and Gymnasium benchmark}
    \label{fig:papo comparison high dim}
    \vspace{-10pt}
\end{figure*}


In the following subsection, we will show a simple strategy can greatly improve the APO computation efficiency and performance.

\begin{wrapfigure}{tr}{0.48\linewidth}
    \centering
    \begin{subfigure}[t]{1.00\linewidth}\raisebox{-\height}{\includegraphics[width=\linewidth]{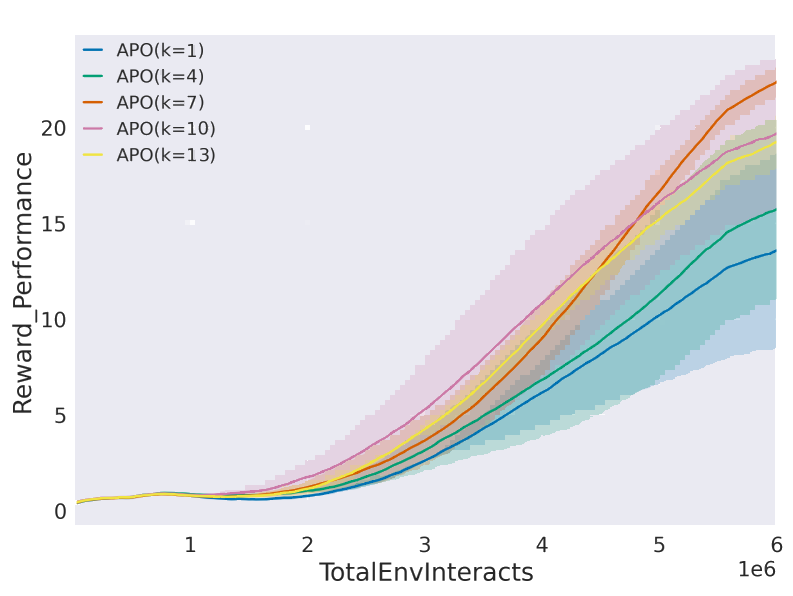}}
    \end{subfigure}
    \caption{Ablation on probability factor $k$}
    \label{fig: k ablation}
    \vspace{-10pt}
\end{wrapfigure}

\subsection{Proximal Absolute Policy Optimization (PAPO)}
With APO's proven success in addressing demanding tasks in both continuous control and Atari game playing, a natural question arises: What potential does APO hold? To shed light on this inquiry, we undertake a natural extension of APO by incorporating a successful variation from TRPO to PPO, i.e. introduction of a \textbf{Clipped Surrogate Objective}~\citep{schulman2017ppo}.

In our experiment, this clipping is implemented with a simple method: \textbf{early stopping}. Specifically, multiple steps of stochastic gradient descent are taken to maximize the objective of \eqref{eq: apo optimization weighted sum version}.
If the mean KL-divergence of the new policy from the old grows beyond a threshold ($\delta$), we stop taking gradient steps. We call this enhanced version of APO as \textit{Proximal Absolute Policy Optimization (PAPO)}.

\subsubsection{Compare PAPO with Baseline Algorithms}

Next we will focus on challenging and high-dimensional tasks in GUARD where trust-region based approaches usually struggle to get good results. The detailed test suites can be found in \Cref{tab: papo testing suites}. The summarized comparison results are presented in \Cref{fig:papo comparison}, highlighting PAPO's superior performance over previous baseline methods, including APO, across all these GUARD challenging control environments. Notably, PAPO significantly improves APO's performance in challenging continuous tasks and outperforms PPO by a learning speed and converged reward, as evidenced by the learning curves.

\subsubsection{Showcase in the Continuous Domain: Mujoco and Gymnasium Benchmark}
To demonstrate the efficacy of PAPO in \textbf{super-high dimensional} continuous benchmarks, we conducted additional experiments in scenarios that were highly integrated with the reality frontier such as dexterous hand manipulation and humanoid robot movement. The results, compared with all baseline methods, including APO, are presented in \Cref{tab:policy_setting} with detailed hyperparameters and learning curves in \Cref{fig:papo comparison high dim}. We can see that the more exploratory proximal methods can achieve better results in this type of super-high dimensional control task compared to the trust region methods. PAPO consistently outperforms all algorithms including all other proximal methods, particularly showcasing enhanced learning efficiency and converged performance in the \textit{HumanoidStandup} task and \textit{HandReach}. Thus the figures and description answer \textbf{Q5}.

\subsection{APO, PAPO Computational Cost Comparison}
\begin{figure}
    \centering
    \begin{subfigure}[b]{0.95\linewidth}
        \begin{subfigure}[t]{1.00\linewidth}\raisebox{-\height}{\includegraphics[width=\linewidth]{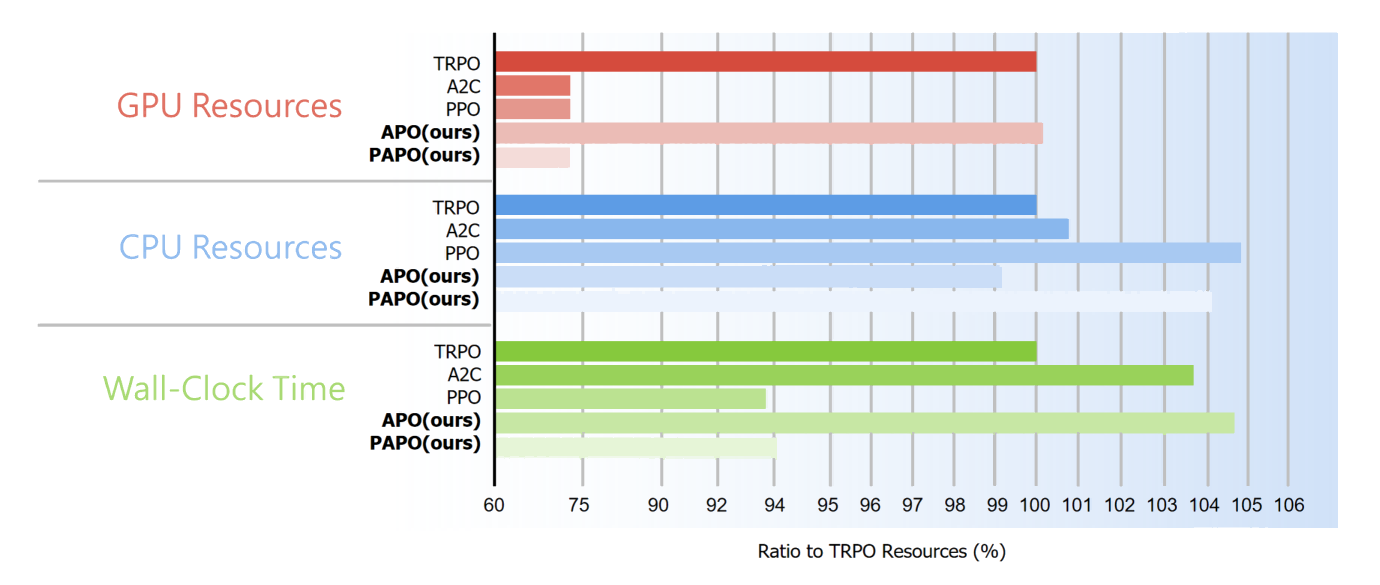}}
        \end{subfigure}
    \end{subfigure}
    \caption{Computational cost(GPU occupancy, CPU occupancy and wall-clock time) comparison} 
    \label{fig: resources}
    \vspace{-20pt}
\end{figure}

We test the comparison of APO, PAPO, PPO, TRPO and A2C in terms of GPU and CPU resource usage, wall-clock time using TRPO as a benchmark and present in \Cref{fig: resources}, where the horizontal axis is the percentage value of the rest of the algorithms compared to TRPO. The experiments are built on Goal\_Ant task and were run and averaged over five different seeds. We can see that APO as a trust region method takes more significant and stable results with essentially the same hardware resource footprint as TRPO and only less than 5\% additional wall-clock time spent, demonstrating that our surrogate absolute bound may seem complex but does not introduce additional unacceptable computational costs. At the same time, there is basically no difference between PAPO and PPO in terms of performance of the indicators, but PAPO achieves better results on both high and low dimensional, complex and simple tasks which demonstrates  the superiority of our method. Additional wall-clock time experiment results can be found in \Cref{fig: apo time,fig: papo time}. Thus, all of these answers \textbf{Q6}.

\subsection{Effective and Less Effective Scenarios for APO and PAPO}
Through extensive experimentation, we have found that APO/PAPO represents a more robust iteration of TRPO/PPO, offering significantly improved and more consistent performance while maintaining the same computational efficiency (sampling efficiency). Consequently, any application scenarios currently utilizing PPO or other on-policy RL methods can readily benefit from the application of APO. Real-world scenarios demanding high performance consistency, such as video games and robotics, are particularly advantageous domains for deploying APO.

For less effective scenarios, APO shares typical limitations inherent to on-policy algorithms. Notably, its sample efficiency may be poor due to the requirement of collecting data exclusively from the current policy's interactions with the environment, making APO impractical for environments where data acquisition is costly or time-consuming. This discussion addresses \textbf{Q7}.

\section{Conclusion and Future Work}
This paper proposed APO, the first general-purpose policy search algorithm that ensures monotonic improvement of the lower probability bound of performance with high confidence. We demonstrate APO's effectiveness on challenging continuous control benchmark tasks and playing Atari games, showing its significant performance improvement 
compared to existing methods and ability to enhance both expected performance and worst-case performance.

Furthermore, we integrate proven techniques that enhance TRPO into Proximal Absolute Policy Optimization (PAPO), resulting in substantial performance improvements on continuous tasks. This effort underscores the promising capability of APO to significantly enhance existing methodologies.

In anticipation of future endeavors, we aspire to leverage APO as the inaugural phase in our exploration of probability bound control employing the trust region method. Subsequently, our objective is to extend this investigation into a more expansive domain within scientific research.


\section*{Impact Statement}
This paper presents work whose goal is to advance the field of Machine Learning. There are many potential societal consequences of our work, none which we feel must be specifically highlighted here.

\section*{Acknowledgements}
This work is partially supported by the National Science Foundation, Grant No. 2144489.

\bibliography{icml2024}
\bibliographystyle{icml2024}

\newpage
\appendix
\onecolumn

\section{Related Works}
\label{sec: related works}
\paragraph{Model-Free Deep Reinforcement Learning}
Model-free deep reinforcement learning (RL) algorithms have found applications from the realm of games~\citep{mnih2013playing,silver2016mastering} to the intricate domain of robotic control~\citep{schulman2015trust}.

The leading contenders of the model free reinforcement learning algorithms include (i) deep Q-learning~\citep{mnih2013dqn,hausknecht2015deep,vanhasselt2015deep,hessel2018rainbow}, (ii) off-policy policy gradient methods~\citep{silver2014deterministic,lillicrap2015continuous,gu2016continuous,fujimoto2018addressing,haarnoja2018soft}, and (iii) trust region on-policy policy gradient methods~\citep{schulman2015trust,schulman2017ppo}. 

Among those categories,  
Q-learning-based techniques
~\citep{mnih2013dqn} 
, augmented with function approximation, 
have exhibited remarkable prowess over tasks with discrete action spaces, e.g. Atari game playing
~\citep{bellemare2013arcade}. 
However, these methods perform poorly in the realm of continuous control benchmarks, notably exemplified in OpenAI Gym~\citep{brockman2016openai,duan2016benchmarking}.

In contrast, off-policy policy gradient methods extend Q-learning-based strategies via introducing an independent actor network to handle continuous control tasks, as exemplified by the Deep Deterministic Policy Gradient (DDPG)\citep{lillicrap2015continuous}. 
However, off-policy methods suffer from stability issues and susceptibility to hyperparameter tuning nuances\citep{haarnoja2018soft}. Recently, enhancements have been made to incorporate entropy to foster exploration~\citep{haarnoja2018soft} and mitigate the overestimation bias through target networks~\citep{fujimoto2018addressing}. 
Additionally, Distributional RL emerges as powerful off-policy methods to model the entire distribution of returns instead of focusing solely on expected values, providing more information about the distribution of rewards, leading to potentially more stable and efficient learning. Early efforts try to model the distribution of value function with various parameterizations, such as quartiles, in algorithms like C51~\citep{marc2017c51}, QR-DQN\cite{will2017qrdqn}, IQN\citep{will2018iqn} and FQF~\cite{derek2019fqf}. Recent advancements address issues like the non-decreasing property of learned quartiles~\cite{Zhou2020NoncrossingQR} and use distributional RL to reduce bias and variance in Q function estimation~\cite{kuang2023qemrl}.
Despite these advancements, the convergence characteristics of off-policy policy gradient methods remain incompletely understood, primarily explored under stringent assumptions such as infinite sampling and Q-updates~\citep{fujimoto2018addressing}. Moreover, off-policy policy gradient methods are primarily tailored for continuous action spaces.
Conversely, trust region on-policy policy gradient methods harmoniously accommodate both continuous and discrete action spaces while showcasing superior stability and dependable convergence properties. Notably, the representative Trust Region Policy Optimization (TRPO)~\citep{schulman2015trust}, complemented by its pragmatic counterpart, Proximal Policy Optimization (PPO)~\citep{schulman2017ppo}, have consistently delivered impressive performance across an array of demanding benchmark tasks. Furthermore, PPO has 
largely helped training of
groundbreaking artificial intelligence applications, including ChatGPT~\citep{schulman2022chatgpt}, the automated Rubik's Cube solver with a robotic hand~\citep{akkaya2019solving}, and the championship-level drone racing~\citep{kaufmann2023champion}, thereby reaffirming their profound impact on advancing the frontiers of AI technology. 

\paragraph{Attempts to Improve Trust Region Methods}
Recently, many efforts are made to improve trust region on-policy methods, including (i) \textit{improve computation efficiency.} TREFree ~\citep{sun2023trust} introduced a novel surrogate objective that eliminates trust region constraints. (ii) \textit{encourage exploration.} COPOS~\citep{pajarinen2019compatible} applied compatible value function approximation to effectively control entropy during policy updates. (iii) \textit{improve training stability and data-efficiency.} Truly PPO (TR-PPO)~\citep{wang2020truly} introduced a new clipping function and trust region-based triggering condition. PPO Smooth (PPOS) ~\citep{zhu2020ppos} use a functional clipping method instead of a flat clipping method. Generalized PPO (GePPO) ~\citep{queeney2021generalized} extended PPO to an off-policy variant, thereby enhancing sampling efficiency through data reuse. Early Stopping Policy Optimization (ESPO) ~\citep{sun2022espo} argued that the clip method in PPO is not reasonable and proposed an early stopping method to replace it. AlphaPPO ~\citep{xu2023improving} introduced alpha divergence, a metric that offers a more effective description of policy differences, resulting in more stable training performance. 

More relevantly, there are improvements considering variance control, including (i) \textit{variance reduction of policy gradient. } \citeauthor{xu2017stochastic} and \citeauthor{papini2018stochastic} applied the stochastic variance reduced gradient descent (SVRG) technique for getting stochastic variance-reduced version of policy gradient (SVRPO) to improve the sample efficiency. \citeauthor{yuanpolicy} incorporates the StochAstic Recursive grAdient algoritHm (SARAH) into the TRPO framework to get more stable variance.  \citeauthor{song2020vmpo} uses on-policy adaptation of Maximum a Posteriori Policy Optimization to replace policy gradients which may have large variance. (ii) \textit{variance reduction of performance update.} ~\citeauthor{tomczak2019policy} introduced a surrogate objective with approximate importance sampling to strike a balance between performance update bias and variance. (iii) \textit{variance reduction of importance sampling.} \citep{lin2023sample} introduced sample dropout to bound the variance of importance sampling estimate by dropping out samples when their ratio deviation is too high.

Although trust-region-based methods have achieved notable success, there hasn't been any approach to exert control over the worst case performance. Unforeseen instances of poor performance can result in training instability, thereby jeopardizing the reliability of solutions in real-world applications. In our research, we bridge this gap by introducing novel theoretical results that ensure a monotonic improvement of the lower bound of near-total performance samples.

Simultaneously, there exist additional domains warranting our consideration, particularly those aligned with the conceptual framework of optimizing worst-case performance. However, it is crucial to acknowledge the fundamental distinctions between these areas and our chosen approach.

\paragraph{Risk-sensitive RL / Probabilistic-constrained RL}
This method primarily operates in the realm of safe RL, focusing on minimizing risk-sensitive criteria like variance-related measures or percentile performance. Common metrics include Value-at-Risk (VaR) and conditional Value-at-Risk (CVaR)~\citep{alexander2004comparison}, aiming to quantify costs in the tail of a distribution. \cite{yinlam2015risk,berkenkamp2017safe,chow2018risk,tang2019worst,jain2021variance,chen2023probabilistic,yu2023global} addressed gradient computation under the Lagrangian function for percentile risk-constrained Markov Decision Processes (MDPs). While the objective of risk-sensitive RL aligns with ours, it is noteworthy that
(i) Our commitment lies in the enhancement of rewards, excluding consideration for safe constraints.
(ii) Diverging from the paradigm of risk-sensitive Reinforcement Learning, we employ surrogate functions and trust region methods to optimize confidence lower bounds. This strategic choice inherits the advantages inherent in trust region methods while preserving the extensibility of our broader research framework.
(iii) Our approach, devoid of intricate assumptions, guarantees the monotonic ascent of the lower probability bound, surpassing mere adherence to constraints.

\paragraph{Robust Reinforcement Learning}
Robust RL research on the topic of mitigating the model discrepancy in the process of model generalization. \cite{peng2017dr} propose domain randomization (DR) method to solve the problem, which randomizes the simulator to generate a variety of environments for training a same policy in the source domain to get better generalization performance. \cite{aravind2016epopt} takes another way to learn policies that are robust to environment perturbations. They train the policy solely on the worst performance subset to get better worst case performance at the cost of sacrificing average performance. \cite{jiang2021mrpo} (MRPO) concurrently improve the average and worst-case performance by derive a lower bound for worst performance which is related to the expected performance. \cite{panaganti2022robust} proposes Robust Fitted Q-Ieration (RFQI) algorithm which uses only offline data and solves problems of data collection, optimization over models and unbiased estimation. \cite{you2022useroriented} integrate user preference into policy learning in robust RL, and propose a novel User-Oriented Robust RL (UOR-RL) framework which gets rid of traditional max-min robustness. It is noteworthy that their emphasis lies on variable environments as opposed to a fixed and stable setting. This nuanced choice implies that their interpretation of "worst performance" pertains to the most adverse outcome arising from performance fluctuations induced by model generalization, diverging from a conventional reliance on the lower bound within a distribution.

\clearpage

\section{Additional Experiment Results}

\begin{figure*}[h]
    \centering
    \begin{subfigure}[t]{0.215\textwidth}
        \begin{subfigure}[t]{1.00\textwidth}
            \raisebox{-\height}{\includegraphics[width=\textwidth]{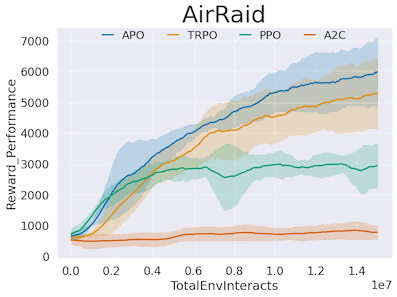}}
        \end{subfigure}
    \end{subfigure}
    \begin{subfigure}[t]{0.215\textwidth}
        \begin{subfigure}[t]{1.00\textwidth}
            \raisebox{-\height}{\includegraphics[width=\textwidth]{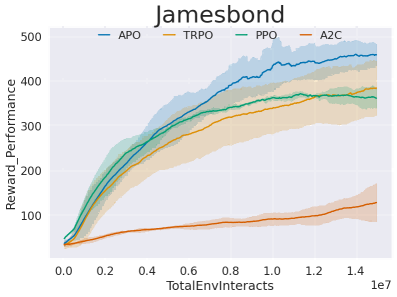}}
        \end{subfigure}
    \end{subfigure}
    \begin{subfigure}[t]{0.215\textwidth}
        \begin{subfigure}[t]{1.00\textwidth}
            \raisebox{-\height}{\includegraphics[width=\textwidth]{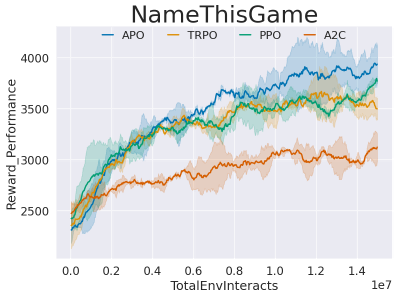}}
        \end{subfigure}
    \end{subfigure}
   \begin{subfigure}[t]{0.215\textwidth}
        \begin{subfigure}[t]{1.00\textwidth}
            \raisebox{-\height}{\includegraphics[width=\textwidth]{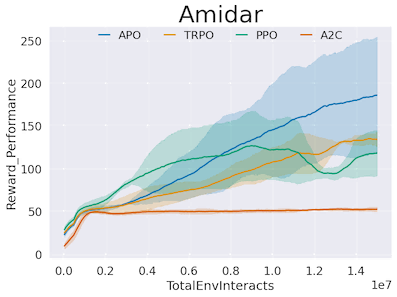}}
        \end{subfigure}
    \end{subfigure}
    \begin{subfigure}[t]{0.215\textwidth}
        \begin{subfigure}[t]{1.00\textwidth}
            \raisebox{-\height}{\includegraphics[width=\textwidth]{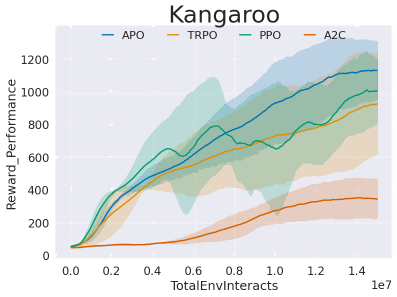}}
        \end{subfigure}
    \end{subfigure}
    \begin{subfigure}[t]{0.215\textwidth}
        \begin{subfigure}[t]{1.00\textwidth}
            \raisebox{-\height}{\includegraphics[width=\textwidth]{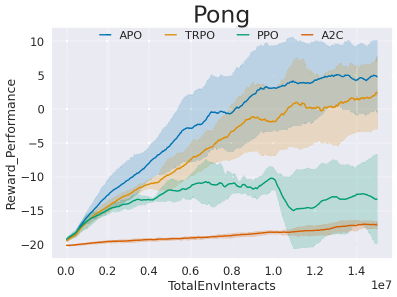}}
        \end{subfigure}
    \end{subfigure}
    \begin{subfigure}[t]{0.215\textwidth}
        \begin{subfigure}[t]{1.00\textwidth}
            \raisebox{-\height}{\includegraphics[width=\textwidth]{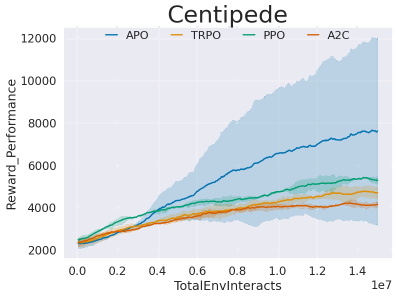}}
        \end{subfigure}
    \end{subfigure}
    \begin{subfigure}[t]{0.215\textwidth}
        \begin{subfigure}[t]{1.00\textwidth}
            \raisebox{-\height}{\includegraphics[width=\textwidth]{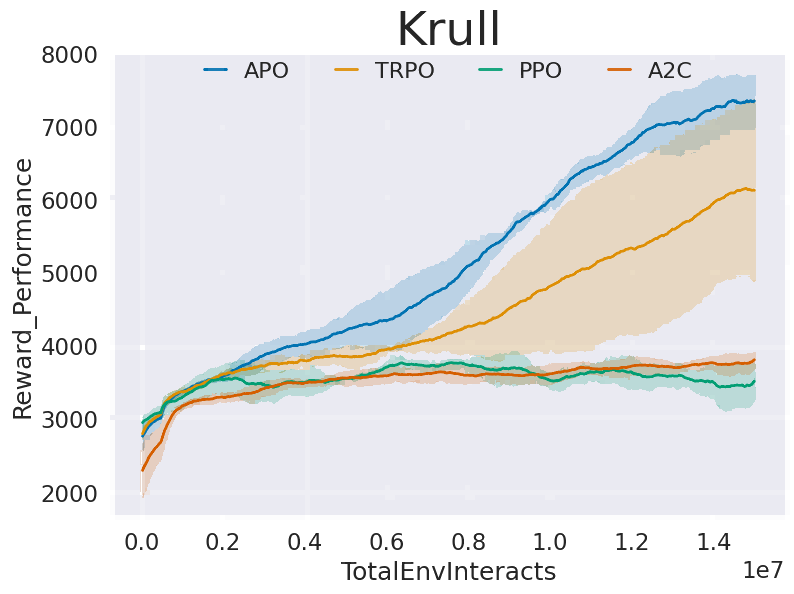}}
        \end{subfigure}
    \end{subfigure}
    \begin{subfigure}[t]{0.215\textwidth}
        \begin{subfigure}[t]{1.00\textwidth}
            \raisebox{-\height}{\includegraphics[width=\textwidth]{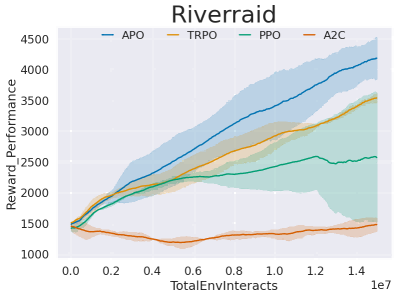}}
        \end{subfigure}
    \end{subfigure}
    \begin{subfigure}[t]{0.215\textwidth}
        \begin{subfigure}[t]{1.00\textwidth}
            \raisebox{-\height}{\includegraphics[width=\textwidth]{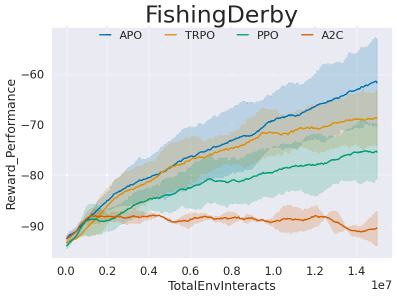}}
        \end{subfigure}
    \end{subfigure}
    \begin{subfigure}[t]{0.215\textwidth}
        \begin{subfigure}[t]{1.00\textwidth}
            \raisebox{-\height}{\includegraphics[width=\textwidth]{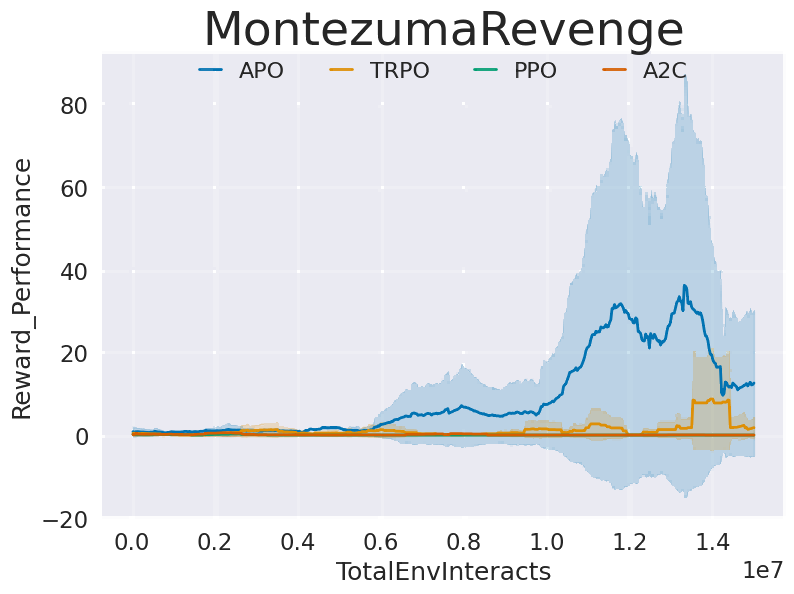}}
        \end{subfigure}
    \end{subfigure}
    \begin{subfigure}[t]{0.215\textwidth}
        \begin{subfigure}[t]{1.00\textwidth}
            \raisebox{-\height}{\includegraphics[width=\textwidth]{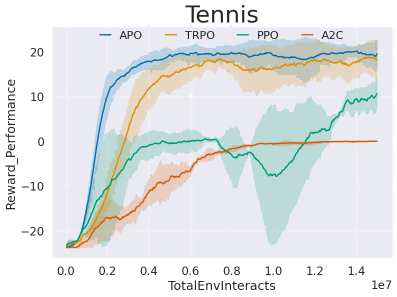}}
        \end{subfigure}
    \end{subfigure}
    \caption{Comparison results from part of the total test suites in Atari Domain}
    \label{fig: selected atari}
\end{figure*}

\begin{figure*}[h]
    \centering
    \begin{subfigure}[t]{0.215\textwidth}
        \begin{subfigure}[t]{1.00\textwidth}
            \raisebox{-\height}{\includegraphics[width=\textwidth]{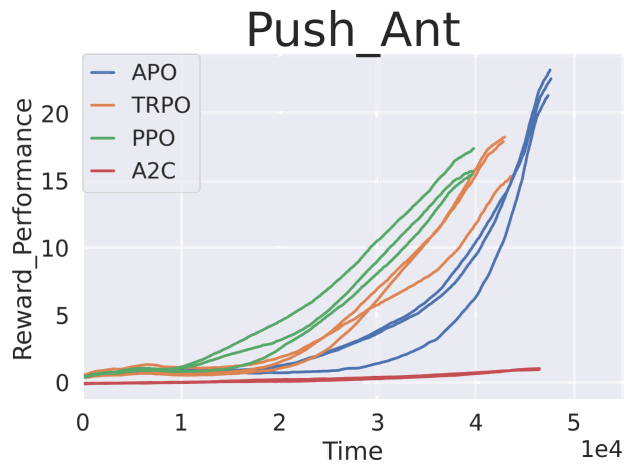}}
        \end{subfigure}
    \end{subfigure}
   \begin{subfigure}[t]{0.215\textwidth}
        \begin{subfigure}[t]{1.00\textwidth}
            \raisebox{-\height}{\includegraphics[width=\textwidth]{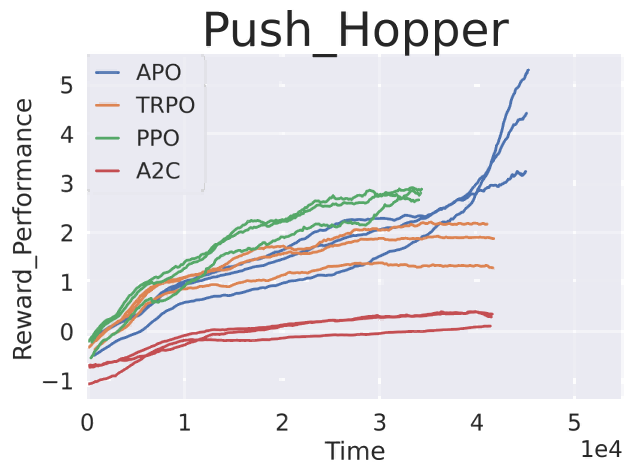}}
        \end{subfigure}
    \end{subfigure}
    \begin{subfigure}[t]{0.215\textwidth}
        \begin{subfigure}[t]{1.00\textwidth}
            \raisebox{-\height}{\includegraphics[width=\textwidth]{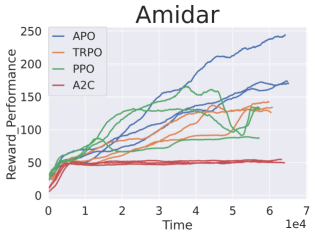}}
        \end{subfigure}
    \end{subfigure}
    \begin{subfigure}[t]{0.215\textwidth}
        \begin{subfigure}[t]{1.00\textwidth}
            \raisebox{-\height}{\includegraphics[width=\textwidth]{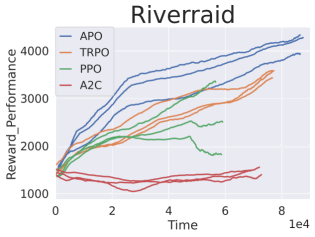}}
        \end{subfigure}
    \end{subfigure}
    \caption{Comparison of wall-clock time from four representative test suites in continuous and discrete domain}
    \label{fig: apo time}
\end{figure*}

\begin{figure*}[h]
    \centering
    \begin{subfigure}[b]{0.215\linewidth}
        \begin{subfigure}[t]{1.00\linewidth}\raisebox{-\height}{\includegraphics[width=\linewidth]{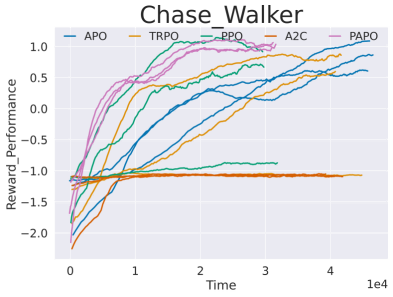}}
        \end{subfigure}
    \end{subfigure}
    \hspace{0.01\linewidth}
    \begin{subfigure}[b]{0.215\linewidth}
        \begin{subfigure}[t]{1.00\linewidth}\raisebox{-\height}{\includegraphics[width=\linewidth]{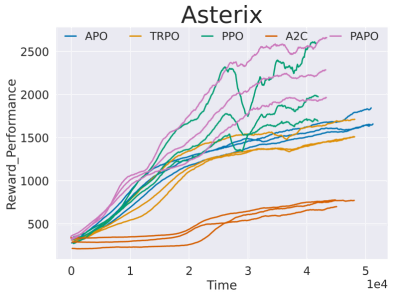}}
        \end{subfigure}
    \end{subfigure}
    \caption{Comparison of wall-clock time of PAPO} 
    \label{fig: papo time}
\end{figure*}

\clearpage
\section{Lower Probability Bound}
\label{proof: absolute performance improvement}

\subsection{Preliminaries}
First we define $R_\pi(s) = \sum_{t=0}^\infty \gamma^t R(s_t, a_t, s_{t+1})$ as infinite-horizon discounted return starts at state $s$ and define expected return ${V_\pi(s) = \underset{\hat\tau \sim \pi}{ \mathbb{E}}\big(R_\pi(s)\big)}$  as the value of state $s$. Then for all trajectories $\hat\tau \sim \pi$ start from state ${s_0 \sim \mu}$, the expectation and variance of  $R_\pi(s_0)$ can be respectively defined as  $\mathcal{J}(\pi)$ and $\mathcal{V}(\pi)$. Formally:

\begin{align}
    \mathcal{J}(\pi) & = \underset{\substack{s_0 \sim \mu \\ {\hat \tau}\sim {\pi}}}{\mathbb{E}}\bigg[R_\pi(s_0)\bigg] 
    = \underset{\substack{s_0 \sim \mu}}{\mathbb{E}} \bigg[V_\pi(s_0)\bigg] \\  
    \mathcal{V}(\pi) &= \underset{\substack{s_0 \sim \mu \\ {\hat \tau}\sim {\pi}}}{\mathbb{E}} \bigg[\big(R_\pi(s_0) - \mathcal{J}(\pi)\big)^2\bigg] \\ \nonumber 
    &= \underset{\substack{s_0 \sim \mu}}{\mathbb{E}}\bigg[\underset{\substack{{\hat \tau}\sim {\pi}}}{\mathbb{V}ar}\big[R_\pi(s_0)\big] +\bigg[\underset{{\hat \tau} \sim \pi}{ \mathbb{E}}\big[R_\pi(s_0)\big]\bigg]^2\bigg] - \mathcal{J}(\pi)^2  \\ \nonumber 
    &= \underset{\substack{s_0 \sim \mu}}{\mathbb{E}}\bigg[\underset{\substack{{\hat \tau}\sim {\pi}}}{\mathbb{V}ar}\big[R_\pi(s_0)\big] + V_\pi(s_0)^2\bigg] - \mathcal{J}(\pi)^2  \\ \nonumber 
    &= \underset{\substack{s_0 \sim \mu}}{\mathbb{E}}\bigg[\underset{\substack{{\hat \tau}\sim {\pi}}}{\mathbb{V}ar}\big[R_\pi(s_0)\big]\bigg] + \underset{\substack{s_0 \sim \mu}}{\mathbb{E}}\bigg[V_\pi(s_0)^2\bigg] - \mathcal{J}(\pi)^2  \\ \nonumber 
    &= \underbrace{\underset{\substack{s_0 \sim \mu}}{\mathbb{E}}\bigg[\underset{\substack{{\hat \tau}\sim {\pi}}}{\mathbb{V}ar}\big[R_\pi(s_0)\big]\bigg]}_{MeanVariance} + \underbrace{\underset{s_0 \sim \mu}{\mathbb{V}ar} [V_\pi(s_0)]}_{VarianceMean} \label{eq:variance interpretation}
\end{align}
Note that for the derivation of ${\mathcal{V}(\pi)}$ we treat the return of all trajectories as a mixture of one-dimensional distributions. Each distribution consists of the returns of trajectories from the same start state. The variance can then be divided into two parts: \\
1. \textbf{MeanVariance} reflects the expected variance of the return over different start states. \\
2. \textbf{VarianceMean} reflects the variance of the average return of different start states. 

\begin{prop}
\label{prop: absolute bound definition}
$\mathcal{B}_k(\pi) = \mathcal{J}(\pi)-k\mathcal{V}(\pi)$ is guaranteed to be an lower probability bound of performance defined by \Cref{def: abs bound}.
\begin{proof}
According to Selberg's inequality theory~\cite{saw1984chebyshev}, if random variable $R_\pi(s_0)$ has finite non-zero variance $\mathcal{V}(\pi)$ and finite expected value $\mathcal{J}(\pi)$. Then for any real number $k\geq0$, following inequality holds:
\begin{align}
    Pr(R_\pi(s_0) < \mathcal{J}(\pi)-k\mathcal{V}(\pi)) ) \le \frac{1}{k^2\mathcal{V}(\pi)+1}~~~,
\end{align}
which equals to:
\begin{align}
    Pr\big(R_\pi(s_0)\geq \mathcal{B}_k(\pi)\big) \geq 1 - \frac{1}{k^2\mathcal{V}(\pi)+1}~~~.
\end{align}\\
Considering that $\pi\in\Pi$, then $\mathcal{V}(\pi)$ belongs to a corresponding variance space which has a non-zero minima of $\mathcal{V}(\pi)$ denoted as $\mathcal{V}_{min} \in \mathbb{R}^+$. Therefore, by treating $\psi = \mathcal{V}_{min}$, the following condition holds:
\begin{align}
        \text{There exists $\psi > 0$, s.t. }Pr\big(R_\pi(s_0)\geq \mathcal{B}_k(\pi)\big) &\geq 1 - \frac{1}{k^2\mathcal{V}(\pi)+1} \\ \nonumber
        &\geq 1 - \frac{1}{k^2\mathcal{V}_{min}+1} \\ \nonumber 
        &= \underbrace{1 - \frac{1}{k^2\psi+1}}_{p_k^\psi} ~~~.
    \end{align}
\end{proof}
\end{prop}

Now to prove \Cref{theo: absolute performance improvement}, we need to prove the optimization object in \eqref{eq: apo optimization final} serves as the lower bound of $\mathcal{B}_k(\pi')$. To do this, we first try to obtain the following terms:\\
1. Upper bound of \textbf{MeanVariance} with ${\pi'}$. (\Cref{proof: MeanVariance Bound})\\
2. Upper bound of \textbf{VarianceMean} with ${\pi'}$. (\Cref{proof: VarianceMean Bound})\\
3. Lower bound of ${\mathcal{J}(\pi')}$.(\Cref{proof: Expectation Bound})\\
And then we can proof our theorem by leveraging this lower bound in \Cref{proof: theo 1}



\subsection{MeanVariance Bound}
\label{proof: MeanVariance Bound}
\begin{prop}[Bound of MeanVariance]
\label{lem: bound of MV}
Denote \textbf{MeanVariance} of policy ${\pi}$ as $MV_\pi = \underset{\substack{s_0 \sim \mu}}{\mathbb{E}}[\mathbb{V}ar[R_\pi(s_0)]$. Given two policies $\pi', \pi$, the following bound holds:
\begin{align}
    |MV_{\pi'} - MV_{\pi}| & \leq   \|\mu^\top\|_\infty\bigg(\frac{1}{1-\gamma^2}\underset{s}{\mathbf{max}}\bigg|\underset{\substack{\\a\sim\pi\\s'\sim P}}{\mathbb{E}}\left[\left(\frac{\pi^{\prime}(a|s)}{\pi(a|s)}-1\right) A_{\pi}(s,a,s')^2\right]\\ \nonumber
    &+ 2\underset{\substack{a \sim \pi \\ s' \sim P}}{\mathbb{E}}\left[\left(\frac{\pi^{\prime}(a|s)}{\pi(a|s)}\right)A_{\pi}(s,a,s')\right]|H(s,a,s')|_{max} + |H(s,a,s')|_{max}^2\bigg| \\ \nonumber
    &+\frac{2\gamma^2}{(1-\gamma^2)^2}\sqrt{\frac{1}{2}\mathcal{D}_{KL}^{max}(\pi'\|\pi)}\cdot\|\bm \Omega_\pi\|_\infty\bigg)
\end{align}
where $\bm \Omega_\pi = \begin{bmatrix}
    \omega_\pi(s^1) \\
    \omega_\pi(s^2) \\
    \vdots
\end{bmatrix}$ and $\omega_\pi(s) = \underset{\substack{a \sim \pi \\ s' \sim P}}{\mathbb{E}} \big[Q_\pi(s,a,s')^2\big] - V_{\pi}(s)^2$ is defined as the variance of the state-action value function ${Q_\pi}$ at state ${s}$. Additionally, 
\begin{align}
    |H(s,a,s')|_{max} &= \left|L(s,a,s')\right| + \dfrac{2(1+\gamma)\gamma\epsilon}{(1-\gamma)^2}\mathcal{D}_{KL}^{max}(\pi'\|\pi)\\ \nonumber
    L(s,a,s')&= \gamma\underset{\substack{s_0 = s' \\{\hat \tau} \sim \pi}}{\mathbb{E}}\bigg[\sum_{t=0}^\infty \gamma^t \bar A_{\pi',\pi}(s_t)\bigg] - \underset{\substack{s_0 = s \\{\hat \tau} \sim \pi}}{\mathbb{E}}\bigg[\sum_{t=0}^\infty \gamma^t\bar A_{\pi',\pi}(s_t)\bigg]\\ \nonumber
    \epsilon &= \underset{s, a}{\mathbf{max}}|A_\pi(s,a)|
\end{align}
where $\mathcal{D}_{KL}^{max}(\pi'\|\pi) = \max_s \mathcal{D}_{KL}(\pi'\|\pi)[s]$.
\end{prop}

\begin{proof}
According to \citep{sobel1982variance}, the following Proposition holds
\begin{prop}[Theorem 1, \citep{sobel1982variance}]
\label{prop: V result}
Define $\bm X_\pi = \begin{bmatrix} \mathbb{V}ar[R_\pi(s^1)] \\ \mathbb{V}ar[R_\pi(s^2)] \\ \vdots \end{bmatrix}$ and ${\hat P}_\pi = P_\pi^\top$, where $\hat{P}_\pi(i, j)$ denotes the probability of the transfer from i-th state to j-th state, the following equation holds 
\begin{align}
    \bm X_\pi = (I - \gamma^2 {\hat P}_\pi)^{-1}\bm\Omega_\pi~~.
\end{align}
\end{prop}

With $\bm X_\pi$, \textbf{MeanVariance} can be computed via
\begin{align}
\label{eq: ME}
    MV_{\pi} &=\underset{\substack{s_0 \sim \mu}}{\mathbb{E}}[\mathbb{V}ar[R_\pi(s_0)] \\ \nonumber
    &=\mu^\top \bm X_\pi \\ \nonumber
    &= \mu^\top \big((I - \gamma^2 {\hat P}_\pi)^{-1}\bm\Omega_\pi\big)\\ \nonumber
\end{align}

The divergence of \textbf{MeanVariance} we want to bound can be written as:
\begin{align}
    |MV_{\pi'} - MV_{\pi}| &=\big|\underset{\substack{s_0 \sim \mu}}{\mathbb{E}}[\mathbb{V}ar[R_{\pi'}(s_0)] - \underset{\substack{s_0 \sim \mu}}{\mathbb{E}}[\mathbb{V}ar[R_{\pi}(s_0)]\big| \label{eq: mv - mv ori} \\ \nonumber 
    &= \|\mu^\top \big(\bm X_{\pi'} - \bm X_\pi) \|_\infty\\ \nonumber
    &\leq \|\mu^\top\|_\infty  \|\bm X_{\pi'} - \bm X_\pi\|_\infty 
    \end{align}
To bound ${\|\bm X_{\pi'} - \bm X_\pi\|_\infty}$, consider the following conditions: 
\begin{align}
    \bm X_\pi - \gamma^2 {\hat P}_\pi \bm X_\pi = \bm\Omega_\pi \\ \nonumber 
    \bm X_{\pi'} - \gamma^2 {\hat P}_{\pi'} \bm X_{\pi'} = \bm\Omega_{\pi'} 
\end{align}

Let ${G_\pi = (I - \gamma^2 {\hat P}_\pi)^{-1}}$, then ${\bm X_\pi}$ can be written as:
\begin{align}
    \bm X_\pi= G_\pi \bm\Omega_\pi \\ \nonumber 
\end{align}
Then we have:
\begin{align}
\label{eq: v v diff}
    \bm X_{\pi'} - \bm X_\pi &= G_{\pi'} \bm\Omega_{\pi'} - G_\pi \bm\Omega_\pi\\\nonumber
    &= G_{\pi'} \bm\Omega_{\pi'} - G_{\pi'} \bm\Omega_\pi + G_{\pi'} \bm\Omega_\pi - G_\pi \bm\Omega_\pi\\\nonumber
    &= G_{\pi'}\big(\bm\Omega_{\pi'} - \bm\Omega_\pi\big) + \big(G_{\pi'} - G_\pi \big)\bm\Omega_\pi\\\nonumber
\end{align}
Now ${\|\bm X_{\pi'} - \bm X_\pi\|_\infty}$ can be bounded by:
\begin{align}
    \|\bm X_{\pi'} - \bm X_\pi\|_\infty &\leq \|G_{\pi'}\|_\infty\|\bm\Omega_{\pi'} - \bm\Omega_\pi\|_\infty + \|G_{\pi'} - G_\pi\|_\infty\|\bm\Omega_\pi\|_\infty~~.
\end{align}
Notice that $\|G_\pi\|_\infty$ is bounded by:
\begin{align}
    \|G_\pi\|_\infty = \| (I - \gamma^2 {\hat P}_\pi)^{-1} \|_\infty \leq \sum_{t=0}^\infty (\gamma^2)^t \| {\hat P}_\pi\|_\infty^t = (1 - \gamma^2)^{-1}~~.
\end{align}

Therefore, ${||G_\pi||_\infty = ||G_{\pi'}||_\infty = (1- \gamma^2)^{-1}}$ and ${\|\bm\Omega_\pi\|_\infty}$ can be obtained with ${\bm\Omega_\pi}$. We only need to tackle with ${\|G_{\pi'} - G_\pi\|_\infty}$ and ${\|\bm\Omega_{\pi'} - \bm\Omega_\pi\|_\infty}$.\\
To get ${\|G_{\pi'} - G_\pi\|_\infty}$, we have
\begin{align}
    G_\pi^{-1} - G_{\pi'}^{-1} &= \big(I - \gamma^2 {\hat P}_\pi) - \big(I - \gamma^2 {\hat P}_{\pi'}) \\\nonumber
    &=\gamma^2 \big({\hat P}_{\pi'} - {\hat P}_\pi\big)\\\nonumber
    &=\gamma^2\Delta\\\nonumber
\end{align}
where ${\Delta = {\hat P}_{\pi'} - {\hat P}_\pi}$.
Then we have
\begin{align}
    G_{\pi'} - G_\pi &=\gamma^2G_{\pi'}\Delta G_\pi\\\nonumber
\end{align}
With ${\|\Delta\|_\infty = 2\mathcal{D}_{TV}^{max}(\pi'\|\pi)}$, where $\mathcal{D}_{TV}^{max}(\pi'\|\pi) = \max_s \mathcal{D}_{TV}(\pi'\|\pi)[s]$, we have
\begin{align}
    \|\gamma^2G_{\pi'}\Delta G_\pi\|_\infty &\leq \gamma^2\|G_{\pi'}\|_\infty\|\Delta\|_\infty\|G_{\pi'}\|_\infty\\\nonumber
    &=\gamma^2\cdot\dfrac{1}{1-\gamma^2}\cdot2\mathcal{D}_{TV}^{max}(\pi'\|\pi)\cdot\dfrac{1}{1-\gamma^2}\\\nonumber
    &=\dfrac{2\gamma^2}{(1-\gamma^2)^2}\mathcal{D}_{TV}^{max}(\pi'\|\pi)\\\nonumber
\end{align}
According to \cite{pinskerInequ}, $\mathcal{D}_{TV}^{max}(\pi'\|\pi)\leq\sqrt{\frac{1}{2}\mathcal{D}_{KL}^{max}(\pi'\|\pi)}$, and given \eqref{eq: mv - mv ori}, we have
\begin{align}
    \label{eq: MV_1}
    |MV_{\pi'} - MV_{\pi}| &= \big|\underset{\substack{s_0 \sim \mu}}{\mathbb{E}}[\mathbb{V}ar[R_{\pi'}(s_0)] - \underset{\substack{s_0 \sim \mu}}{\mathbb{E}}[\mathbb{V}ar[R_{\pi}(s_0)]\big| \\ \nonumber
    &\leq \|\mu^\top\|_\infty  \|\bm X_{\pi'} - \bm X_\pi\|_\infty \\ \nonumber
    &\leq  \|\mu^\top\|_\infty\big(\dfrac{1}{1-\gamma^2}\|\bm\Omega_{\pi'} - \bm\Omega_\pi\|_\infty + \dfrac{2\gamma^2}{(1-\gamma^2)^2}\mathcal{D}_{TV}^{max}(\pi'\|\pi)\|\bm\Omega_\pi\|_\infty\big)\\ \nonumber
    &\leq  \|\mu^\top\|_\infty\big(\dfrac{1}{1-\gamma^2}\|\bm\Omega_{\pi'} - \bm\Omega_\pi\|_\infty + \dfrac{2\gamma^2}{(1-\gamma^2)^2}\sqrt{\frac{1}{2}\mathcal{D}_{KL}^{max}(\pi'\|\pi)}\cdot\|\bm\Omega_\pi\|_\infty\big)
\end{align}

To address ${\|\bm\Omega_{\pi'} - \bm\Omega_\pi\|}$, we notice that $\omega_\pi(s) = \underset{\substack{a \sim \pi \\ s' \sim P}}{\mathbb{V}ar}[Q_\pi(s,a,s')] = \underset{\substack{a \sim \pi \\ s' \sim P}}{\mathbb{V}ar}[A_\pi(s,a,s')]$, which means:
\begin{align}
    &\|\bm\Omega_{\pi'} -\bm\Omega_\pi\|_\infty = \underset{s}{\mathbf{max}}\bigg|\underset{\substack{a \sim \pi' \\ s' \sim P}}{\mathbb{V}ar}[A_{\pi'}(s,a,s')] - \underset{\substack{a \sim \pi \\ s' \sim P}}{\mathbb{V}ar}[A_\pi(s,a,s')]\bigg|\\ \nonumber
\end{align}
Where
\begin{align}
    A_\pi(s,a,s')&= R(s, a, s') + \gamma V_{\pi}(s') - V_{\pi}(s)\\ \nonumber
    A_{\pi'}(s,a,s')&= R(s, a, s') + \gamma V_{\pi'}(s') - V_{\pi'}(s)\\ \nonumber
\end{align}
Define $H(s,a,s')=A_{\pi'}(s,a,s') - A_{\pi}(s,a,s')$, we have:
\begin{align}
    H(s,a,s')&= \gamma( V_{\pi'}(s') - V_{\pi}(s')) - (V_{\pi'}(s) - V_{\pi}(s))\\ \nonumber
\end{align}
Similar to TRPO~\cite{schulman2015trust}:
\begin{align}
\label{eq: Expectation of A = -v + v}
    &\underset{\substack{s_0 = s \\{\hat \tau} \sim \pi'}}{\mathbb{E}}\bigg[\sum_{t=0}^\infty \gamma^tA_{\pi}(s_t,a_t,s_{t+1})\bigg]\\ \nonumber
    &=\underset{\substack{s_0 = s \\{\hat \tau} \sim \pi'}}{\mathbb{E}}\bigg[\sum_{t=0}^\infty \gamma^t(R_{\pi}(s_t,a_t,s_{t+1})+\gamma V_{\pi}(s_{t+1})-V_{\pi}(s_{t}))\bigg]\\ \nonumber
    &=\underset{\substack{s_0 = s \\{\hat \tau} \sim \pi'}}{\mathbb{E}}\bigg[-V_{\pi}(s_{0}) +\sum_{t=0}^\infty \gamma^tR_{\pi}(s_t,a_t,s_{t+1})\bigg]\\ \nonumber
    &=\underset{\substack{s_0 = s }}{\mathbb{E}}\bigg[-V_{\pi}(s_{0})\bigg] + \underset{\substack{s_0 = s \\{\hat \tau} \sim \pi'}}{\mathbb{E}}\bigg[\sum_{t=0}^\infty \gamma^tR_{\pi}(s_t,a_t,s_{t+1})\bigg]\\ \nonumber
    &=-V_{\pi}(s) + V_{\pi'}(s)\\ \nonumber
\end{align}
Then $H(s,a,s')$ can be written as:
\begin{align}
    H(s,a,s')&= \gamma\underset{\substack{s_0 = s' \\{\hat \tau} \sim \pi'}}{\mathbb{E}}\bigg[\sum_{t=0}^\infty \gamma^tA_{\pi}(s_t,a_t,s_{t+1})\bigg] - \underset{\substack{s_0 = s \\{\hat \tau} \sim \pi'}}{\mathbb{E}}\bigg[\sum_{t=0}^\infty \gamma^tA_{\pi}(s_t,a_t,s_{t+1})\bigg]\\ \nonumber
\end{align}
Define ${\bar A_{\pi',\pi}(s)}$ to be the expected advantage of ${\pi'}$ over ${\pi}$ at state ${s}$:
\begin{align}
    \bar A_{\pi',\pi}(s) =  \underset{a \sim \pi'}{\mathbb{E}}\bigg[A_{\pi}(s,a)\bigg]\\ \nonumber
\end{align}
Now $H(s,a,s')$ can be written as:
\begin{align}
    H(s,a,s')&= \gamma\underset{\substack{s_0 = s' \\{\hat \tau} \sim \pi'}}{\mathbb{E}}\bigg[\sum_{t=0}^\infty \gamma^t \bar A_{\pi',\pi}(s_t)\bigg] - \underset{\substack{s_0 = s \\{\hat \tau} \sim \pi'}}{\mathbb{E}}\bigg[\sum_{t=0}^\infty \gamma^t\bar A_{\pi',\pi}(s_t)\bigg]\\ \nonumber
\end{align}
Define ${L(s,a,s')}$ as:
\begin{align}
    L(s,a,s')&= \gamma\underset{\substack{s_0 = s' \\{\hat \tau} \sim \pi}}{\mathbb{E}}\bigg[\sum_{t=0}^\infty \gamma^t \bar A_{\pi',\pi}(s_t)\bigg] - \underset{\substack{s_0 = s \\{\hat \tau} \sim \pi}}{\mathbb{E}}\bigg[\sum_{t=0}^\infty \gamma^t\bar A_{\pi',\pi}(s_t)\bigg]\\ \nonumber
\end{align}
With $\epsilon = \mathbf{max}_{s,a}|A_\pi(s,a)|$, 
we have:
\begin{align}
    &|H(s,a,s') - L(s,a,s')| \\ \nonumber
    &= \bigg|\gamma\bigg(\underset{\substack{s_0 = s' \\{\hat \tau} \sim \pi'}}{\mathbb{E}}\bigg[\sum_{t=0}^\infty \gamma^t \bar A_{\pi',\pi}(s_t)\bigg] - \underset{\substack{s_0 = s' \\{\hat \tau} \sim \pi}}{\mathbb{E}}\bigg[\sum_{t=0}^\infty \gamma^t \bar A_{\pi',\pi}(s_t)\bigg]\bigg) \\ \nonumber
    &- \bigg(\underset{\substack{s_0 = s \\{\hat \tau} \sim \pi'}}{\mathbb{E}}\bigg[\sum_{t=0}^\infty \gamma^t\bar A_{\pi',\pi}(s_t)\bigg] - \underset{\substack{s_0 = s \\{\hat \tau} \sim \pi}}{\mathbb{E}}\bigg[\sum_{t=0}^\infty \gamma^t\bar A_{\pi',\pi}(s_t)\bigg]\bigg)\bigg|\\ \nonumber
    &\leq \gamma\bigg|\underset{\substack{s_0 = s' \\{\hat \tau} \sim \pi'}}{\mathbb{E}}\bigg[\sum_{t=0}^\infty \gamma^t \bar A_{\pi',\pi}(s_t)\bigg] - \underset{\substack{s_0 = s' \\{\hat \tau} \sim \pi}}{\mathbb{E}}\bigg[\sum_{t=0}^\infty \gamma^t \bar A_{\pi',\pi}(s_t)\bigg]\bigg| \\ \nonumber
    &+ \bigg|\underset{\substack{s_0 = s \\{\hat \tau} \sim \pi'}}{\mathbb{E}}\bigg[\sum_{t=0}^\infty \gamma^t\bar A_{\pi',\pi}(s_t)\bigg] - \underset{\substack{s_0 = s \\{\hat \tau} \sim \pi}}{\mathbb{E}}\bigg[\sum_{t=0}^\infty \gamma^t\bar A_{\pi',\pi}(s_t)\bigg]\bigg|\\ \nonumber
    &\leq \dfrac{4\gamma(1+\gamma)\epsilon}{(1-\gamma)^2}(\mathcal{D}_{TV}^{max}(\pi'\|\pi))^2 \quad \leftarrow \texttt{([Lemma3, \citep{schulman2015trust}])}
\end{align}

Then according to \cite{pinskerInequ} $\mathcal{D}_{TV}^{max}(\pi'\|\pi)\leq\sqrt{\frac{1}{2}\mathcal{D}_{KL}^{max}(\pi'\|\pi)}$, we can then bound ${|H(s,a,s')|}$ with:
\begin{align}
    |H(s,a,s')| &\leq \left|L(s,a,s')\right| + \dfrac{4\gamma(1+\gamma)\epsilon}{(1-\gamma)^2}(\mathcal{D}_{TV}^{max}(\pi'\|\pi))^2 \\ \nonumber
    &\leq \left|L(s,a,s')\right| + \dfrac{2\gamma(1+\gamma)\epsilon}{(1-\gamma)^2}\mathcal{D}_{KL}^{max}(\pi'\|\pi) \doteq |H(s,a,s')|_{max}
\end{align}

With $A_{\pi'}(s,a,s') = A_{\pi}(s,a,s') + H(s,a,s')$, we have:

\begin{align}
    &\underset{\substack{a \sim \pi' \\ s' \sim P}}{\mathbb{V}ar}[A_{\pi'}(s,a,s')] - \underset{\substack{a \sim \pi \\ s' \sim P}}{\mathbb{V}ar}[A_\pi(s,a,s')]\\ \nonumber
    &=\underset{\substack{a \sim \pi' \\ s' \sim P}}{\mathbb{E}}[A_{\pi'}(s,a,s')^2] - \underset{\substack{a \sim \pi \\ s' \sim P}}{\mathbb{E}}[A_{\pi}(s,a,s')^2]\\ \nonumber
    &= \underset{\substack{a \sim \pi' \\ s' \sim P}}{\mathbb{E}}[(A_{\pi}(s,a,s') + H(s,a,s'))^2] - \underset{\substack{a \sim \pi \\ s' \sim P}}{\mathbb{E}}[A_{\pi}(s,a,s')^2]\\ \nonumber
    &= \underset{\substack{a \sim \pi' \\ s' \sim P}}{\mathbb{E}}[A_{\pi}(s,a,s')^2] - \underset{\substack{a \sim \pi \\ s' \sim P}}{\mathbb{E}}[A_{\pi}(s,a,s')^2] + 2\underset{\substack{a \sim \pi' \\ s' \sim P}}{\mathbb{E}}[A_{\pi}(s,a,s') H(s,a,s')] + \underset{\substack{a \sim \pi' \\ s' \sim P}}{\mathbb{E}}[H(s,a,s')^2]\\ \nonumber
    &= \underset{\substack{\\a\sim\pi\\s'\sim P}}{\mathbb{E}}\left[\left(\frac{\pi^{\prime}(a|s)}{\pi(a|s)}-1\right) A_{\pi}(s,a,s')^2\right] + 2\underset{\substack{a \sim \pi' \\ s' \sim P}}{\mathbb{E}}[A_{\pi}(s,a,s') H(s,a,s')] + \underset{\substack{a \sim \pi' \\ s' \sim P}}{\mathbb{E}}[H(s,a,s')^2]\\ \nonumber
    &\leq \underset{\substack{\\a\sim\pi\\s'\sim P}}{\mathbb{E}}\left[\left(\frac{\pi^{\prime}(a|s)}{\pi(a|s)}-1\right) A_{\pi}(s,a,s')^2\right] + 2\underset{\substack{a \sim \pi \\ s' \sim P}}{\mathbb{E}}\left[\left(\frac{\pi^{\prime}(a|s)}{\pi(a|s)}\right)A_{\pi}(s,a,s')\right]|H(s,a,s')|_{max} + |H(s,a,s')|_{max}^2\\ \nonumber
\end{align}
Then we can bound ${\|\bm\Omega_{\pi'} -\bm\Omega_\pi\|_\infty}$ with:
\begin{align}
\label{eq: Omega differences}
 &\|\bm\Omega_{\pi'} -\bm\Omega_\pi\|_\infty \\ \nonumber
 &\leq \underset{s}{\mathbf{max}}\bigg|\underset{\substack{\\a\sim\pi\\s'\sim P}}{\mathbb{E}}\left[\left(\frac{\pi^{\prime}(a|s)}{\pi(a|s)}-1\right) A_{\pi}(s,a,s')^2\right] \\ \nonumber 
 &~~~~~+ 2\underset{\substack{a \sim \pi \\ s' \sim P}}{\mathbb{E}}\left[\left(\frac{\pi^{\prime}(a|s)}{\pi(a|s)}\right)A_{\pi}(s,a,s')\right]|H(s,a,s')|_{max} + |H(s,a,s')|_{max}^2\bigg|\\ \nonumber
\label{eq: MV_2}
\end{align}
By substituting \Cref{eq: Omega differences} into \Cref{eq: MV_1}, \Cref{lem: bound of MV} is proved.
\end{proof}
\subsection{VarianceMean Bound}
\label{proof: VarianceMean Bound}
\begin{prop}[Bound of VarianceMean]
\label{lem: bound of VM}
Denote \textbf{VarianceMean} of policy ${\pi}$ as $VM_\pi = \underset{s_0 \sim \mu}{\mathbb{V}ar} [V_\pi(s_0)]$. Given two policies $\pi', \pi$, the \textbf{VarianceMean} of $\pi'$ can be bounded by:
\begin{align}
    VM_{\pi'}  &\leq \underset{s_0 \sim \mu}{\mathbb{E}} [V_\pi^2(s_0)] + \|\mu^\top\|_\infty\underset{s}{\mathbf{max}}\bigg||\eta(s)|_{max}^2+2|V_\pi(s)|\cdot|\eta(s)|_{max}\bigg| \\ \nonumber
    &- \left(\mathbf{min}\left\{\mathbf{max}\left\{0,\ \mathcal{J}^l_{\pi^{\prime}, \pi}\right\}, \mathcal{J}^u_{\pi^{\prime}, \pi}\right\}\right)^2
\end{align}
\end{prop}
\begin{proof}
\begin{align}
\label{eq: bound of VM}
    VM_{\pi'} = \underset{s_0 \sim \mu}{\mathbb{E}} [V_{\pi'}^2(s_0)] - \mathcal{J}(\pi')^2
\end{align}
Since both terms on the right of \Cref{eq: bound of VM} are non-negative, we can bound $VM_{\pi'}$ with the upper bound of $\underset{s_0 \sim \mu}{\mathbb{E}} [V_{\pi'}^2(s_0)]$ and the lower bound of $\mathcal{J}(\pi')^2$.\\
Define $\bm Y_\pi = \begin{bmatrix} V_\pi^2(s^1) \\ V_\pi^2(s^2) \\ \vdots \end{bmatrix}$, where $\underset{s_0 \sim \mu}{\mathbb{E}} [V_\pi^2(s_0)] = \mu^\top \bm Y_\pi$.
Then we have
\begin{align}
\label{eq: bound EV}
    &\big|\underset{s_0 \sim \mu}{\mathbb{E}} [V_{\pi'}^2(s_0)]-\underset{s_0 \sim \mu}{\mathbb{E}} [V_\pi^2(s_0)]\big| \\ \nonumber
    &= \|\mu^\top \big(\bm Y_{\pi'} - \bm Y_\pi) \|_\infty\\ \nonumber
    &\leq \|\mu^\top\|_\infty  \|\bm Y_{\pi'} - \bm Y_\pi\|_\infty
\end{align}
To address $\|\bm Y_{\pi'} - \bm Y_\pi\|_\infty$, we have:
\begin{align}
    V_{\pi'}^2(s) - V_\pi^2(s) 
    &= \bigg(V_{\pi'}(s) - V_\pi(s)\bigg)\bigg(V_{\pi'}(s)+V_\pi(s)\bigg) \\ \nonumber
\end{align}
According to \eqref{eq: Expectation of A = -v + v}:
\begin{align}
    V_{\pi'}(s) - V_{\pi}(s) &= \underset{\substack{s_0 = s \\{\hat \tau} \sim \pi'}}{\mathbb{E}}\bigg[\sum_{t=0}^\infty \gamma^tA_{\pi}(s_t,a_t,s_{t+1})\bigg] \\ \nonumber
    &= \underset{\substack{s_0 = s \\{\hat \tau} \sim \pi'}}{\mathbb{E}}\bigg[\sum_{t=0}^\infty \gamma^t\bar A_{\pi',\pi}(s_t)\bigg]\\ \nonumber
    &\dot = \ \ \eta(s)
\end{align}
Define $T(s) = \underset{\substack{s_0 = s \\{\hat \tau} \sim \pi}}{\mathbb{E}}\bigg[\sum_{t=0}^\infty \gamma^t\bar A_{\pi',\pi}(s_t)\bigg]$, then we have:
\begin{align}
    |\eta(s)-T(s)| = 
    \bigg|\underset{\substack{s_0 = s \\{\hat \tau} \sim \pi'}}{\mathbb{E}}\bigg[\sum_{t=0}^\infty \gamma^t\bar A_{\pi',\pi}(s_t)\bigg] - \underset{\substack{s_0 = s \\{\hat \tau} \sim \pi}}{\mathbb{E}}\bigg[\sum_{t=0}^\infty \gamma^t\bar A_{\pi',\pi}(s_t)\bigg] \bigg| 
    \leq \dfrac{4\gamma\epsilon}{(1-\gamma)^2}(\mathcal{D}_{TV}^{max}(\pi'\|\pi))^2
\end{align}
And according to \cite{pinskerInequ}, we can bound $|\eta(s)|$ with:
\begin{align}
    |\eta(s)| &\leq \left|T(s)\right| + \dfrac{4\gamma\epsilon}{(1-\gamma)^2}(\mathcal{D}_{TV}^{max}(\pi'\|\pi))^2 \\ \nonumber
    &\leq \left|T(s)\right| + \dfrac{2\gamma\epsilon}{(1-\gamma)^2}\mathcal{D}_{KL}^{max}(\pi'\|\pi)) \doteq |\eta(s)|_{max} 
\end{align}
Further, we can obtain:
\begin{align}
    |V_{\pi'}(s)+V_\pi(s)| &\leq  |V_{\pi'}(s)|+|V_\pi(s)| \\ \nonumber
    &= |V_{\pi'}(s)| - |V_\pi(s)| + 2|V_\pi(s)| \\ \nonumber
    &\leq |V_{\pi'}(s)-V_\pi(s)| + 2|V_\pi(s)| \\ \nonumber
    &\leq |\eta(s)|_{max} + 2|V_\pi(s)| 
\end{align}
Thus the following inequality holds:
\begin{align}
\label{eq: bound V difference}
    &\|\bm Y_{\pi'} - \bm Y_\pi\|_\infty \\ \nonumber
    &\leq \underset{s}{\mathbf{max}}\bigg||V_{\pi'}(s)-V_\pi(s)|\cdot|V_{\pi'}(s)+V_\pi(s)|\bigg| \\ \nonumber
    &\leq \underset{s}{\mathbf{max}}\bigg||\eta(s)|_{max}\cdot\left(|\eta(s)|_{max} + 2|V_\pi(s)|\right)\bigg| \\ \nonumber
    &= \underset{s}{\mathbf{max}}\bigg||\eta(s)|_{max}^2+2|V_\pi(s)|\cdot|\eta(s)|_{max}\bigg|
\end{align}
Substitute \Cref{eq: bound V difference} into \Cref{eq: bound EV} the upper bound of $\underset{s_0 \sim \mu} {\mathbb{E}} [V_{\pi'}^2(s_0)]$ is obtained:
\begin{align}
\label{eq: upper bound of the first term}
\underset{s_0 \sim \mu}{\mathbb{E}} [V_{\pi'}^2(s_0)] &\leq \underset{s_0 \sim \mu}{\mathbb{E}} [V_\pi^2(s_0)] + \|\mu^\top\|_\infty\underset{s}{\mathbf{max}}\bigg||\eta(s)|_{max}^2+2|V_\pi(s)|\cdot|\eta(s)|_{max}\bigg| \\ \nonumber
\end{align}
The lower bound of $\mathcal{J}(\pi')^2$ can then be obtained according to \citep{achiam2017cpo}:
\begin{align}
\label{eq: lower bound of the second term}
    \mathcal{J}(\pi')^2 \geq \left(\mathbf{min}\left\{\mathbf{max}\left\{0,\ \mathcal{J}^l_{\pi^{\prime}, \pi}\right\}, \mathcal{J}^u_{\pi^{\prime}, \pi}\right\}\right)^2
\end{align}
where
\begin{align}
\nonumber
    \mathcal{J}^l_{\pi^{\prime}, \pi}&=\mathcal{J}(\pi) + \frac{1}{1-\gamma} \underset{\substack{s \sim d^\pi \\ a\sim {\pi'}}}{\mathbb{E}} \bigg[ A_\pi(s,a) - \frac{2\gamma \epsilon^{\pi'}}{1-\gamma} \sqrt{\frac 12 \mathcal{D}_{KL}({\pi'} \| \pi)[s]} \bigg] \\ \nonumber
    \mathcal{J}^u_{\pi^{\prime}, \pi}&=\mathcal{J}(\pi) + \frac{1}{1-\gamma} \underset{\substack{s \sim d^\pi \\ a\sim {\pi'}}}{\mathbb{E}} \bigg[ A_\pi(s,a) + \frac{2\gamma \epsilon^{\pi'}}{1-\gamma} \sqrt{\frac 12 \mathcal{D}_{KL}({\pi'} \| \pi)[s]} \bigg]
\end{align}
By substituting \Cref{eq: upper bound of the first term} and \Cref{eq: lower bound of the second term} into \Cref{eq: bound of VM} \Cref{lem: bound of VM} is proved.
\end{proof}

\subsection{Expectation Bound [Theorem 1, \cite{achiam2017cpo}]}
\label{proof: Expectation Bound}
\begin{prop} 
For any policies $\pi', \pi$, with $\epsilon^{\pi'} \doteq \underset{s}{\mathbf{max}}|\underset{a\sim\pi'}{\mathbb{E}}[A^{\pi}(s,a)]|$, and define $d^\pi = (1-\gamma)\sum\limits_{t=0}^\infty\gamma^t P(s_t=s|\pi)$ as the discounted state distribution using $\pi$, then the following bound holds:
\begin{align}
  \mathcal{J}(\pi') - \mathcal{J}(\pi) \geq \frac{1}{1-\gamma} \underset{\substack{s \sim d^\pi \\ a\sim {\pi'}}}{\mathbb{E}} \bigg[ A^\pi(s,a) - \frac{2\gamma \epsilon^{\pi'}}{1-\gamma} \mathcal{D}_{TV}({\pi'} \| \pi)[s] \bigg]
\end{align}
\label{lem: lower bound of mean}
\end{prop}
\begin{proof}
$d^\pi$ we used is defined as
\begin{align}
    d^\pi(\hat{s})=(1-\gamma)\sum\limits_{t=0}^\infty\gamma^t P(s_t=s|\pi).
\end{align}

Then it allows us to express the expected discounted total reward compactly as:
\begin{align}
\label{eq: J def}
    \mathcal{J}(\pi)=\dfrac{1}{1-\gamma}\underset{\substack{s \sim d^\pi \\ a\sim {\pi}\\s'\sim P}}{\mathbb{E}}\left[R(s,a,s')\right],
\end{align}

where by $a\sim \pi$, we mean $a \sim \pi(\cdot|s)$, and by $s'\sim P$,we mean $s'\sim P(\cdot|s,a)$. We drop the explicit notation for the sake of reducing clutter, but it should be clear from context that $a$ and $s'$ depend on $s$. 

Define $P(s'|s,a)$ is the probability of transitioning to state $s'$ given that the previous state was $s$ and the agent took action $a$ at state $s$, and $\mu : \mathcal{\mathcal{S}} \mapsto [0, 1]$ is the initial augmented state distribution. Let $p_\pi^t\in\mathbb{R}^{|\mathcal{S}|} $ denote the vector with components $p_\pi^t(s) = P(s_t = s|\pi)$, and let $P_{\pi} \in \mathbb{R}^{|\mathcal{S}|\times|\mathcal{S}}|$ denote the transition matrix with components $P_{\pi}(s'|s) =\int P(s'|s,a)\pi(a|s) da$; then $p_\pi^t=P_\pi p_\pi^{t-1}=P_\pi^t\mu$ and
\begin{align}
\label{eq: d pi}
    d^\pi& =(1-\gamma)\sum\limits_{t=0}^\infty(\gamma P_\pi)^t \mu  \\ \nonumber
    &=(1-\gamma)(I-\gamma P_\pi)^{-1}\mu
\end{align}
This formulation helps us easily obtain the following lemma.

\begin{lem} 
\label{lem: lemma 6}
For any function $f:\mathcal{S}\mapsto\mathbb{R}$ and any policy $\pi$,
\begin{align}
    (1-\gamma)\underset{s\sim \mu}{\mathbb{E}}[f( s)]+\underset{\substack{s \sim d^\pi \\ a\sim {\pi}\\ s'\sim P}}{\mathbb{E}}[\gamma f(s')]-\underset{s\sim d^\pi}{\mathbb{E}}[f( s)]=0.
\end{align}
\end{lem}
\begin{proof}
    Multiply both sides of \eqref{eq: d pi} by $(I-\gamma P_\pi)$ and take the inner product with the vector $f\in\mathbb{R}^{|\mathcal{S}|}$.
\end{proof}
Combining \Cref{lem: lemma 6} with \eqref{eq: J def}, we obtain the following, for any function $f$ and any policy $\pi$:
\begin{align}
\label{eq: J pi def}
    \mathcal{J}(\pi)=\underset{s\sim\mu}{\mathbb{E}}\left[f(s)\right]+\dfrac{1}{1-\gamma}\underset{\substack{s \sim d^\pi \\ a\sim \pi\\ s'\sim P}}{\mathbb{E}}\left[R(s,a,s')+\gamma f(s')-f(s)\right]
\end{align}
Then we will derive and present the new policy improvement bound. We will begin with a lemma:
\begin{lem}
\label{lem: lemma 7}
For any function $f \mapsto \mathbb{R}$ and any policies $\pi'$ and $\pi$, define
\begin{align}
L_{\pi,f}(\pi')\doteq\underset{\substack{s \sim d^\pi \\ a\sim \pi\\ s'\sim P}}{\mathbb{E}}\left[\left(\dfrac{\pi'(a|s)}{\pi(a|s)}-1\right)(R(s,a,s')+\gamma f(s')-f(s))\right],
\end{align}
and $\epsilon_f^{\pi\prime}\doteq\underset{s}{\mathbf{max}}|\underset{\substack{a\sim\pi^\prime\\s^\prime\sim P}}{\mathbb{E}}[R(s,a,s^\prime)+\gamma f(s^\prime)-f(s)]|.$ Then the following bounds hold:
\begin{align}
\label{eq: J lower bound}
    \mathcal{J}(\pi') - \mathcal{J}(\pi) \geq\frac{1}{1-\gamma}\left(L_{\pi,f}(\pi^{\prime})-2\epsilon_{f}^{\pi^{\prime}}D_{T V}(d^{\pi^{\prime}}||d^{\pi})\right), \\
\label{eq: J upper bound}
    \mathcal{J}(\pi') - \mathcal{J}(\pi) \leq\frac{1}{1-\gamma}\left(L_{\pi,f}(\pi^{\prime})+2\epsilon_{f}^{\pi^{\prime}}D_{T V}(d^{\pi^{\prime}}||d^{\pi})\right), 
\end{align}
where $D_{TV}$ is the total variational divergence. Furthermore, the bounds are tight(when $\pi'$ = $\pi$, the LHS and RHS are identically zero).
\end{lem}
\begin{proof}
    
First, for notational convenience, let $\delta_f(s,a,s') \doteq R(s,a,s') + \gamma f(s') - f(s)$. By \eqref{eq: J pi def}, we obtain the identity
\begin{align}
\label{eq: J difference detail}
     \mathcal{J}(\pi')-\mathcal{J}(\pi)=\dfrac{1}{1-\gamma}\left( \underset{\substack{s \sim d^{\pi'} \\ a\sim {\pi'}\\ s'\sim P}}{\mathbb{E}}[\delta_f(s,a,s')] - \underset{\substack{s \sim d^\pi \\ a\sim \pi\\ s'\sim P}}{\mathbb{E}}[\delta_f(s,a,s')] \right)
\end{align}

Now, we restrict our attention to the first term in \eqref{eq: J difference detail}. Let $\dagger\delta_f^{\pi'}\in\mathbb{R}^{|\mathcal{S}|}$ denote the vector of components, where $\dagger\delta_f^{\pi'}(s) = \underset{\substack{a\sim\pi'\\s'\sim P}}{\mathbb{E}}[\delta_f (s,a,s')|s]$. Observe that
$$
\begin{aligned}
\underset{\substack{s \sim d^{\pi'} \\ a\sim {\pi'}\\ s'\sim P}}{\mathbb{E}}[\delta_{f}(s,a,s^{\prime})]& =\left\langle d^{\pi'},\dagger{\delta}_{f}^{\pi'}\right\rangle  \\
&=\left\langle d^{\pi},\dagger{\delta}_{f}^{\pi'}\right\rangle+\left\langle d^{\pi'}-d^{\pi},\dagger{\delta}_{f}^{\pi'}\right\rangle
\end{aligned}
$$
With the Hölder's inequality; for any $p,q\in [1,\infty]$ such that $\dfrac 1p+\dfrac 1q = 1$, we have 
\begin{align}
\label{eq: discounted original bounds}
    \left\langle d^{\pi},\dagger{\delta}_{f}^{\pi'}\right\rangle + \left\|d^{\pi'}-d^{\pi}\right\|_p\left\|\dagger{\delta}_f^{\pi'}\right\|_q \ge \underset{\substack{s \sim d^{\pi'} \\ a\sim {\pi'}\\ s'\sim P}}{\mathbb{E}}[\delta_{f}(s,a,s^{\prime})] \ge \left\langle d^{\pi},\dagger{\delta}_{f}^{\pi'}\right\rangle - \left\|d^{\pi'}-d^{\pi}\right\|_p\left\|\dagger{\delta}_f^{\pi'}\right\|_q
\end{align}

We choose $p=1$ and $q=\infty$; With $\left\|d^{\pi'}-d^{\pi}\right\|_1 = 2D_{TV}(d^{\pi'}||d^{\pi})$ and $\left\| \dagger{\delta}_f^{\pi'} \right\|_{\infty} = \epsilon_{f}^{\pi^{\prime}}$, and by the importance sampling identity, we have
\begin{align}
\label{eq: importance sampling theo 1}
     \left\langle d^{\pi},\dagger{\delta}_{f}^{\pi'}\right\rangle & = \underset{\substack{s \sim d^{\pi} \\ a\sim {\pi'}\\ s'\sim P}}{\mathbb{E}}[\delta_{f}(s,a,s^{\prime})]\\ \nonumber 
& =\underset{\substack{s \sim d^{\pi} \\ a\sim {\pi}\\ s'\sim P}}{\mathbb{E}}[\left( \frac{\pi'(a|s)}{\pi(a|s)}\right)\delta_{f}(s,a,s^{\prime})]
\end{align}
After bringing \eqref{eq: importance sampling theo 1}, $\left\|  d^{\pi'}- d^{\pi}\right\|_1$, $\left\| \dagger{\delta}_f^{\pi'} \right\|_{\infty}$ into \eqref{eq: discounted original bounds}, then substract $\underset{\substack{s \sim d^\pi \\ a\sim \pi\\ s'\sim P}}{\mathbb{E}}[ \delta_f(s,a,s')]$,  the bounds are obtained. The lower bound leads to \eqref{eq: J lower bound}, and the upper bound leads to \eqref{eq: J upper bound}.
\end{proof}
Then we will bound the divergence term, $||d^{\pi'}-d^{\pi}||_1$, i.e. $2D_{TV}(d^{\pi'}|| d^{\pi})$.

\begin{lem}
\label{lem: lemma 8}
The divergence between discounted future state visitation distributions, $||d^{\pi'}-d^{\pi}||_1$, is bounded by an average divergence of the policies $\pi'$ and $\pi$:
\begin{align}
    \|d^{\pi'}-d^{\pi}\|_1\leq\dfrac{2\gamma}{1-\gamma}\underset{\hat{s}\sim d^{\pi}}{\mathbb{E}}\left[D_{TV}(\pi'||\pi)[s]\right],
\end{align}
where $D_{TV}(\pi'||\pi)[s]=\dfrac 12\sum_a|\pi'(a|s)-\pi(a|s)|.$
\end{lem}
\begin{proof}
Firstly, we introduce an identity for the vector difference of the discounted future state visitation distributions on two different policies, $\pi'$ and $\pi$. Define the matrices $G \doteq (I-\gamma P_{\pi})^{-1}, \bar G \doteq (I-\gamma P_{\pi'})^{-1}$, and $\Delta = P_{\pi'}-P_{\pi}$. Then:
\begin{align}
G^{-1}-\bar{G}^{-1}& =(I-\gamma P_{\pi})-(I-\gamma P_{\pi'})  \\ \nonumber 
&=\gamma\Delta,
\end{align}
left-multiplying by $G$ and right-multiplying by $\bar G$, we obtain
\begin{align}
    \bar G - G = \gamma \bar G\Delta G.
\end{align}
Thus, the following equality holds: 
\begin{align}
\label{eq: difference between ds}
d^{\pi'}-d^{\pi}&=(1-\gamma)\left(\bar{G}-G\right)\mu  \\ \nonumber 
&=\gamma(1-\gamma)\bar{G}\Delta G \mu \\ \nonumber
&=\gamma\bar{G}\Delta d^{\pi}.
\end{align}
Using \eqref{eq: difference between ds}, we obtain
\begin{align}
\label{eq: bound for diff ds}
\|d^{\pi^{\prime}}-d^{\pi}\|_1& =\gamma\|\bar{G}\Delta d^{\pi}\|_1  \\ \nonumber
&\leq\gamma\|\bar{G}\|_1\|\Delta d^\pi\|_1,
\end{align}
where $||\bar G||_1$ is bounded by:
\begin{align}
\label{eq: bound for G 1}
    \|\bar{G}\|_1=\|(I-\gamma P_{\pi'})^{-1}\|_1\le\sum\limits_{t=0}^\infty\gamma^t\|P_{\pi'}\|_1^t=(1-\gamma)^{-1}.
\end{align}
Next, we bound $\|\Delta d^{\pi}_1\|$ as following:
\begin{align}
\label{eq: bound for delta d pi}
\|\Delta d^{\pi}\|_1& =\quad \sum\limits_{s'}\left|\sum\limits_{s}\Delta(s'|s)d^\pi(s)\right|  \\ \nonumber
& \le \quad\sum \limits_{s,s'}|\Delta(s'|s)|d^\pi(s) \\ \nonumber
& = \quad \sum_{s,s'}\left|\sum_a P(s'|s,a)\left(\pi'(a|s)-\pi(a|s)\right)\right|d^{\pi}(s) \\ \nonumber
& \le \quad \sum_{s,a,s'}P(s'|s,a)|\pi'(a|s)-\pi(a|s)|d^{\pi}(s) \\ \nonumber
& = \quad \sum_{s,a}|\pi'(a|s)-\pi(a|s)|d^\pi(s) \\ \nonumber
& = \quad 2\underset{s\sim d^\pi}{\mathbb{E}}[D_{TV}(\pi'||\pi)[s]].
\end{align}
By taking \eqref{eq: bound for delta d pi} and \eqref{eq: bound for G 1} into \eqref{eq: bound for diff ds}, this lemma is proved. 
\end{proof}
The new policy improvement bound follows immediately.
\begin{lem}
\label{lem: lemma 9}
For any function $f:\mathcal{S}\mapsto \mathbb{R}$ and any policies $\pi'$ and $\pi$, define $\delta_f(s,a,s')\doteq R(s,a,s')+\gamma f(s')-f(s)$,
$$
\begin{gathered}
\epsilon_f^{\pi^{\prime}}\doteq\operatorname*{\underset{s}{\mathbf{max}}}|\underset{\substack{a \sim \pi^{\prime}\\ s^{\prime} \sim P}}{\mathbb{E}}[\delta_{f}(s,a,s^{\prime})]|, \\
L_{\pi,f}(\pi^{\prime})\doteq\underset{\substack{s \sim d^{\pi}\\a\sim\pi\\s'\sim P}}{\mathbb{E}}\left[\left(\frac{\pi^{\prime}(a|s)}{\pi(a|s)}-1\right)\delta_{f}(s,a,s^{\prime})\right],and \\
D_{\pi,f}^{\pm}(\pi^{\prime})\doteq\frac{L_{\pi,f}(\pi^{\prime})}{1-\gamma}\pm\frac{2\gamma\epsilon_{f}^{\pi^{\prime}}}{(1-\gamma)^{2}}\underset{s\sim d^{\pi}}{\mathbb{E}}[D_{T V}(\pi^{\prime}||\pi)[s]],
\end{gathered}
$$
where $D_{TV}(\pi'||\pi)[s] = \frac 12\sum_a|\pi'(a|s)-\pi(a|s)|$ is the total variational divergence between action distributions at $s$.
The following bounds hold:
$$
D_{\pi,f}^{+}(\pi^{\prime})\geq \mathcal{J}(\pi^{\prime})-\mathcal{J}(\pi)\geq D_{\pi,f}^{-}(\pi^{\prime}).
$$
Furthermore, the bounds are tight (when $\pi'$ = $\pi$, all three expressions are identically zero)
\end{lem}
\begin{proof}
    Begin with the bounds from \cref{lem: lemma 7} and bound the divergence $D_{TV}(d^{\pi'}||d^{\pi})$ by \cref{lem: lemma 8}.
\end{proof}
The choice of $f=V_{\pi}$ in \cref{lem: lemma 9} leads to following inequality: \\

For any policies $\pi', \pi$, with $\epsilon^{\pi'} \doteq \underset{s}{\mathbf{max}}|\underset{\substack{a\sim\pi'}}{\mathbb{E}}[A_{\pi}(s,a)]|$, the following bound holds:
$$
\begin{aligned}
&\mathcal{J}(\pi') - \mathcal{J}(\pi) \geq\dfrac{1}{1-\gamma} \underset{\substack{s\sim d^\pi\\a\sim\pi'}}{\mathbb{E}}\left[A_\pi(s,a)-\dfrac{2\gamma\epsilon^{\pi'}}{1-\gamma}D_{TV}(\pi'||\pi)[s]\right]
\end{aligned} 
$$
At this point, the \cref{lem: lower bound of mean} is proved.
\end{proof}

\subsection{Proof of \Cref{theo: absolute performance improvement}}
\label{proof: theo 1}

With \Cref{lem: bound of MV}, \Cref{lem: bound of VM},  and \Cref{lem: lower bound of mean}, we have the following surrogate function of lower probability bound $\mathcal{B}_k(\pi')$:
\begin{align}
\label{eq: bk lower bound}
    \mathcal{B}_k(\pi') \geq \mathcal{J}^l_{\pi', \pi} - k\left({MV}_{\pi',\pi} + {VM}_{\pi',\pi} \right)
\end{align}
where
\begin{align}
\nonumber
    &~~~~{MV}_{\pi',\pi} = \frac{\|\mu^\top\|_\infty}{1-\gamma^2}\underset{s}{\textbf{max}}\Bigg|\underset{\substack{\\a\sim\pi'\\s'\sim P}}{\mathbb{E}}\left[A_{\pi}(s,a,s')^2\right]-\underset{\substack{\\a\sim\pi\\s'\sim P}}{\mathbb{E}}\left[A_{\pi}(s,a,s')^2\right] + |H(s,a,s')|_{max}^2 \\ \nonumber
    &~~~~~~~~~~~~~~ + 2\underset{\substack{a \sim \pi' \\ s' \sim P}}{\mathbb{E}}\left[A_{\pi}(s,a,s')\right]\cdot|H(s,a,s')|_{max}\Bigg| + MV_{\pi} + \frac{2\gamma^2\|\mu^\top\|_\infty}{(1-\gamma^2)^2}\sqrt{\frac{1}{2}\mathcal{D}_{KL}^{max}(\pi' \| \pi)}\cdot\|\Omega_{\pi}\|_\infty\\ \nonumber
    &~~~~{VM}_{\pi',\pi} = \|\mu^\top\|_\infty\underset{s}{\textbf{max}}\bigg||\eta(s)|_{max}^2+2|V_{\pi}(s)|\cdot|\eta(s)|_{max}\bigg| + \underset{s_0 \sim \mu}{\mathbb{E}} [V_{\pi}^2(s_0)] \\\nonumber 
    &~~~~~~~~~~~~~~~~~~~~~~~~- \left(\mathbf{min}\left\{\mathbf{max}\left\{0,\ \mathcal{J}^l_{\pi^{\prime}, \pi}\right\}, \mathcal{J}^u_{\pi^{\prime}, \pi}\right\}\right)^2 \\ \nonumber
    &~~~~\mathcal{J}^l_{\pi', \pi} = \mathcal{J}(\pi) + \frac{1}{1-\gamma} \underset{\substack{s \sim d^\pi \\ a\sim {\pi'}}}{\mathbb{E}} \bigg[ A_\pi(s,a) - \frac{2\gamma \epsilon^{\pi'}}{1-\gamma} \sqrt{\frac 12 \mathcal{D}_{KL}({\pi'} \| \pi)[s]} \bigg]
\end{align}

We define $\mathcal{M}_k^j(\pi) = \mathcal{J}^l_{\pi, \pi_j} - k\left({MV}_{\pi,\pi_j} + {VM}_{\pi,\pi_j} \right)$, and it can be found that $\mathcal{B}_k(\pi_j) = \mathcal{M}_k^j({\pi_j})$. Then by \Cref{eq: bk lower bound}, we have $\mathcal{B}_k(\pi_{j+1})\geq \mathcal{M}_k^j(\pi_{j+1})$ and the following holds:
\begin{align}
    \mathcal{B}_k(\pi_{j+1}) - \mathcal{B}_k(\pi_{j}) \geq \mathcal{M}_k^j(\pi_{j+1}) - \mathcal{M}_k^j(\pi_{j})
\end{align}
Thus, by maximizing $\mathcal{M}_k^j$ at each iteration, we guarantee that the true lower probability bound of performance $\mathcal{B}_k$ is non-decreasing. So far \Cref{theo: absolute performance improvement} has been proved.

\begin{rmk}
    $\mathcal{M}_k^j(\pi)$ is the objective function in our optimization problem which we can guarantee its monotonic improvement theoretically. Thus the RHS greater than or equal to zero, which can lead to the monotonic improvement of lower probability bound $\mathcal{B}_k(\pi)$.
\end{rmk}

\clearpage

\subsection{Additional Results}
\label{sec: additional results}
\begin{lem}
\label{lem: b eq m}
    $\mathcal{B}_k(\pi_j) = \mathcal{M}_k^j(\pi_j)$, where $\mathcal{B}_k(\pi_j) = \mathcal{J}(\pi_j) - k\mathcal{V}(\pi_j)$ and $\mathcal{M}_k^j(\pi) = \mathcal{J}^l_{\pi, \pi_j} - k\left({MV}_{\pi,\pi_j} + {VM}_{\pi,\pi_j} \right)$.
\end{lem}
\begin{proof}
    To prove \Cref{lem: b eq m}, we will show that (i) $\mathcal{J}^l_{\pi_j, \pi_j} = \mathcal{J}(\pi_j)$, and (ii) $\left({MV}_{\pi,\pi_j} + {VM}_{\pi,\pi_j} \right) = \mathcal{V}(\pi_j)$.

    \paragraph{Expectation Part}
    For the same policy $\pi_j$, the following two conditions hold:
    \begin{align}
        \underset{\substack{s \sim d^{\pi_j}\\a\sim {\pi_j}}}{\mathbb{E}} \big[ A_{\pi_j}(s,a)\big] = 0 \label{eq: expect 0} \\ 
        \mathcal{D}_{KL}({\pi_j} \| \pi_j)[s] = 0 ~~. \label{eq: kl 0} 
    \end{align}
    With \eqref{eq: kl 0} and \eqref{eq: expect 0}, the following equality holds:
    \begin{align}
    \label{eq: lower eq}
        \mathcal{J}^l_{\pi_j, \pi_j} &= \mathcal{J}(\pi_j) + \frac{1}{1-\gamma} \underset{\substack{s \sim d^{\pi_j}\\a\sim {\pi_j}}}{\mathbb{E}} \bigg[ A_{\pi_j}(s,a) - \frac{2\gamma \epsilon^{\pi}}{1-\gamma} \sqrt{\frac 12 \mathcal{D}_{KL}({\pi_j} \| \pi_j)[s]} \bigg] \\ \nonumber 
        &= \mathcal{J}(\pi_j) ~~.
    \end{align}
    
    Similarly, 
    \begin{align}
    \label{eq: upper eq}
        \mathcal{J}^u_{\pi_j, \pi_j} &= \mathcal{J}(\pi_j) + \frac{1}{1-\gamma} \underset{\substack{s \sim d^{\pi_j} \\ a\sim {\pi_j}}}{\mathbb{E}} \bigg[ A_{\pi_j}(s,a) + \frac{2\gamma \epsilon^{\pi_j}}{1-\gamma} \sqrt{\frac 12 \mathcal{D}_{KL}({\pi_j} \| \pi_j)[s]} \bigg] \\ \nonumber 
       &=  \mathcal{J}(\pi_j)~~. 
    \end{align}

    \paragraph{Variance Part}
    For the same policy $\pi_j$, the following conditions hold: 
    \begin{align}
        \underset{\substack{\\a\sim\pi_j\\s'\sim P}}{\mathbb{E}}\left[A_{\pi_j}(s,a,s')^2\right]-\underset{\substack{\\a\sim\pi_j\\s'\sim P}}{\mathbb{E}}\left[A_{\pi_j}(s,a,s')^2\right] = 0  \label{eq: minus 0}\\ 
        \underset{\substack{a \sim \pi_j \\ s' \sim P}}{\mathbb{E}}\left[A_{\pi_j}(s,a,s')\right] = 0 \label{eq: a a o} \\
        \forall s, \bar A_{\pi_j,\pi_j}(s) = 0 ~~~.\label{eq: bar 0} 
    \end{align}

\eqref{eq: kl 0}, \eqref{eq: a a o}  and \eqref{eq: bar 0} indicate that:
\begin{align}
\label{eq: H 0}
    |H(s,a,s')|_{max} &= \left|\gamma\underset{\substack{s_0 = s' \\ \hat\tau \sim \pi_j}}{\mathbb{E}}\bigg[\sum_{t=0}^\infty \gamma^t \bar A_{\pi_j,\pi_j}(s_t)\bigg] - \underset{\substack{s_0 = s \\ \hat\tau \sim \pi_j}}{\mathbb{E}}\bigg[\sum_{t=0}^\infty \gamma^t\bar A_{\pi_j,\pi_j}(s_t)\bigg] \right| + \dfrac{2\gamma(1+\gamma)\epsilon}{(1-\gamma)^2}\mathcal{D}_{KL}^{max}(\pi_j||\pi_j) \\ \nonumber 
    &= 0  \\
    \label{eq: eta 0}
    |\eta(s)|_{max} &= \left| \underset{\substack{s_0 = s \\\hat\tau \sim \pi_j}}{\mathbb{E}}\bigg[\sum_{t=0}^\infty \gamma^t\bar A_{\pi_j,\pi_j}(s_t)\bigg] \right| + \frac{2\gamma \epsilon}{(1-\gamma)^2}\mathcal{D}_{KL}^{max}({\pi_j} \| \pi_j) \\ \nonumber 
    &= 0 ~~~.
\end{align}

With \eqref{eq: lower eq}, \eqref{eq: upper eq},  \eqref{eq: minus 0}, \eqref{eq: H 0} and \eqref{eq: eta 0}, we have the following condition hold:
\begin{align}
\label{eq: variance eq}
    &{MV}_{\pi_j,\pi_j} + {VM}_{\pi_j,\pi_j} \\ \nonumber 
    &= \frac{\|\mu^\top\|_\infty}{1-\gamma^2}\underset{s}{\textbf{max}}\Bigg|\underset{\substack{\\a\sim\pi_j\\s'\sim P}}{\mathbb{E}}\left[A_{\pi_j}(s,a,s')^2\right]-\underset{\substack{\\a\sim\pi_j\\s'\sim P}}{\mathbb{E}}\left[A_{\pi_j}(s,a,s')^2\right] + |H(s,a,s')|_{max}^2 \\ \nonumber
    & ~~~+ 2\underset{\substack{a \sim \pi_j \\ s' \sim P}}{\mathbb{E}}\left[A_{\pi_j}(s,a,s')\right]\cdot|H(s,a,s')|_{max}\Bigg| + MV_{\pi_j} + \frac{2\gamma^2\|\mu^\top\|_\infty}{(1-\gamma^2)^2}\sqrt{\frac{1}{2}\mathcal{D}_{KL}^{max}(\pi_j \| \pi_j)}\cdot\|\Omega_{\pi_j}\|_\infty\\ \nonumber 
    & ~~~+ \|\mu^\top\|_\infty\underset{s}{\textbf{max}}\bigg||\eta(s)|_{max}^2+2|V_{\pi_j}(s)|\cdot|\eta(s)|_{max}\bigg| 
    - \textbf{min} ~\left( \mathcal{J}(\pi) \right)^2 
    + \underset{s_0 \sim \mu}{\mathbb{E}} [V_{\pi_j}^2(s_0)]  \\ \nonumber 
    &= MV_{\pi_j}  - \textbf{min} ~\left( \mathcal{J}(\pi) \right)^2 + \underset{s_0 \sim \mu}{\mathbb{E}} [V_{\pi_j}^2(s_0)] \\ \nonumber 
    &= \underset{\substack{s_0 \sim \mu}}{\mathbb{E}}[\underset{\substack{{\hat \tau}\sim {\pi_j}}}{\mathbb{V}ar}[R_{\pi_j}(s_0)] - \minimizewrt{\mathcal{J}(\pi) \in [\mathcal{J}^l_{\pi_j, \pi_j}, \mathcal{J}^u_{\pi_j, \pi_j}]} \left( \mathcal{J}(\pi) \right)^2 + \underset{s_0 \sim \mu}{\mathbb{E}} [V_{\pi_j}^2(s_0)] \\ \nonumber 
    &= \underset{\substack{s_0 \sim \mu}}{\mathbb{E}}[\underset{\substack{{\hat \tau}\sim {\pi_j}}}{\mathbb{V}ar}[R_{\pi_j}(s_0)] - \minimizewrt{\mathcal{J}(\pi) \in [\mathcal{J}(\pi_j), \mathcal{J}(\pi_j)]} \left( \mathcal{J}(\pi) \right)^2 + \underset{s_0 \sim \mu}{\mathbb{E}} [V_{\pi_j}^2(s_0)] \\ \nonumber 
    &= \underset{\substack{s_0 \sim \mu}}{\mathbb{E}}[\underset{\substack{{\hat \tau}\sim {\pi_j}}}{\mathbb{V}ar}[R_{\pi_j}(s_0)]  + \underset{s_0 \sim \mu}{\mathbb{E}} [V_{\pi_j}^2(s_0)] -  \mathcal{J}(\pi_j)^2 \\ \nonumber 
    &= \mathcal{V}(\pi_j) 
\end{align}

\paragraph{Summarize}
With \eqref{eq: lower eq} and \eqref{eq: variance eq}, we have the following condition hold: 
\begin{align}
    \mathcal{B}_k(\pi_j) = \mathcal{J}(\pi_j) - k\mathcal{V}(\pi_j) = \mathcal{J}^l_{\pi_j, \pi_j} - k\left({MV}_{\pi_j,\pi_j} + {VM}_{\pi_j,\pi_j} \right) = \mathcal{M}_k^j(\pi_j)~~,
\end{align}
 which proves the Lemma.
\end{proof}

\newpage

\section{APO Pseudocode}
\label{append: apo}

\begin{algorithm}
\caption{Absolute Policy Optimization}\label{alg:apo_main}
\begin{algorithmic}
\STATE \textbf{Input:} Initial policy $\pi_0\in\Pi_\theta$.
\FOR{$j=0,1,2,\dots$}
\STATE Sample trajectory $\tau\sim\pi_j=\pi_{\theta_j}$
\STATE Estimate gradient $g \gets \left.\nabla_{\theta} 
 O_{\pi, \pi_j}\right\rvert_{\theta=\theta_j}$
    
\COMMENT{Define $O_{\pi, \pi_j}=\bigg(\frac{1}{1-\gamma} \underset{\substack{s \sim d^{\pi_j} \\ a\sim {\pi}}}{\mathbb{E}} \left[ A_{\pi_j}(s,a) \right]- k\left(\overline{MV}_{\pi, \pi_j}+\overline{VM}_{\pi, \pi_j}\right)\bigg)$}
\STATE Estimate Hessian $H \gets \left.\nabla^2_{\theta} \mathbb{E}_{s \sim \pi_j}[\mathcal{D}_{KL}(\pi \| \pi_j)[s]]\right\rvert_{\theta=\theta_j}$
\STATE Solve convex programming 
\COMMENT{\cite{achiam2017cpo}}
\begin{align*}
    \theta^*_{j+1} &= \argmax_\theta g^\top(\theta-\theta_j) \\
    &~~\text{s.t.} ~\frac{1}{2}(\theta-\theta_j)^\top H (\theta-\theta_j) \leq \delta 
\end{align*}
\STATE Get search direction $\Delta\theta^* \gets \theta^*_{j+1} - \theta_j$
\FOR[Line search]{$k=0,1,2,\dots$}
\STATE $\theta' \gets \theta_{j} + \xi^k\Delta\theta^*$ \COMMENT{$\xi\in(0,1)$ is the backtracking coefficient}
\IF[Objective]{
$\mathbb{E}_{s \sim \pi_j}[\mathcal{D}_{KL}(\pi_{\theta'} \| \pi_j)[s]]\leq \delta$ 
\textbf{and} \{Trust Region\}\\
$~~~~~~~~~~~~~~ O_{\pi_{\theta'}, \pi_j}\geq O_{\pi_{j}, \pi_j}$}
\STATE $\theta_{j+1} \gets \theta'$ \COMMENT{Update policy}
\STATE \textbf{break}
\ENDIF
\ENDFOR
\ENDFOR

\end{algorithmic}
\end{algorithm}

\clearpage
\section{Expeiment Details}

\begin{figure*}[h]
  \centering
  \begin{subfigure}[t]{1.0\linewidth}
  \centering
  \begin{subfigure}[t]{0.12\linewidth}
      \begin{subfigure}[t]{\linewidth}
        \centering
        \includegraphics[scale=0.115]{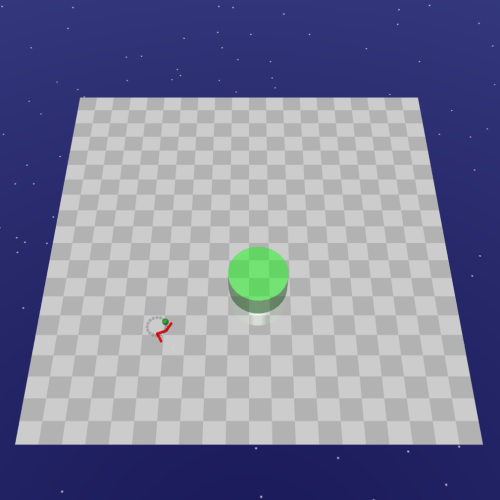}
      \end{subfigure}
      \caption{Goal}
      \label{fig: Goal3D1}
  \end{subfigure}
  \begin{subfigure}[t]{0.12\linewidth}
      \begin{subfigure}[t]{1\linewidth}
        \centering
        \includegraphics[scale=0.115]{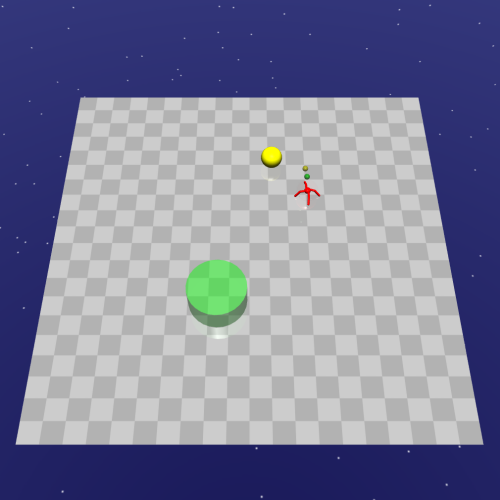}
      \end{subfigure}
      \caption{Push}
      \label{fig: PushBall1}
  \end{subfigure}
  \begin{subfigure}[t]{0.12\linewidth}
      \begin{subfigure}[t]{\linewidth}
        \centering
        \includegraphics[scale=0.115]{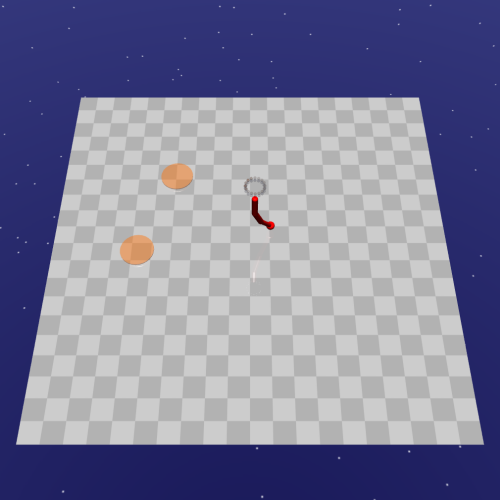}
      \end{subfigure}
      \caption{Chase}
      \label{fig: Chase}
  \end{subfigure}
  \begin{subfigure}[t]{0.12\linewidth}
      \begin{subfigure}[t]{\linewidth}
        \centering
        \includegraphics[scale=0.382]{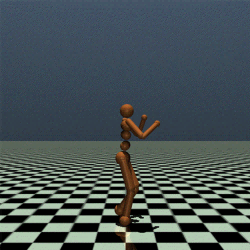}
      \end{subfigure}
      \caption{Humanoid}
      \label{fig: Humanoid}
  \end{subfigure}
  \begin{subfigure}[t]{0.12\linewidth}
      \begin{subfigure}[t]{\linewidth}
        \centering
        \includegraphics[scale=0.382]{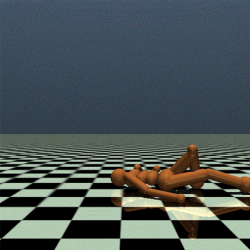}
      \end{subfigure}
      \caption{H.Standup}
      \label{fig: HumanoidStandup}
  \end{subfigure}
  \begin{subfigure}[t]{0.12\linewidth}
      \begin{subfigure}[t]{\linewidth}
        \centering
        \includegraphics[scale=0.39]{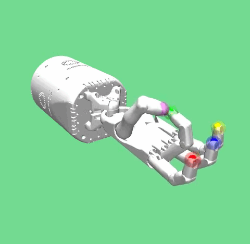}
      \end{subfigure}
      \caption{HandReach}
      \label{fig: HandReach}
  \end{subfigure}
  \begin{subfigure}[t]{0.12\linewidth}
      \begin{subfigure}[t]{\linewidth}
        \centering
        \includegraphics[scale=0.385]{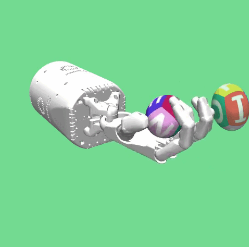}
      \end{subfigure}
      \caption{H.M.Egg}
      \label{fig: HandManipulateEgg}
  \end{subfigure}
  \begin{subfigure}[t]{0.12\linewidth}
      \begin{subfigure}[t]{\linewidth}
        \centering
        \includegraphics[scale=0.385]{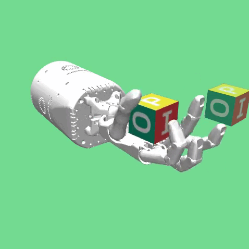}
      \end{subfigure}
      \caption{H.M.Block}
      \label{fig: HandManipulateBlock}
  \end{subfigure}
  \end{subfigure}
  \caption{Tasks of continuous experiments}
\end{figure*}

\begin{figure*}[h]
  \centering
  \captionsetup[subfigure]{}
  \begin{subfigure}[t]{0.119\linewidth}
        \centering
      \includegraphics[scale=0.085]{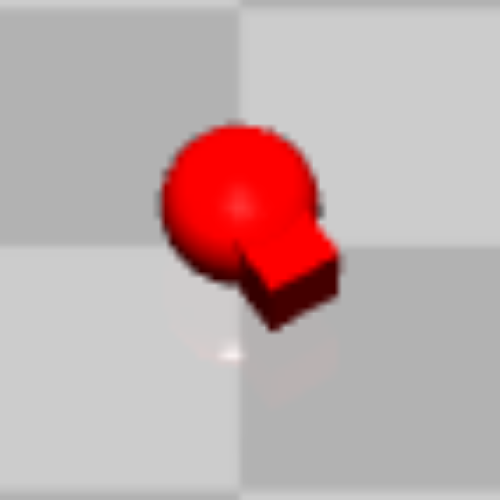}
      \caption{\footnotesize{Point}}
      \label{fig: Point}
  \end{subfigure}
  \begin{subfigure}[t]{0.119\linewidth}
        \centering
      \includegraphics[scale=0.085]{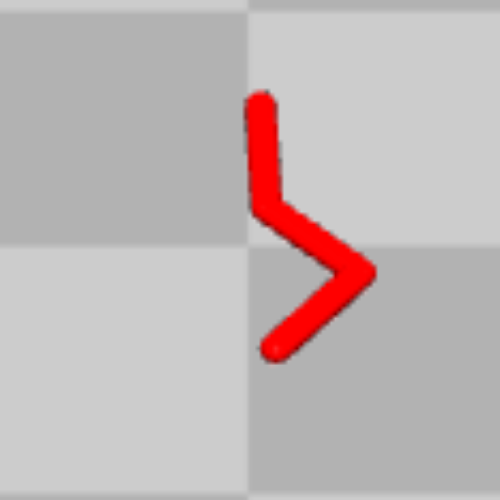}
      \caption{\footnotesize{Swimmer}}
      \label{fig: Swimmer}
  \end{subfigure}
  \begin{subfigure}[t]{0.119\linewidth}
        \centering
      \includegraphics[scale=0.085]{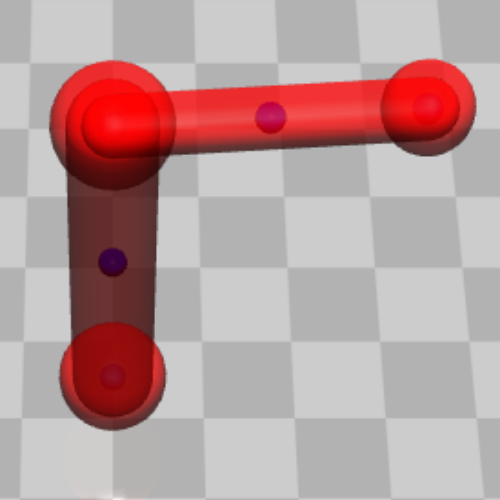}
      \caption{Arm3}
      \label{fig: Arm3}
  \end{subfigure}
  \begin{subfigure}[t]{0.119\linewidth}
        \centering
      \includegraphics[scale=0.085]{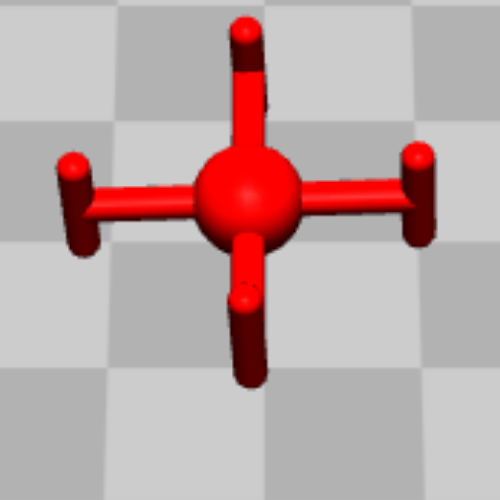}
      \caption{Drone}
      \label{fig: Drone}
  \end{subfigure}
  \begin{subfigure}[t]{0.119\linewidth}
        \centering
      \includegraphics[scale=0.085]{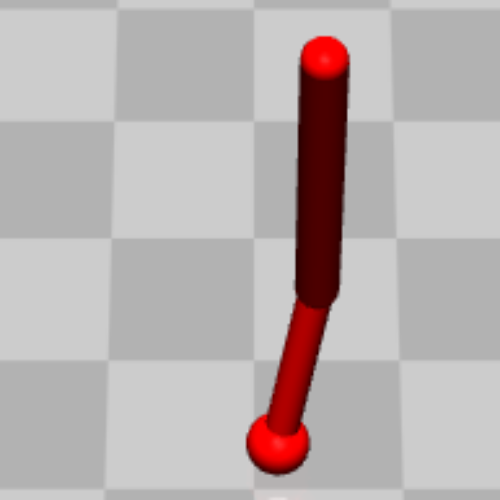}
      \caption{Hopper}
      \label{fig: Hopper}
  \end{subfigure}
  \begin{subfigure}[t]{0.119\linewidth}
        \centering
      \includegraphics[scale=0.085]{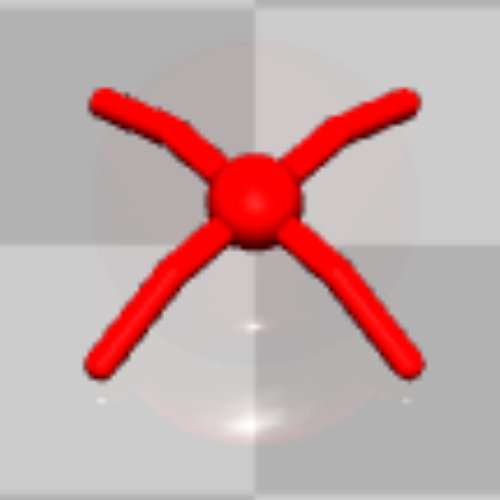}
      \caption{Ant}
      \label{fig: Ant}
  \end{subfigure}
  \begin{subfigure}[t]{0.119\linewidth}
        \centering
      \includegraphics[scale=0.085]{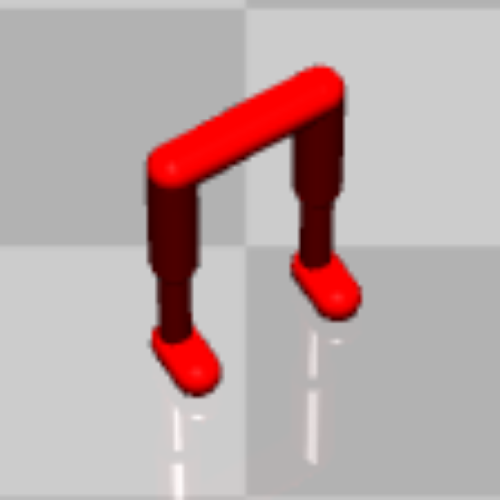}
      \caption{Walker}
      \label{fig: Walker}
  \end{subfigure}
\caption{Robots of continuous tasks benchmark GUARD.}
\label{fig: robots}
\vspace{-15pt}
\end{figure*}

\label{sec: experiment details}
\subsection{GUARD Environment Settings}
\label{appendix:Environment Settings}
\paragraph{Goal Task}
In the Goal task environments, the reward function is:
\begin{equation}\notag
\begin{split}
    & r(x_t) = d^{g}_{t-1} - d^{g}_{t} + \mathbf{1}[d^g_t < R^g]~,\\
\end{split}
\end{equation}
where $d^g_t$ is the distance from the robot to its closest goal and $R^g$ is the size (radius) of the goal. When a goal is achieved, the goal location is randomly reset to someplace new while keeping the rest of the layout the same. 

\paragraph{Push Task}
In the Push task environments, the reward function is
\begin{equation}\notag
\begin{split}
    & r(x_t) = d^{r}_{t-1} - d^{r}_{t} + d^{b}_{t-1} - d^{b}_t + \mathbf{1}[d^g_t < R^g]~,\\
\end{split}
\end{equation}
where $d^r$ and $d^b$ are the distance from the robot to its closest goal and the distance from the box to its closest goal, and $R^g$ is the size (radius) of the goal. The box size is 0.2 for all the Push task environments. Like the goal task, a new goal location is drawn each time a goal is achieved.

\paragraph{Chase Task}
In the Chase task environments, the reward function is
\begin{equation}\notag
\begin{split}
    & r(x_t) = d^{r}_{t-1} - d^{r}_{t} + \mathbf{1}[d^g_t < R^g]~,\\
\end{split}
\end{equation}
where $d^r$ is the distance from the robot to its closest goal and $R^g$ is the size (radius) of the goal. Those targets continuously move away from the robot at a slow speed. The dense reward component provides a bonus for minimizing the distance between the robot and the targets. The targets are constrained to a circular area.

The test suites of APO and PAPO continuous experiments are summarized in \Cref{tab: testing suites} and \Cref{tab: papo testing suites}, respectively.

\begin{table}[h]
\vskip 0.15in
\caption{The test suites environments of APO continuous experiments}
\begin{center}
\begin{tabular}{cc|ccc|ccc|cc}
\toprule
 \multicolumn{2}{c|}{\textbf{Task Settings}} & \multicolumn{3}{c|}{Moving Area} &\multicolumn{3}{c|}{Task} &\multicolumn{2}{c}{Dimension}\\
\cline{3-10}\\[-1.02em]
& & Ground & Aerial & \multicolumn{1}{c|}{Fixed}  & Goal & Push & Chase & Low & High\\
\cline{1-10}\\[-0.99em]
\multicolumn{1}{c|}{} & Arm3 ($\mathbb{R}^{3}$)   &  &  & \checkmark  & \checkmark&&&\checkmark & \cellcolor[HTML]{C0C0C0} \\
\multicolumn{1}{c|}{} & Drone ($\mathbb{R}^{4}$)   &  & \checkmark & & \checkmark&&&\checkmark & \cellcolor[HTML]{C0C0C0}\\
\multicolumn{1}{c|}{GUARD} & Point ($\mathbb{R}^{2}$)   &  \checkmark &  & & \checkmark&\checkmark&&\checkmark & \cellcolor[HTML]{C0C0C0} \\
\multicolumn{1}{c|}{} & Swimmer ($\mathbb{R}^{2}$)  & \checkmark &  & & \checkmark&\checkmark&&\checkmark & \cellcolor[HTML]{C0C0C0}\\
\multicolumn{1}{c|}{Robot} & Hopper ($\mathbb{R}^{5}$)    & \checkmark &  & & \checkmark&\checkmark&&\checkmark & \cellcolor[HTML]{C0C0C0}\\
\multicolumn{1}{c|}{} & Ant ($\mathbb{R}^{8}$)      & \checkmark &  & & &\checkmark&\checkmark & \cellcolor[HTML]{C0C0C0}&\checkmark\\
\multicolumn{1}{c|}{} & Walker ($\mathbb{R}^{10}$)   & \checkmark &  & & &\checkmark&\checkmark & \cellcolor[HTML]{C0C0C0}&\checkmark\\
\cline{1-10}\\[-0.99em]
\multicolumn{1}{c|}{Mujoco} & Humanoid ($\mathbb{R}^{17}$)  & \checkmark &  & &\cellcolor[HTML]{C0C0C0} &\cellcolor[HTML]{C0C0C0}&\cellcolor[HTML]{C0C0C0} & \cellcolor[HTML]{C0C0C0}&\checkmark\\
\multicolumn{1}{c|}{Robot} & HumanoidStandup ($\mathbb{R}^{17}$)  & \checkmark &  & & \cellcolor[HTML]{C0C0C0}&\cellcolor[HTML]{C0C0C0}&\cellcolor[HTML]{C0C0C0} & \cellcolor[HTML]{C0C0C0}&\checkmark\\
\bottomrule
\end{tabular}
\label{tab: testing suites}
\end{center}
\end{table}

\begin{table}[h]
\vskip 0.15in
\caption{The test suites environments of PAPO continuous experiments}
\begin{center}
\begin{tabular}{cc|ccc|cc|cc}
\toprule
 \multicolumn{2}{c|}{\textbf{Task Settings}} & \multicolumn{3}{c|}{Moving Area} &\multicolumn{2}{c|}{Task} &\multicolumn{2}{c}{Dimension}\\
\cline{3-9}\\[-1.02em]
& & Ground & Aerial & \multicolumn{1}{c|}{Fixed}  & Goal  & Chase & Low & High\\
\cline{1-9}\\[-0.99em]
\multicolumn{1}{c|}{} & Arm3 ($\mathbb{R}^{3}$)   &  &  & \checkmark  & \checkmark&&\checkmark & \cellcolor[HTML]{C0C0C0} \\
\multicolumn{1}{c|}{} & Drone ($\mathbb{R}^{4}$)   &  & \checkmark & & \checkmark&&\checkmark & \cellcolor[HTML]{C0C0C0}\\
\multicolumn{1}{c|}{GUARD} & Point ($\mathbb{R}^{2}$)   &  \checkmark &  & && \checkmark&\checkmark & \cellcolor[HTML]{C0C0C0} \\
\multicolumn{1}{c|}{} & Swimmer ($\mathbb{R}^{2}$)  & \checkmark &  & & &\checkmark&\checkmark & \cellcolor[HTML]{C0C0C0}\\
\multicolumn{1}{c|}{Robot} & Hopper ($\mathbb{R}^{5}$)    & \checkmark &  & & &\checkmark&\checkmark & \cellcolor[HTML]{C0C0C0}\\
\multicolumn{1}{c|}{} & Ant ($\mathbb{R}^{8}$)      & \checkmark &  & &\checkmark &\checkmark & \cellcolor[HTML]{C0C0C0}&\checkmark\\
\multicolumn{1}{c|}{} & Walker ($\mathbb{R}^{10}$)   & \checkmark &  & &\checkmark &\checkmark & \cellcolor[HTML]{C0C0C0}&\checkmark\\
\cline{1-9}\\[-0.99em]
\multicolumn{1}{c|}{Mujoco} & Humanoid ($\mathbb{R}^{17}$)  & \checkmark &  &  &\cellcolor[HTML]{C0C0C0}&\cellcolor[HTML]{C0C0C0} & \cellcolor[HTML]{C0C0C0}&\checkmark\\
\multicolumn{1}{c|}{Robot} & HumanoidStandup ($\mathbb{R}^{17}$)  & \checkmark &  & &\cellcolor[HTML]{C0C0C0}&\cellcolor[HTML]{C0C0C0} & \cellcolor[HTML]{C0C0C0}&\checkmark\\
\cline{1-9}\\[-0.99em]
\multicolumn{1}{c|}{Gymnasium} & HandReach ($\mathbb{R}^{24}$)  &  &  &\checkmark &\cellcolor[HTML]{C0C0C0}&\cellcolor[HTML]{C0C0C0} & \cellcolor[HTML]{C0C0C0}&\checkmark\\
\multicolumn{1}{c|}{} & HandManipulateEgg ($\mathbb{R}^{24}$)  &  &  &\checkmark &\cellcolor[HTML]{C0C0C0}&\cellcolor[HTML]{C0C0C0} & \cellcolor[HTML]{C0C0C0}&\checkmark\\
\multicolumn{1}{c|}{Robot} & HandManipulateBlock ($\mathbb{R}^{24}$)  &  &  &\checkmark &\cellcolor[HTML]{C0C0C0}&\cellcolor[HTML]{C0C0C0} & \cellcolor[HTML]{C0C0C0}&\checkmark\\
\bottomrule
\end{tabular}
\label{tab: papo testing suites}
\end{center}
\end{table}

\paragraph{State Space}
The internal state spaces describe the state of the robots, which can be obtained from standard robot sensors (accelerometer, gyroscope, magnetometer, velocimeter, joint position sensor, joint velocity sensor and touch sensor). The details of the internal state spaces of the robots in our test suites are summarized in \Cref{tab:internal_state_space}.

\begin{table}
\vskip 0.15in
\caption{The internal state space components of different test suites environments.}
\begin{center}
\begin{tabular}{c|ccccccc}
\toprule
\textbf{Internal State Space} & Point  & Swimmer & Walker & Ant & Drone & Hopper & Arm3 \\
\hline
Accelerometer ($\mathbb{R}^3$) & \checkmark & \checkmark & \checkmark & \checkmark & \checkmark &\checkmark &\checkmark \\
Gyroscope ($\mathbb{R}^3$) & \checkmark & \checkmark & \checkmark & \checkmark & \checkmark &\checkmark &\checkmark \\
Magnetometer ($\mathbb{R}^3$) & \checkmark & \checkmark & \checkmark & \checkmark & \checkmark &\checkmark &\checkmark \\
Velocimeter ($\mathbb{R}^{3}$) & \checkmark & \checkmark & \checkmark & \checkmark & \checkmark &\checkmark &\checkmark \\
Joint position sensor ($\mathbb{R}^{n}$) & ${n=0}$ & ${n=2}$ & ${n=10}$ & ${n=8}$ & ${n=0}$ &${n=6}$ &${n=3}$ \\
Joint velocity sensor ($\mathbb{R}^{n}$)  & ${n=0}$ & ${n=2}$ & ${n=10}$ & ${n=8}$ & ${n=0}$ &${n=6}$ &${n=3}$ \\
Touch sensor ($\mathbb{R}^{n}$) & ${n=0}$ & ${n=4}$ & ${n=2}$ & ${n=8}$ & ${n=0}$ &${n=1}$ &${n=1}$ \\
\bottomrule
\end{tabular}
\label{tab:internal_state_space}
\end{center}
\end{table}

\paragraph{Control Space}
For all the experiments, the control space of all robots are continuous, and linearly scaled to [-1, +1].

\subsection{Policy Settings}
\label{appendix:Policy Settings}
The hyper-parameters used in our experiments are listed in \Cref{tab:policy_setting} as default.

Our experiments use separate multi-layer perceptrons with ${tanh}$ activations for the policy network and value network. Each network consists of two hidden layers of size (64,64). Policy networks and value networks are trained using $Adam$ optimizer. Policy networks are trained with a learning rate of 1e-3 while value networks are trained with 3e-4. 

We apply an on-policy framework in our experiments. During each epoch the agent interacts $B$ times with the environment and then performs a policy update based on the experience collected from the current epoch. The maximum length of the trajectory is set to 1000. The steps in each epoch are set to 30000. The total epoch number $N$ is set to 200 in continuous tasks and 500 in atari tasks as default. 

The policy update step is based on the scheme of TRPO, which performs up to 100 steps of backtracking with a coefficient of 0.8 for line searching.

For all experiments, we use a discount factor of $\gamma = 0.99$, an advantage discount factor $\lambda =0.97$, and a KL-divergence step size of $\delta_{KL} = 0.02$.

Other unique hyper-parameters for each algorithm follow the original paper to attain best performance. 

Each model is trained on a server with a 48-core Intel(R) Xeon(R) Silver 4214 CPU @ 2.2.GHz, Nvidia RTX A4000 GPU with 16GB memory, and Ubuntu 20.04.

\begin{table}
\vskip 0.15in
\caption{Important hyper-parameters of different algorithms in our experiments}
\begin{center}
\resizebox{\textwidth}{!}{%
\begin{tabular}{lr|cccccccc}
\toprule
\textbf{Policy Parameter} & & A2C & TRPO & PPO & APO & PAPO& ESPO& $\alpha$-PPO& V-MPO\\
\hline\\[-1.0em]
Epochs in continuous tasks & $N_1$ & 200  & 200 & 200 & 200 & 200& 200 & 200 & 200 \\
Epochs in discrete tasks & $N_2$ & 500  & 500 & 500 & 500 & - & - & - & - \\
Steps per epoch & $B$& 30000 & 30000 & 30000 & 30000 & 30000 & 30000 & 30000 & 30000\\
Maximum length of trajectory & $L$ & 1000 & 1000 & 1000 & 1000 & 1000 & 1000 & 1000 & 1000\\

Discount factor  &  $\gamma$ & 0.99 & 0.99 & 0.99 & 0.99 & 0.99& 0.99 & 0.99 & 0.99\\
Advantage discount factor  & $\lambda$ & 0.97 & 0.97 & 0.97 & 0.97 &0.97& 0.97 & 0.97 &0.97\\
backtracking steps & & - &100 &- &100 & -& -& -& -\\
backtracking coefficient & & - &0.8 &- &0.8 & -& -& -& -\\
Target KL & $\delta_{KL}$& - & 0.02 & 0.02 & 0.02 & 0.02 & - & 0.02 & 0.02\\
Stopping Threshold & $\delta$& - & - & - & - & - & 0.25 & - & -\\
Clip Ratio & $\epsilon$ & - & - & 0.2 & - & 0.2& -& 0.2& -\\
Probability factor & $k$ & - & - & - & 7 & 7 & - & - & -\\
Policy network hidden layers & & (64, 64) & (64, 64) & (64, 64) & (64, 64) & (64, 64) & (64, 64) & (64, 64) & (64, 64) \\
Policy network iteration & & 80 & - & 80 & - & 80 & 80 & 80 & 80 \\
Policy network optimizer & & Adam & - & Adam & - & Adam & Adam & Adam & Adam \\
Policy learning rate & & 3e-4 & - & 3e-4 & - & 3e-4 & 3e-4 & 3e-4 & 3e-4 \\
Value network hidden layers & & (64, 64) & (64, 64) & (64, 64) & (64, 64) & (64, 64) & (64, 64) & (64, 64) & (64, 64)\\
Value network iteration & & 80 & 80 & 80 & 80 & 80 & 80 & 80 & 80\\
Value network optimizer & & Adam & Adam & Adam & Adam & Adam & Adam & Adam & Adam\\
Value learning rate & & 1e-3 & 1e-3 & 1e-3 & 1e-3 & 1e-3 & 1e-3 & 1e-3 & 1e-3\\
\bottomrule
\end{tabular}
}
\label{tab:policy_setting}
\end{center}
\end{table}

\subsection{Normalized Score Settings}
\label{appendix:Score Settings}
We intend to use 100\% as the TRPO baseline measure. Thus the score we used is slightly changed from the normalization algorithm proposed by \cite{vanhasselt2015deep} with the form as follow:
\begin{align}
\nonumber
    {\Delta}_1 &\doteq score_{agent}-score_{random} \\ \nonumber
    {\Delta}_2 &\doteq score_{TRPO}-score_{random} \\ \nonumber
    score_{normalized} &= \frac{\Delta_2}{\Delta_1} ~ if ~ \Delta_1 < 0 ~ and ~ \Delta_2 <0 ~ else ~ \frac{\Delta_1}{\Delta_2}
\end{align}
The difference is that we take the inverse of the metrics used in \cite{vanhasselt2015deep} when $\Delta_1 < 0 ~ and ~ \Delta_2 <0$. For further explanation, we need to address the positive and negative cases of $\Delta_1$ and $\Delta_2$:
\begin{itemize}
    \item $\Delta_1 > 0 ~ and ~ \Delta_2 > 0~~$ In this case, $\frac{\Delta_1}{\Delta_2}$ can effectively demonstrate the ability of algorithms.
    \item $\Delta_1 < 0 ~ and ~ \Delta_2 < 0~~$ Practically in this case, assuming $score_{agent}>score_{TRPO}$, $score_{normalized}$ should be greater than 1 to demonstrate that agent is more capable than baseline TRPO. However, we will get a decimal if we obey the original score algorithms in \cite{vanhasselt2015deep} which is incorrect. Thus we take the inverse of it in this situation.
    \item $\Delta_1 < 0 ~ and ~ \Delta_2 > 0~~$ In this case we will get a negative number which is reasonable to show the negative effects in terms of reward enhancement.
    \item $\Delta_1 > 0 ~ and ~ \Delta_2 < 0~~$ The changed normalization algorithm is still incorrect in this situation. However, we have not encountered such cases in all of our atari game statistics.
\end{itemize}

\section{Total Experiment results}
\label{sec:total experiments}

\begin{figure*}[h]
    \centering

    \begin{subfigure}[t]{0.24\textwidth}
        \begin{subfigure}[t]{1.00\textwidth}
            \raisebox{-\height}{\includegraphics[width=\textwidth]{fig/continuous/Goal_Point_Reward_Performance.png}}
        \end{subfigure}
    \hfill
        \begin{subfigure}[t]{1.00\textwidth}
            \raisebox{-\height}{\includegraphics[width=\textwidth]{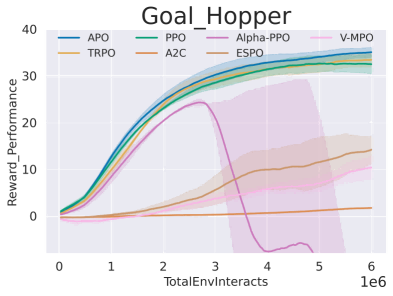}}
        \end{subfigure}
    \hfill
        \begin{subfigure}[t]{1.00\textwidth}
            \raisebox{-\height}{\includegraphics[width=\textwidth]{fig/continuous/Push_Ant_Reward_Performance.png}}
        \end{subfigure}
    \end{subfigure}
    \hfill
    \begin{subfigure}[t]{0.24\textwidth}
        \begin{subfigure}[t]{1.00\textwidth}
            \raisebox{-\height}{\includegraphics[width=\textwidth]{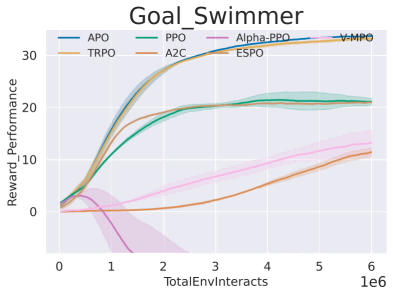}}
        \end{subfigure}
    \hfill
        \begin{subfigure}[t]{1.00\textwidth}
            \raisebox{-\height}{\includegraphics[width=\textwidth]{fig/continuous/Push_Point_Reward_Performance.png}}
        \end{subfigure}
    \hfill
        \begin{subfigure}[t]{1.00\textwidth}
            \raisebox{-\height}{\includegraphics[width=\textwidth]{fig/continuous/Push_Walker_Reward_Performance.png}}
        \end{subfigure}
    \end{subfigure}
    \hfill
    \begin{subfigure}[t]{0.24\textwidth}
        \begin{subfigure}[t]{1.00\textwidth}
            \raisebox{-\height}{\includegraphics[width=\textwidth]{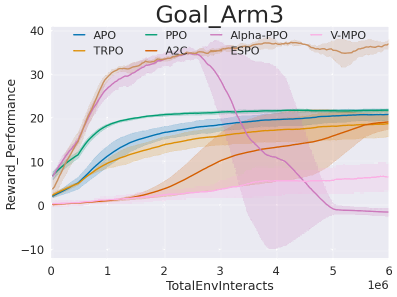}}
        \end{subfigure}
    \hfill
        \begin{subfigure}[t]{1.00\textwidth}
            \raisebox{-\height}{\includegraphics[width=\textwidth]{fig/continuous/Push_Swimmer_Reward_Performance.png}}
        \end{subfigure}
    \hfill
        \begin{subfigure}[t]{1.00\textwidth}
            \raisebox{-\height}{\includegraphics[width=\textwidth]{fig/continuous/Chase_Ant_Reward_Performance.png}}
        \end{subfigure}
    \end{subfigure}
    \hfill
    \begin{subfigure}[t]{0.24\textwidth}
        \begin{subfigure}[t]{1.00\textwidth}
            \raisebox{-\height}{\includegraphics[width=\textwidth]{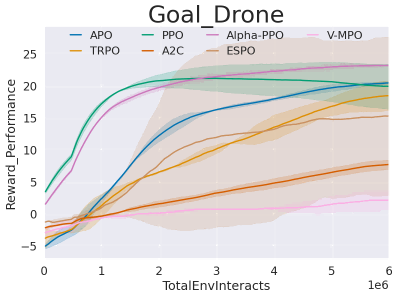}}
        \end{subfigure}
    \hfill
        \begin{subfigure}[t]{1.00\textwidth}
            \raisebox{-\height}{\includegraphics[width=\textwidth]{fig/continuous/Push_Hopper_Reward_Performance.png}}
        \end{subfigure}
    \hfill
        \begin{subfigure}[t]{1.00\textwidth}
            \raisebox{-\height}{\includegraphics[width=\textwidth]{fig/continuous/Chase_Walker_Reward_Performance.png}}
        \end{subfigure}
    \end{subfigure}
    \caption{All continuous GUARD tasks expected performance learning curves for APO test suites.}
    \label{fig: all guard figures}
\end{figure*}

\begin{figure*}[h]
    \centering
    \begin{subfigure}[t]{0.24\textwidth}
        \begin{subfigure}[t]{1.00\textwidth}
            \raisebox{-\height}{\includegraphics[width=\textwidth]{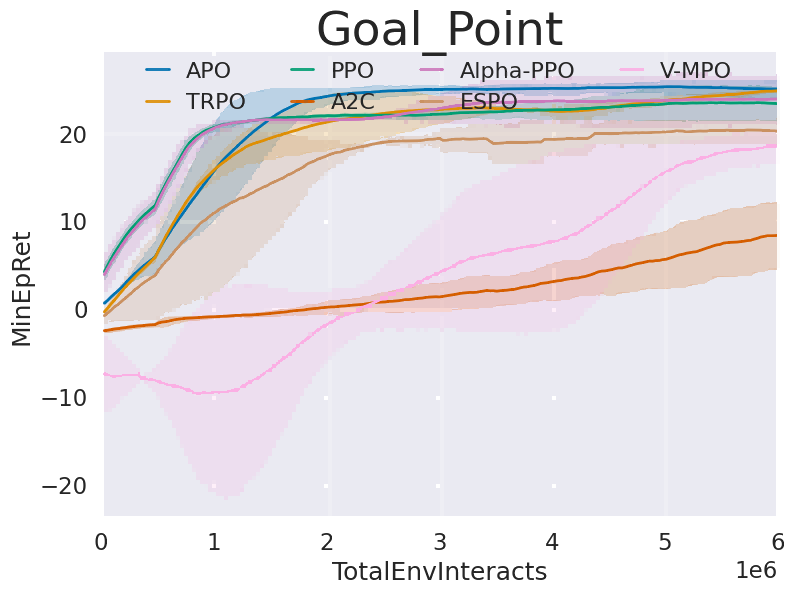}}
        \end{subfigure}
    \hfill
        \begin{subfigure}[t]{1.00\textwidth}
            \raisebox{-\height}{\includegraphics[width=\textwidth]{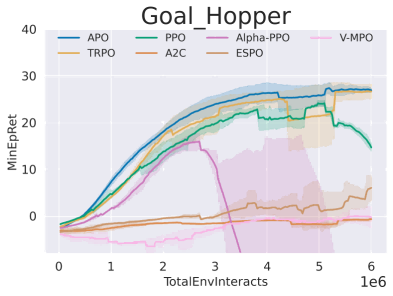}}
        \end{subfigure}
    \hfill
        \begin{subfigure}[t]{1.00\textwidth}
            \raisebox{-\height}{\includegraphics[width=\textwidth]{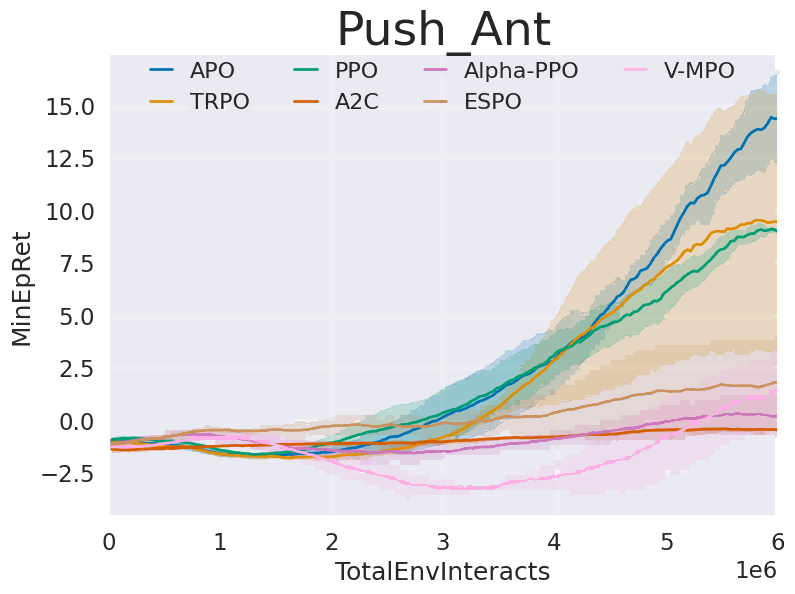}}
        \end{subfigure}
    \end{subfigure}
    \hfill
    \begin{subfigure}[t]{0.24\textwidth}
        \begin{subfigure}[t]{1.00\textwidth}
            \raisebox{-\height}{\includegraphics[width=\textwidth]{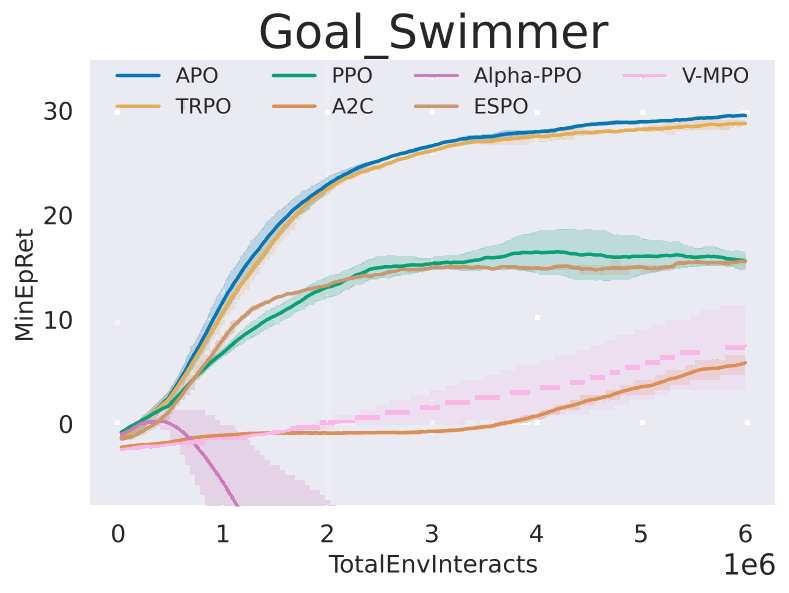}}
        \end{subfigure}
    \hfill
        \begin{subfigure}[t]{1.00\textwidth}
            \raisebox{-\height}{\includegraphics[width=\textwidth]{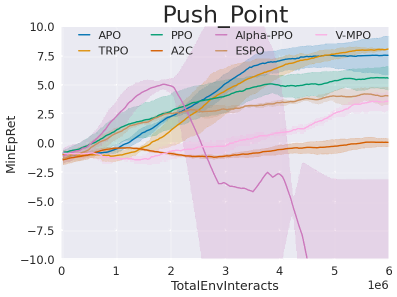}}
        \end{subfigure}
    \hfill
        \begin{subfigure}[t]{1.00\textwidth}
            \raisebox{-\height}{\includegraphics[width=\textwidth]{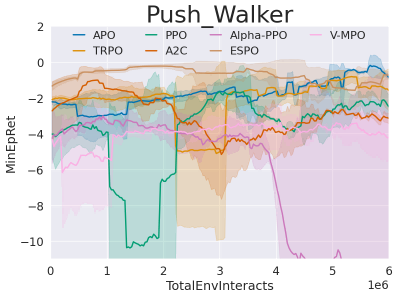}}
        \end{subfigure}
    \end{subfigure}
    \hfill
    \begin{subfigure}[t]{0.24\textwidth}
        \begin{subfigure}[t]{1.00\textwidth}
            \raisebox{-\height}{\includegraphics[width=\textwidth]{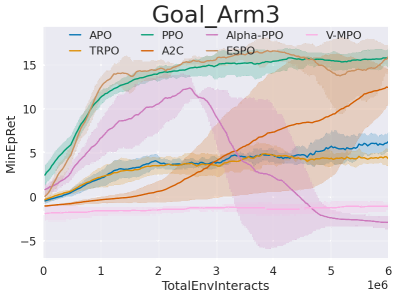}}
        \end{subfigure}
    \hfill
        \begin{subfigure}[t]{1.00\textwidth}
            \raisebox{-\height}{\includegraphics[width=\textwidth]{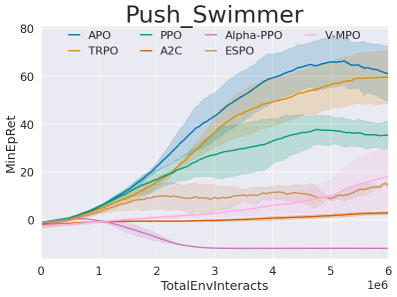}}
        \end{subfigure}
    \hfill
        \begin{subfigure}[t]{1.00\textwidth}
            \raisebox{-\height}{\includegraphics[width=\textwidth]{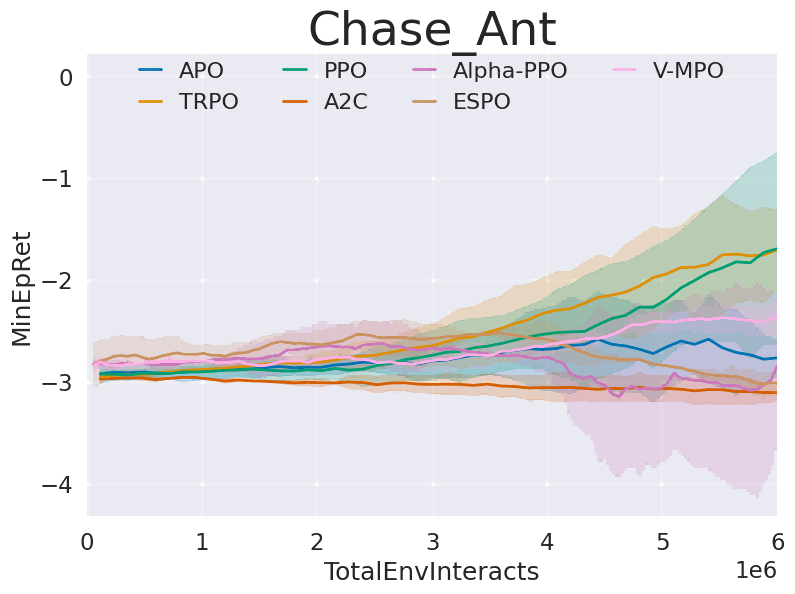}}
        \end{subfigure}
    \end{subfigure}
    \hfill
    \begin{subfigure}[t]{0.24\textwidth}
        \begin{subfigure}[t]{1.00\textwidth}
            \raisebox{-\height}{\includegraphics[width=\textwidth]{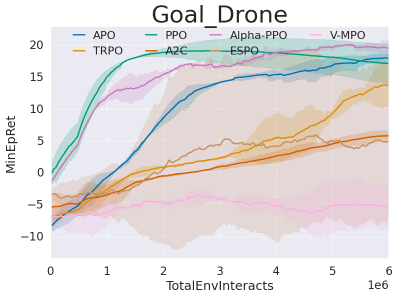}}
        \end{subfigure}
    \hfill
        \begin{subfigure}[t]{1.00\textwidth}
            \raisebox{-\height}{\includegraphics[width=\textwidth]{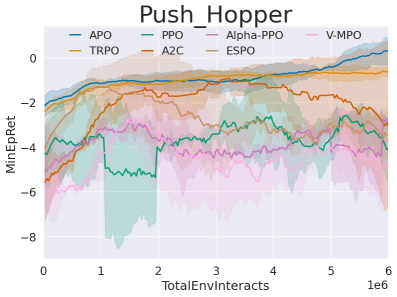}}
        \end{subfigure}
    \hfill
        \begin{subfigure}[t]{1.00\textwidth}
            \raisebox{-\height}{\includegraphics[width=\textwidth]{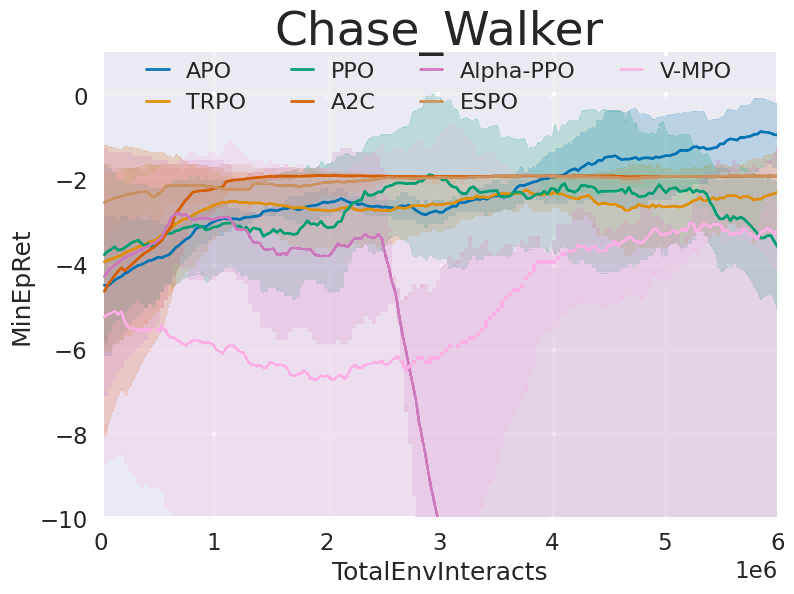}}
        \end{subfigure}
    \end{subfigure}

    \caption{All continuous GUARD tasks worst performance learning curves for APO test suites.}
    \label{fig: all guard min figures}
\end{figure*}

\begin{figure*}[h]
    \centering

    \begin{subfigure}[t]{0.24\textwidth}
        \begin{subfigure}[t]{1.00\textwidth}
            \raisebox{-\height}{\includegraphics[width=\textwidth]{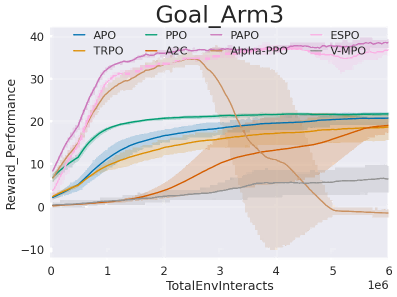}}
        \end{subfigure}
    \hfill
        \begin{subfigure}[t]{1.00\textwidth}
            \raisebox{-\height}{\includegraphics[width=\textwidth]{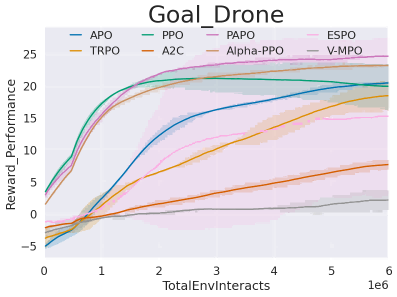}}
        \end{subfigure}
    \hfill
        \begin{subfigure}[t]{1.00\textwidth}
            \raisebox{-\height}{\includegraphics[width=\textwidth]{fig/papo/Goal_Ant_Reward_Performance.png}}
        \end{subfigure}
    \end{subfigure}
    \hfill
    \begin{subfigure}[t]{0.24\textwidth}
        \begin{subfigure}[t]{1.00\textwidth}
            \raisebox{-\height}{\includegraphics[width=\textwidth]{fig/papo/Goal_Walker_Reward_Performance.png}}
        \end{subfigure}
    \hfill
        \begin{subfigure}[t]{1.00\textwidth}
            \raisebox{-\height}{\includegraphics[width=\textwidth]{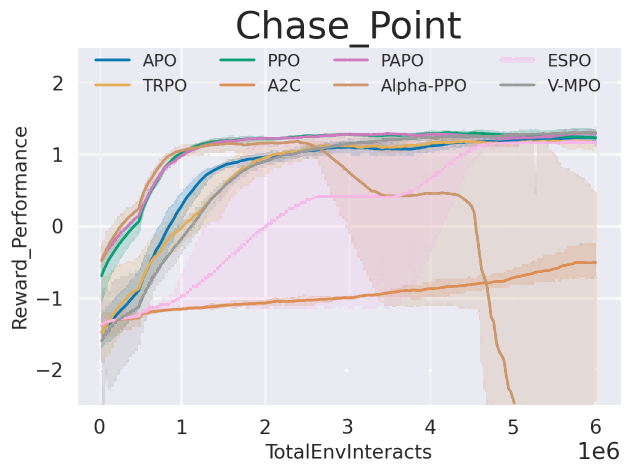}}
        \end{subfigure}
    \end{subfigure}
    \hfill
    \begin{subfigure}[t]{0.24\textwidth}
        \begin{subfigure}[t]{1.00\textwidth}
            \raisebox{-\height}{\includegraphics[width=\textwidth]{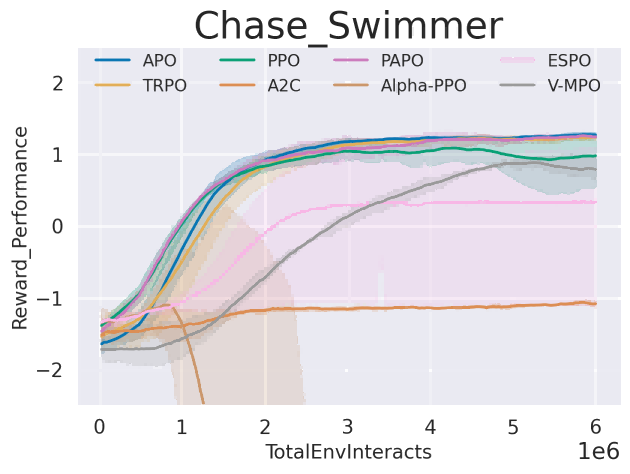}}
        \end{subfigure}
    \hfill
        \begin{subfigure}[t]{1.00\textwidth}
            \raisebox{-\height}{\includegraphics[width=\textwidth]{fig/papo/Chase_Hopper_Reward_Performance.png}}
        \end{subfigure}
    \end{subfigure}
    \hfill
    \begin{subfigure}[t]{0.24\textwidth}
        \begin{subfigure}[t]{1.00\textwidth}
            \raisebox{-\height}{\includegraphics[width=\textwidth]{fig/papo/Chase_Ant_Reward_Performance.png}}
        \end{subfigure}
    \hfill
        \begin{subfigure}[t]{1.00\textwidth}
            \raisebox{-\height}{\includegraphics[width=\textwidth]{fig/papo/Chase_Walker_Reward_Performance.png}}
        \end{subfigure}
    \end{subfigure}
    \caption{All continuous GUARD tasks expected performance learning curves for PAPO test suites.}
    \label{fig: all guard figures papo}
\end{figure*}

\begin{figure*}[h]
    \centering

    \begin{subfigure}[t]{0.24\textwidth}
        \begin{subfigure}[t]{1.00\textwidth}
            \raisebox{-\height}{\includegraphics[width=\textwidth]{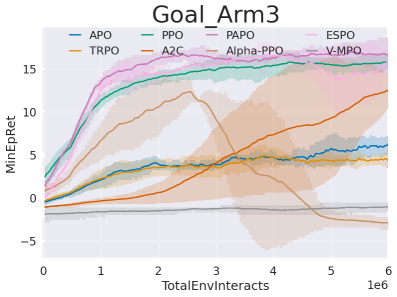}}
        \end{subfigure}
    \hfill
        \begin{subfigure}[t]{1.00\textwidth}
            \raisebox{-\height}{\includegraphics[width=\textwidth]{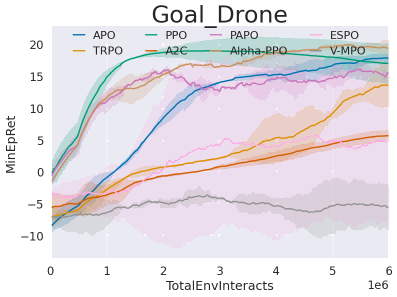}}
        \end{subfigure}
    \hfill
        \begin{subfigure}[t]{1.00\textwidth}
            \raisebox{-\height}{\includegraphics[width=\textwidth]{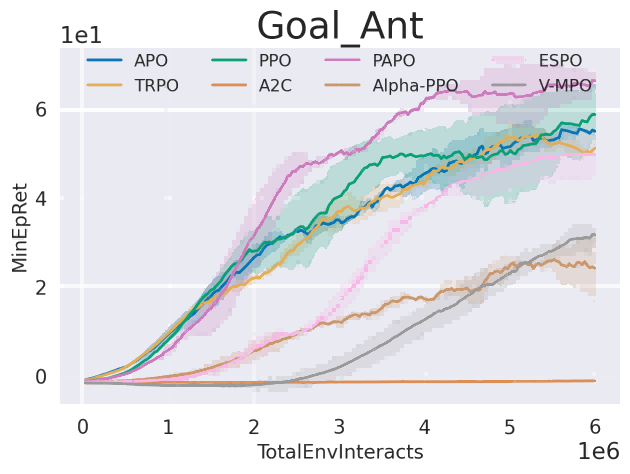}}
        \end{subfigure}
    \end{subfigure}
    \hfill
    \begin{subfigure}[t]{0.24\textwidth}
        \begin{subfigure}[t]{1.00\textwidth}
            \raisebox{-\height}{\includegraphics[width=\textwidth]{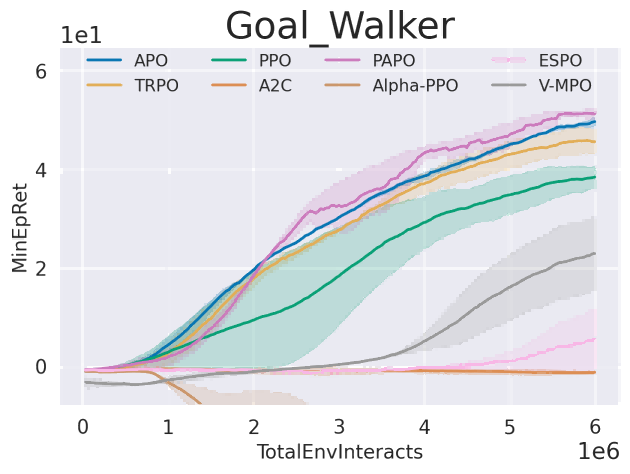}}
        \end{subfigure}
    \hfill
        \begin{subfigure}[t]{1.00\textwidth}
            \raisebox{-\height}{\includegraphics[width=\textwidth]{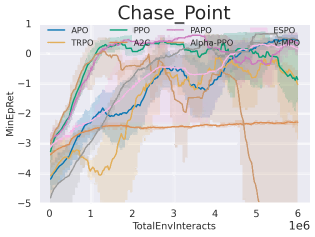}}
        \end{subfigure}
    \end{subfigure}
    \hfill
    \begin{subfigure}[t]{0.24\textwidth}
        \begin{subfigure}[t]{1.00\textwidth}
            \raisebox{-\height}{\includegraphics[width=\textwidth]{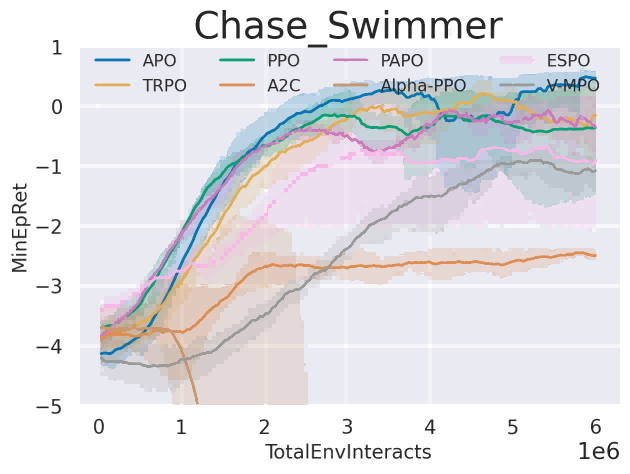}}
        \end{subfigure}
    \hfill
        \begin{subfigure}[t]{1.00\textwidth}
            \raisebox{-\height}{\includegraphics[width=\textwidth]{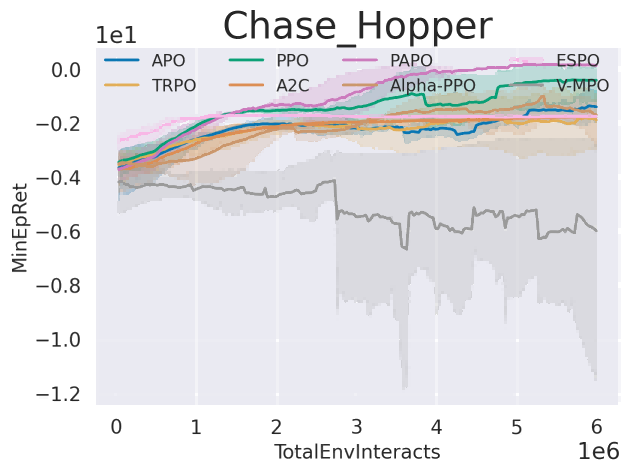}}
        \end{subfigure}
    \end{subfigure}
    \hfill
    \begin{subfigure}[t]{0.24\textwidth}
        \begin{subfigure}[t]{1.00\textwidth}
            \raisebox{-\height}{\includegraphics[width=\textwidth]{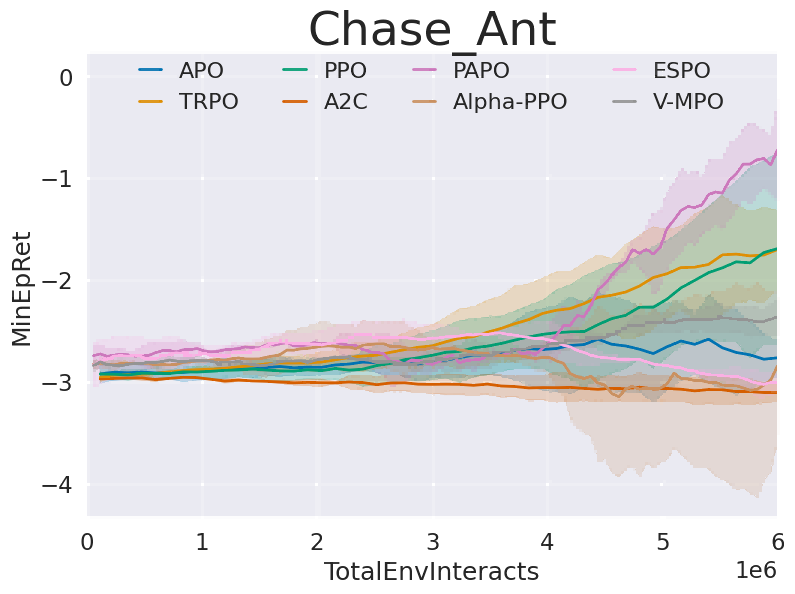}}
        \end{subfigure}
    \hfill
        \begin{subfigure}[t]{1.00\textwidth}
            \raisebox{-\height}{\includegraphics[width=\textwidth]{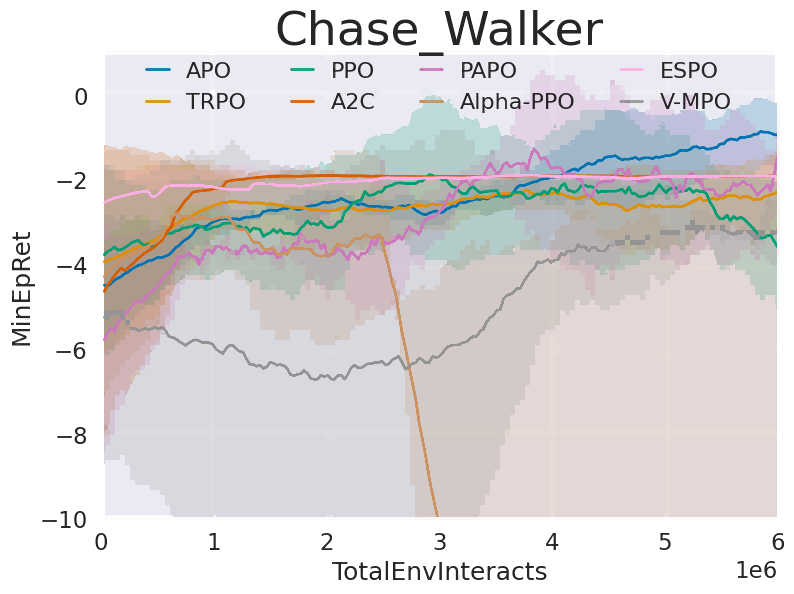}}
        \end{subfigure}
    \end{subfigure}
    \caption{All continuous GUARD tasks worst performance learning curves for PAPO test suites.}
    \label{fig: all guard min figures papo}
\end{figure*}

\begin{figure}[t]
    \vspace{-40pt}
    \begin{subfigure}[t]{0.204\textwidth}
    \begin{subfigure}[t]{1.000\textwidth}
        \raisebox{-\height}{\includegraphics[width=\textwidth]{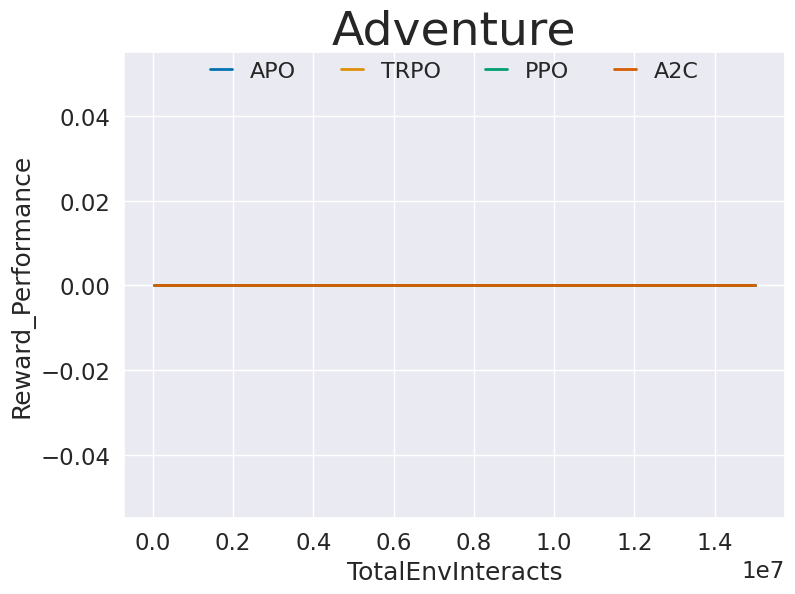}}
    \end{subfigure}
    \hfill
    \begin{subfigure}[t]{1.000\textwidth}
        \raisebox{-\height}{\includegraphics[width=\textwidth]{fig/atari/AirRaid_Reward_Performance.png}}
    \end{subfigure}
    \hfill
    \begin{subfigure}[t]{1.000\textwidth}
        \raisebox{-\height}{\includegraphics[width=\textwidth]{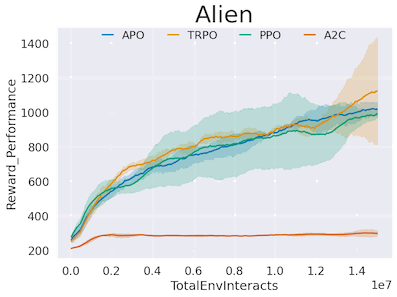}}
    \end{subfigure}
    \hfill
    \begin{subfigure}[t]{1.000\textwidth}
        \raisebox{-\height}{\includegraphics[width=\textwidth]{fig/atari/Amidar_Reward_Performance.png}}
    \end{subfigure}
    \hfill
    \begin{subfigure}[t]{1.000\textwidth}
        \raisebox{-\height}{\includegraphics[width=\textwidth]{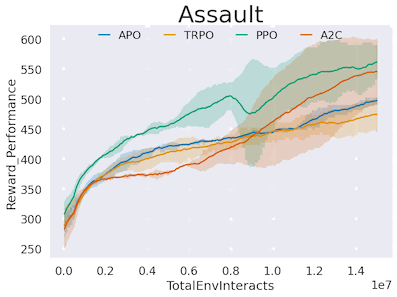}}
    \end{subfigure}
    \hfill
    \begin{subfigure}[t]{1.000\textwidth}
        \raisebox{-\height}{\includegraphics[width=\textwidth]{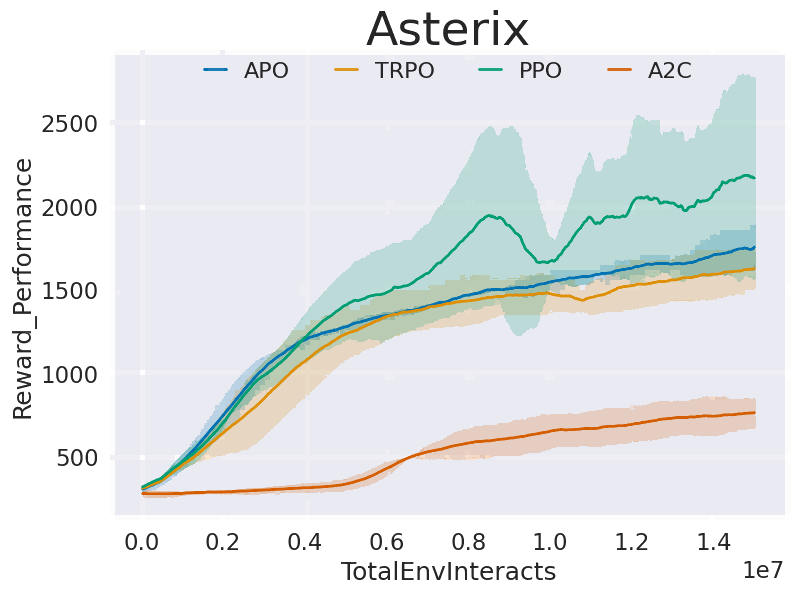}}
    \end{subfigure}
    \hfill
    \begin{subfigure}[t]{1.000\textwidth}
        \raisebox{-\height}{\includegraphics[width=\textwidth]{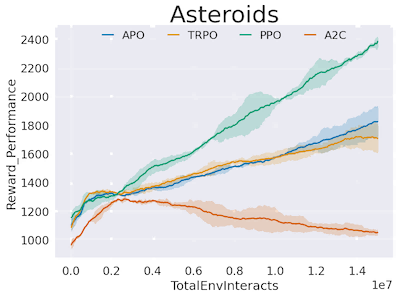}}
    \end{subfigure}
    \hfill
    \begin{subfigure}[t]{1.000\textwidth}
        \raisebox{-\height}{\includegraphics[width=\textwidth]{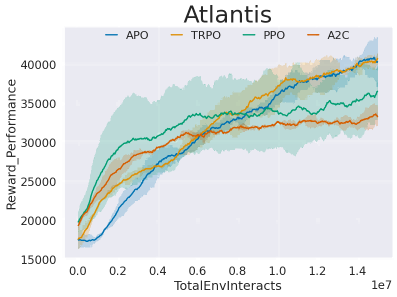}}
    \end{subfigure}
    \end{subfigure}
    \hfill
    \begin{subfigure}[t]{0.204\textwidth}
    \begin{subfigure}[t]{1.000\textwidth}
        \raisebox{-\height}{\includegraphics[width=\textwidth]{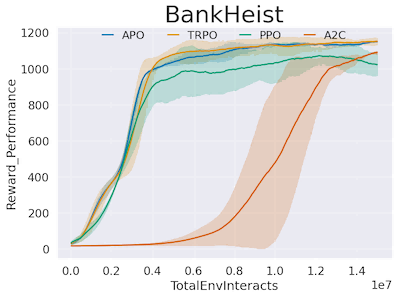}}
    \end{subfigure}
    \hfill
    \begin{subfigure}[t]{1.000\textwidth}
        \raisebox{-\height}{\includegraphics[width=\textwidth]{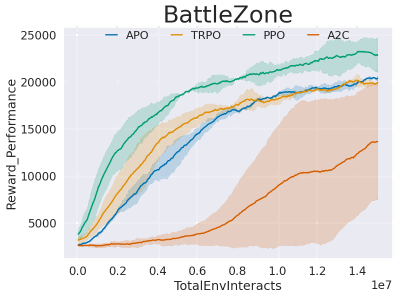}}
    \end{subfigure}
    \hfill
    \begin{subfigure}[t]{1.000\textwidth}
        \raisebox{-\height}{\includegraphics[width=\textwidth]{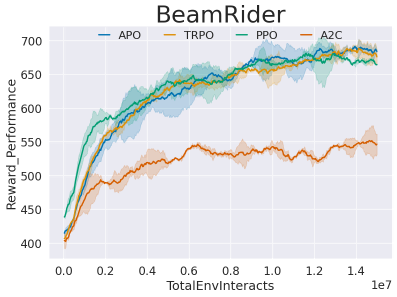}}
    \end{subfigure}
    \hfill
    \begin{subfigure}[t]{1.000\textwidth}
        \raisebox{-\height}{\includegraphics[width=\textwidth]{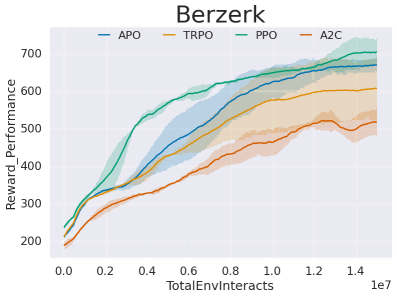}}
    \end{subfigure}
    \hfill
    \begin{subfigure}[t]{1.000\textwidth}
        \raisebox{-\height}{\includegraphics[width=\textwidth]{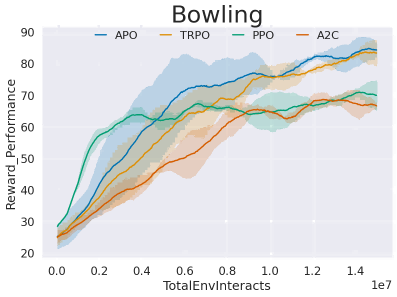}}
    \end{subfigure}
    \hfill
    \begin{subfigure}[t]{1.000\textwidth}
        \raisebox{-\height}{\includegraphics[width=\textwidth]{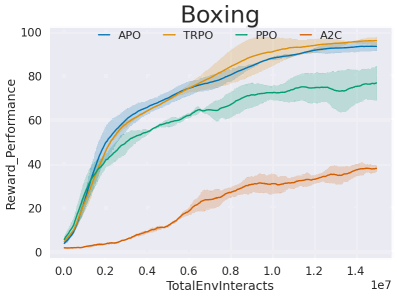}}
    \end{subfigure}
    \hfill
    \begin{subfigure}[t]
    {1.000\textwidth}
        \raisebox{-\height}{\includegraphics[width=\textwidth]{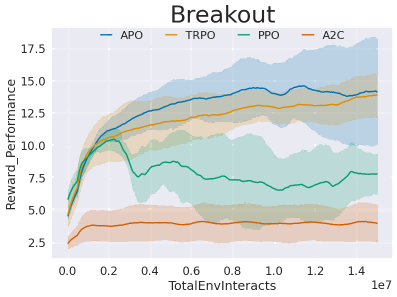}}
    \end{subfigure}
    \hfill
    \begin{subfigure}[t]{1.000\textwidth}
        \raisebox{-\height}{\includegraphics[width=\textwidth]{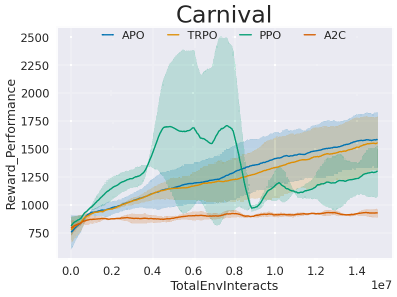}}
    \end{subfigure}
    \end{subfigure}
    \hfill
    \begin{subfigure}[t]{0.204\textwidth}
    \begin{subfigure}[t]{1.000\textwidth}
        \raisebox{-\height}{\includegraphics[width=\textwidth]{fig/atari/Centipede_Reward_Performance.png}}
    \end{subfigure}
    \hfill
    \begin{subfigure}[t]{1.000\textwidth}
        \raisebox{-\height}{\includegraphics[width=\textwidth]{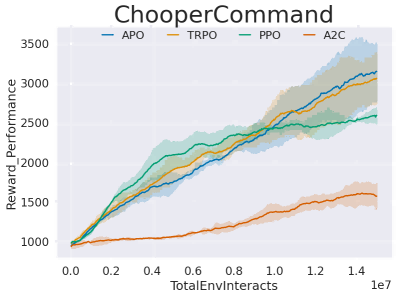}}
    \end{subfigure}
    \hfill
    \begin{subfigure}[t]{1.000\textwidth}
        \raisebox{-\height}{\includegraphics[width=\textwidth]{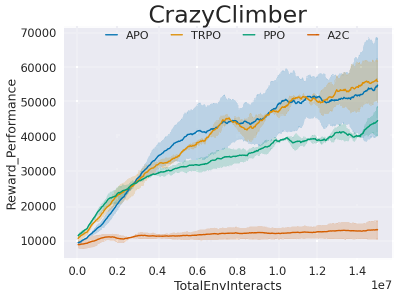}}
    \end{subfigure}
    \hfill
    \begin{subfigure}[t]{1.000\textwidth}
        \raisebox{-\height}{\includegraphics[width=\textwidth]{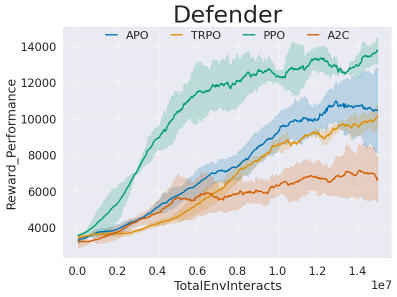}}
    \end{subfigure}
    \hfill
    \begin{subfigure}[t]{1.000\textwidth}
        \raisebox{-\height}{\includegraphics[width=\textwidth]{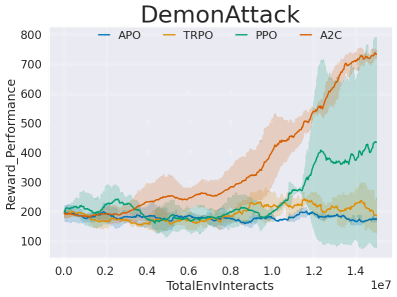}}
    \end{subfigure}
    \hfill
    \begin{subfigure}[t]{1.000\textwidth}
        \raisebox{-\height}{\includegraphics[width=\textwidth]{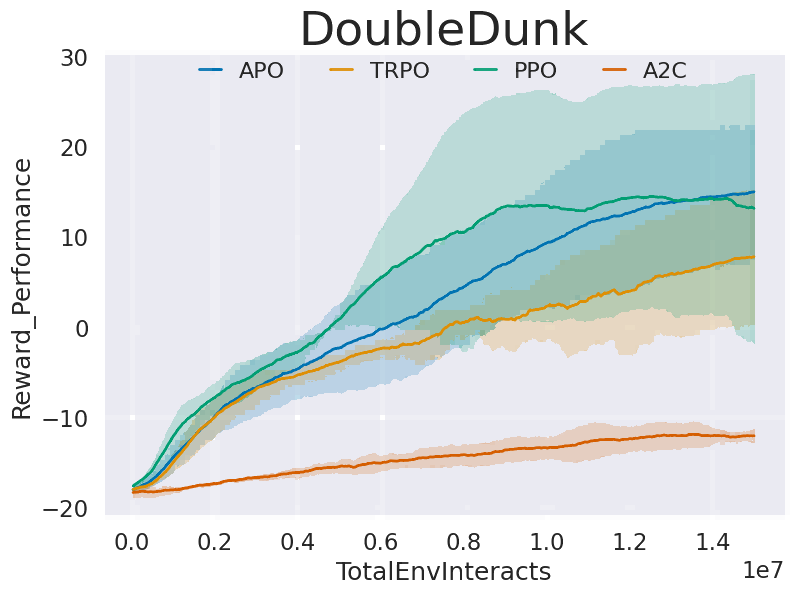}}
    \end{subfigure}
    \hfill
    \begin{subfigure}[t]{1.000\textwidth}
        \raisebox{-\height}{\includegraphics[width=\textwidth]{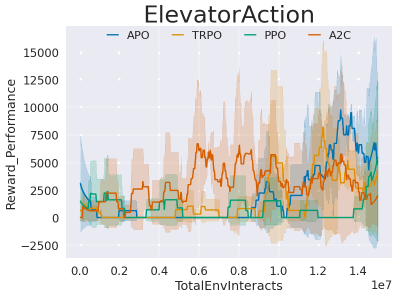}}
    \end{subfigure}
    \hfill
    \begin{subfigure}[t]{1.000\textwidth}
        \raisebox{-\height}{\includegraphics[width=\textwidth]{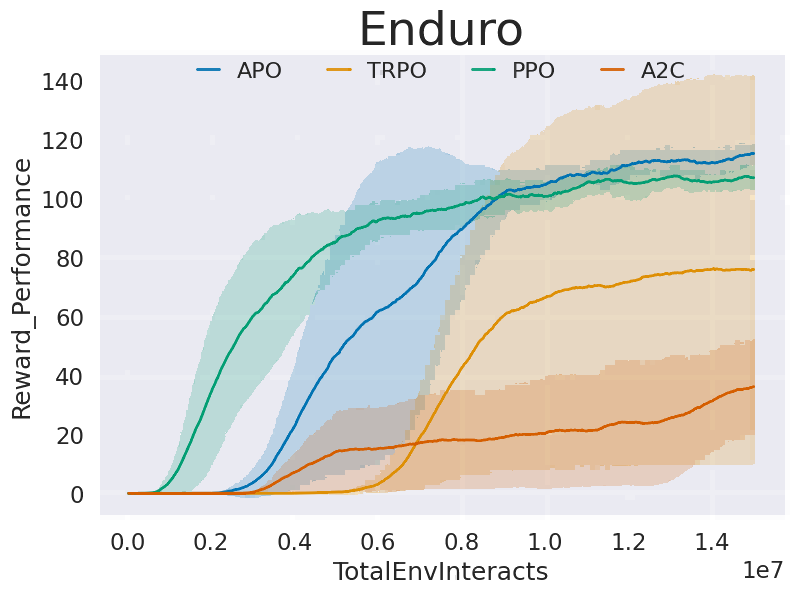}}
    \end{subfigure}
    \end{subfigure}
    \hfill
    \begin{subfigure}[t]{0.204\textwidth}
    \begin{subfigure}[t]{1.000\textwidth}
        \raisebox{-\height}{\includegraphics[width=\textwidth]{fig/atari/FishingDerby_Reward_Performance.png}}
    \end{subfigure}
    \hfill
    \begin{subfigure}[t]{1.000\textwidth}
        \raisebox{-\height}{\includegraphics[width=\textwidth]{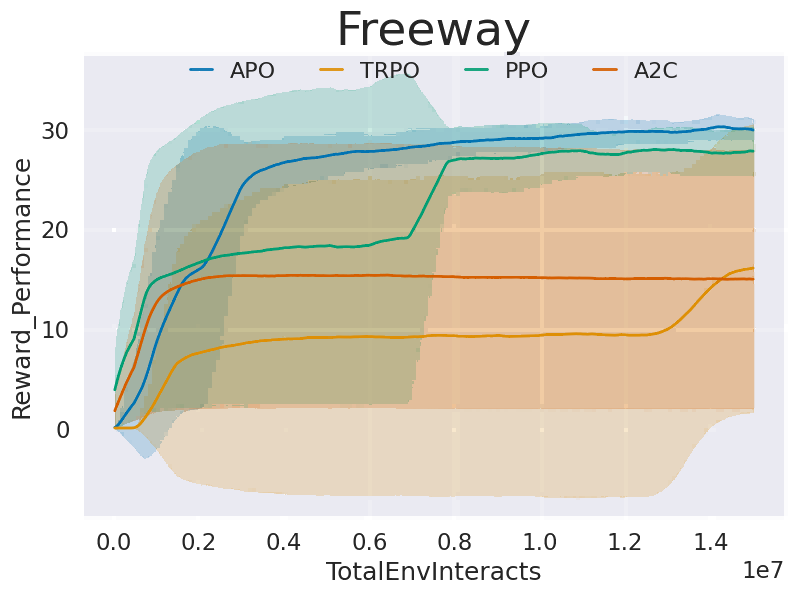}}
    \end{subfigure}
    \hfill
    \begin{subfigure}[t]{1.000\textwidth}
        \raisebox{-\height}{\includegraphics[width=\textwidth]{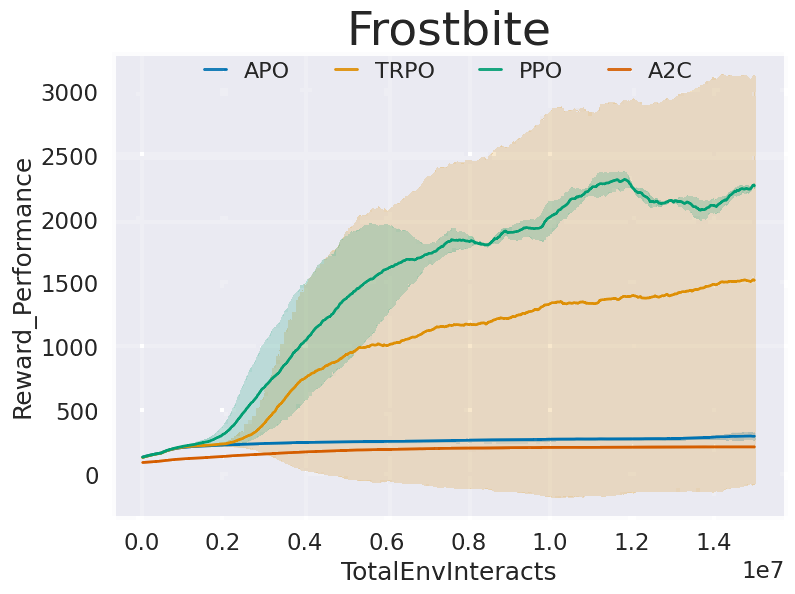}}
    \end{subfigure}
    \hfill
    \begin{subfigure}[t]{1.000\textwidth}
        \raisebox{-\height}{\includegraphics[width=\textwidth]{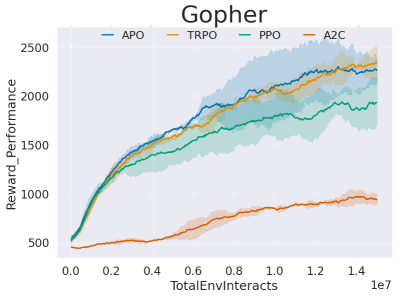}}
    \end{subfigure}
    \hfill
    \begin{subfigure}[t]{1.000\textwidth}
        \raisebox{-\height}{\includegraphics[width=\textwidth]{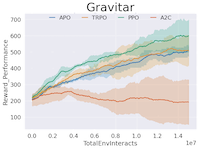}}
    \end{subfigure}
    \hfill
    \begin{subfigure}[t]{1.000\textwidth}
        \raisebox{-\height}{\includegraphics[width=\textwidth]{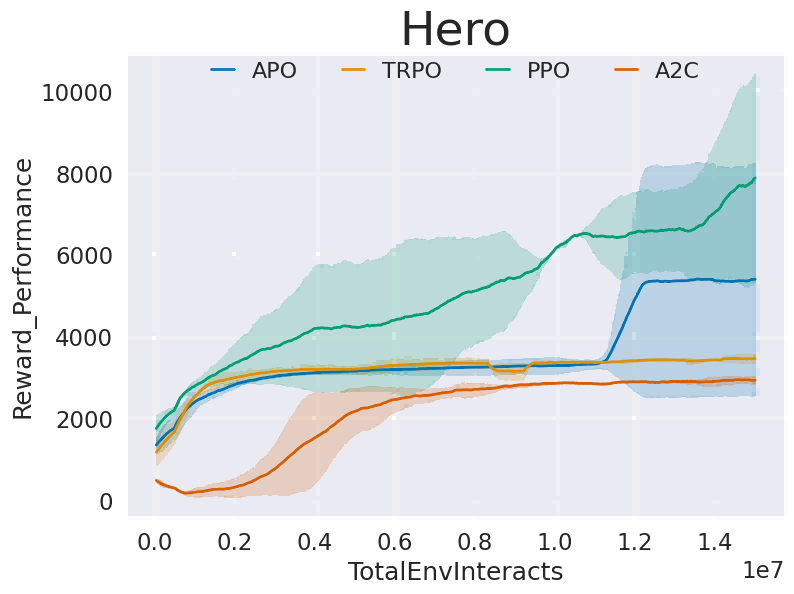}}
    \end{subfigure}
    \hfill
    \begin{subfigure}[t]{1.000\textwidth}
        \raisebox{-\height}{\includegraphics[width=\textwidth]{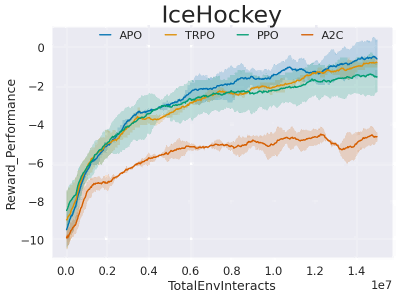}}
    \end{subfigure}
    \hfill
    \begin{subfigure}[t]{1.000\textwidth}
        \raisebox{-\height}{\includegraphics[width=\textwidth]{fig/atari/Jamesbond_Reward_Performance.png}}
    \end{subfigure}
    \end{subfigure}
    \caption{Comparison of expected performance on Atari Game No.1 - No. 32}
    \label{fig: all atari figures 1}
\end{figure}

\begin{figure}[t]
    \vspace{-40pt}
    \begin{subfigure}[t]{0.204\textwidth}
    \begin{subfigure}[t]{1.000\textwidth}
        \raisebox{-\height}{\includegraphics[width=\textwidth]{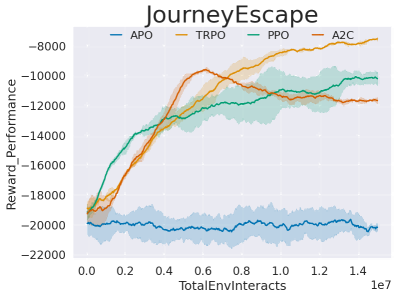}}
    \end{subfigure}
    \hfill
    \begin{subfigure}[t]{1.000\textwidth}
        \raisebox{-\height}{\includegraphics[width=\textwidth]{fig/atari/Kangaroo_Reward_Performance.png}}
    \end{subfigure}
    \hfill
    \begin{subfigure}[t]{1.000\textwidth}
        \raisebox{-\height}{\includegraphics[width=\textwidth]{fig/atari/Krull_Reward_Performance.png}}
    \end{subfigure}
    \hfill
    \begin{subfigure}[t]{1.000\textwidth}
        \raisebox{-\height}{\includegraphics[width=\textwidth]{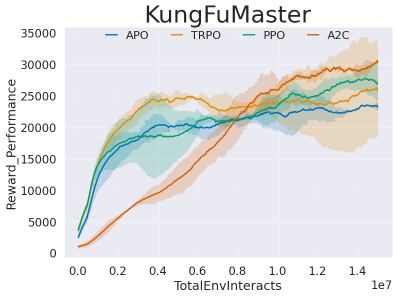}}
    \end{subfigure}
    \hfill
    \begin{subfigure}[t]{1.000\textwidth}
        \raisebox{-\height}{\includegraphics[width=\textwidth]{fig/atari/MontezumaRevenge_Reward_Performance.png}}
    \end{subfigure}
    \hfill
    \begin{subfigure}[t]{1.000\textwidth}
        \raisebox{-\height}{\includegraphics[width=\textwidth]{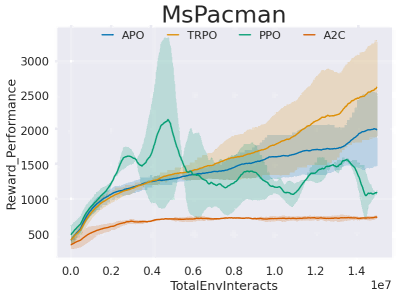}}
    \end{subfigure}
    \hfill
    \begin{subfigure}[t]{1.000\textwidth}
        \raisebox{-\height}{\includegraphics[width=\textwidth]{fig/atari/NameThisGame_Reward_Performance.png}}
    \end{subfigure}
    \hfill
    \begin{subfigure}[t]{1.000\textwidth}
        \raisebox{-\height}{\includegraphics[width=\textwidth]{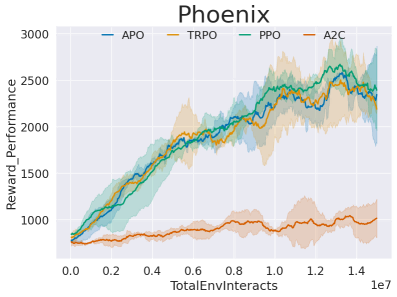}}
    \end{subfigure}
    \end{subfigure}
    \hfill
    \begin{subfigure}[t]{0.204\textwidth}
    \begin{subfigure}[t]{1.000\textwidth}
        \raisebox{-\height}{\includegraphics[width=\textwidth]{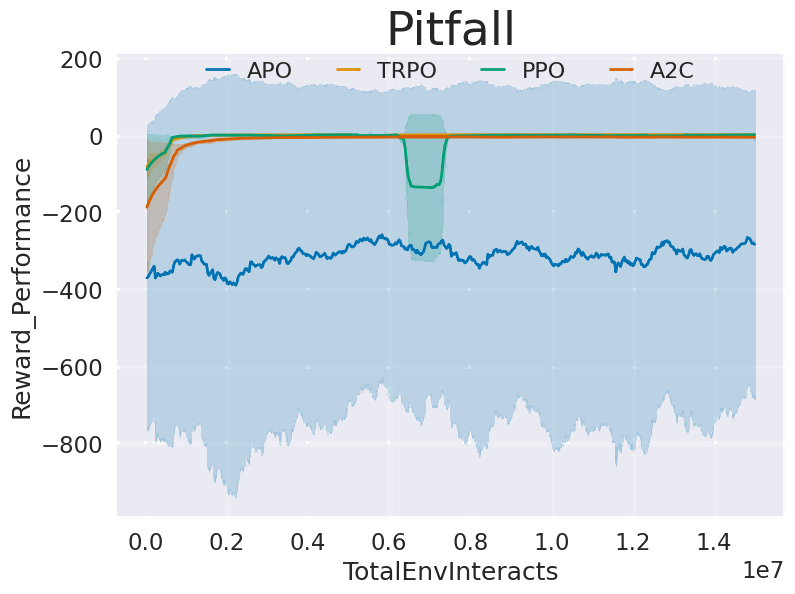}}
    \end{subfigure}
    \hfill
    \begin{subfigure}[t]{1.000\textwidth}
        \raisebox{-\height}{\includegraphics[width=\textwidth]{fig/atari/Pong_Reward_Performance.png}}
    \end{subfigure}
    \hfill
    \begin{subfigure}[t]{1.000\textwidth}
        \raisebox{-\height}{\includegraphics[width=\textwidth]{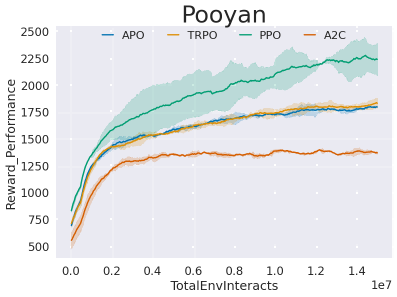}}
    \end{subfigure}
    \hfill
    \begin{subfigure}[t]{1.000\textwidth}
        \raisebox{-\height}{\includegraphics[width=\textwidth]{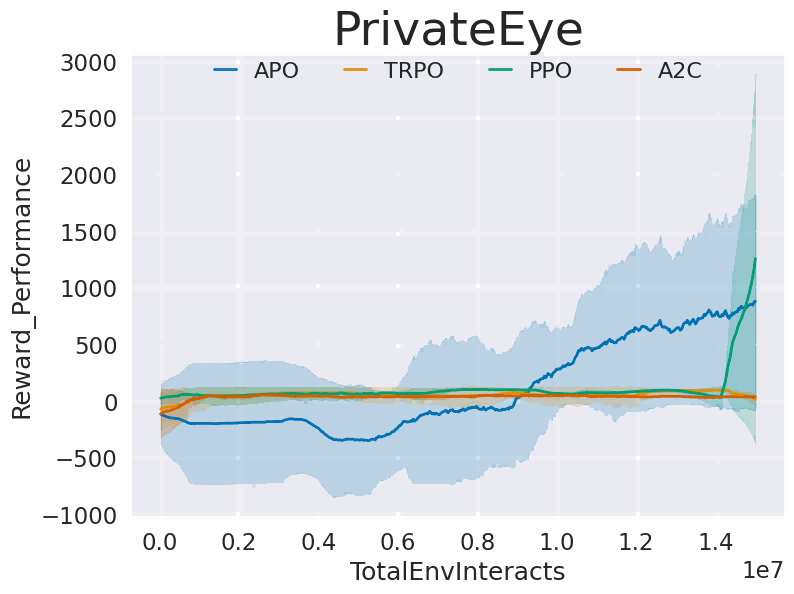}}
    \end{subfigure}
    \hfill
    \begin{subfigure}[t]{1.000\textwidth}
        \raisebox{-\height}{\includegraphics[width=\textwidth]{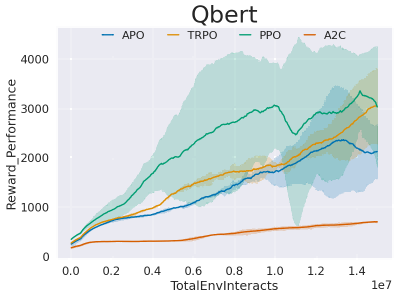}}
    \end{subfigure}
    \hfill
    \begin{subfigure}[t]{1.000\textwidth}
        \raisebox{-\height}{\includegraphics[width=\textwidth]{fig/atari/Riverraid_Reward_Performance.png}}
    \end{subfigure}
    \hfill
    \begin{subfigure}[t]{1.000\textwidth}
        \raisebox{-\height}{\includegraphics[width=\textwidth]{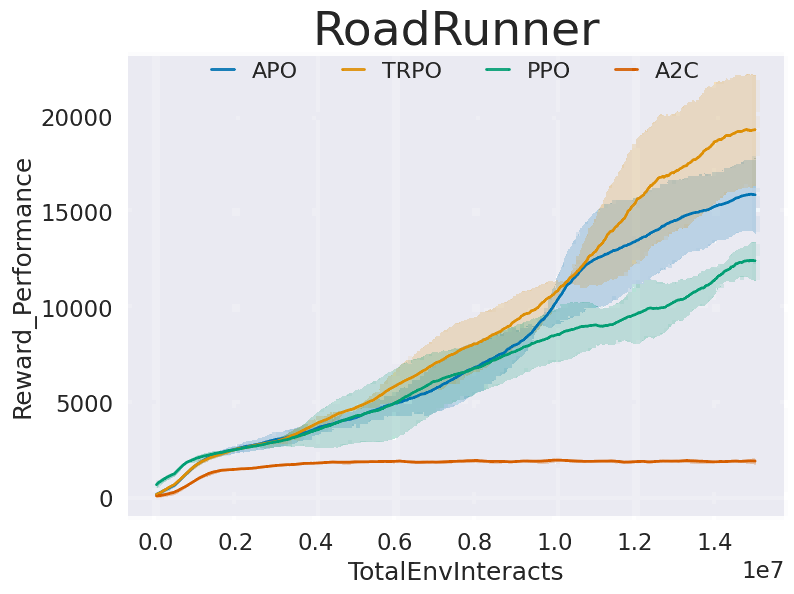}}
    \end{subfigure}
    \hfill
    \begin{subfigure}[t]{1.000\textwidth}
        \raisebox{-\height}{\includegraphics[width=\textwidth]{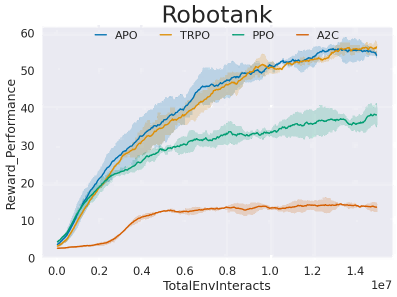}}
    \end{subfigure}
    \end{subfigure}
    \hfill
    \begin{subfigure}[t]{0.204\textwidth}
    \begin{subfigure}[t]{1.000\textwidth}
        \raisebox{-\height}{\includegraphics[width=\textwidth]{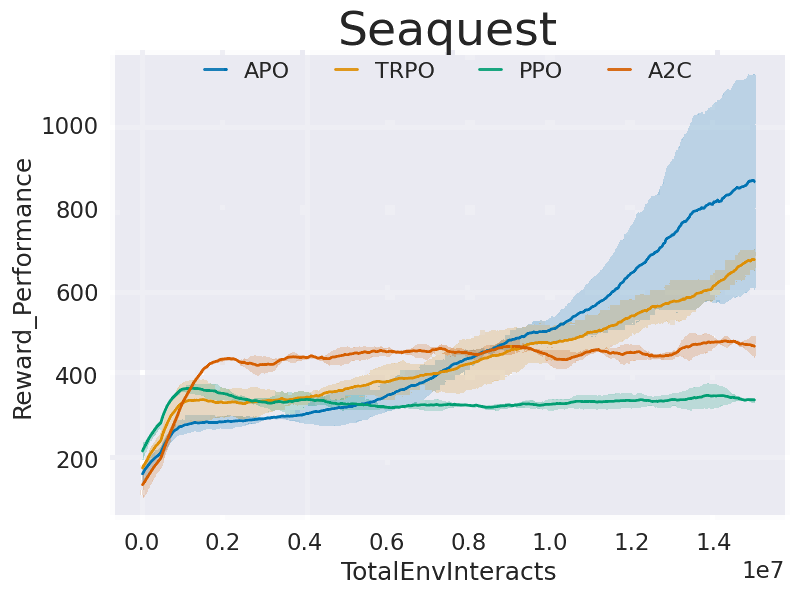}}
    \end{subfigure}
    \hfill
    \begin{subfigure}[t]{1.000\textwidth}
        \raisebox{-\height}{\includegraphics[width=\textwidth]{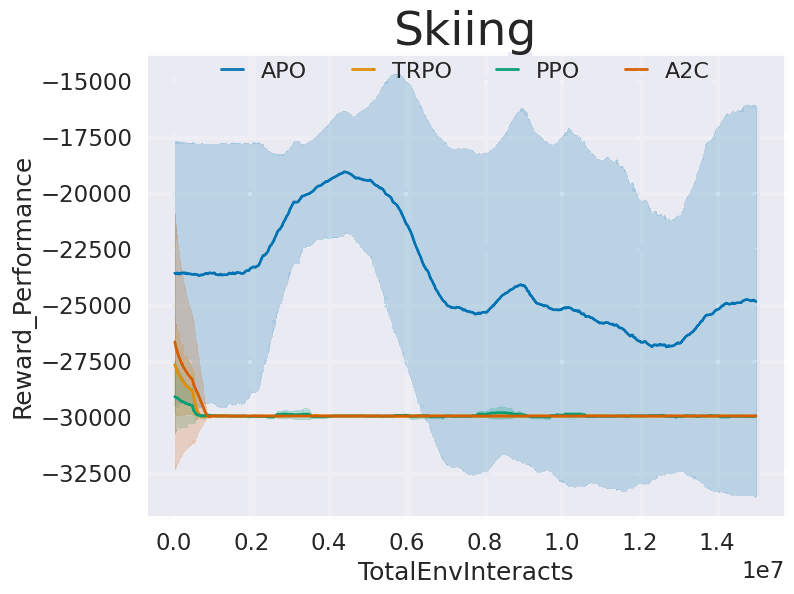}}
    \end{subfigure}
    \hfill
    \begin{subfigure}[t]{1.000\textwidth}
        \raisebox{-\height}{\includegraphics[width=\textwidth]{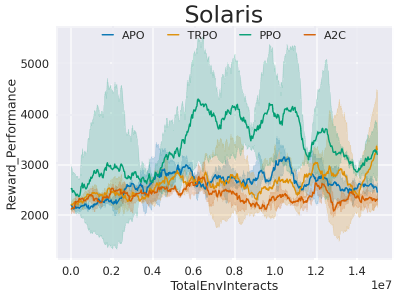}}
    \end{subfigure}
    \hfill
    \begin{subfigure}[t]{1.000\textwidth}
        \raisebox{-\height}{\includegraphics[width=\textwidth]{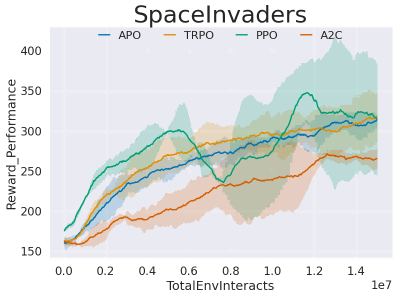}}
    \end{subfigure}
    \hfill
    \begin{subfigure}[t]{1.000\textwidth}
        \raisebox{-\height}{\includegraphics[width=\textwidth]{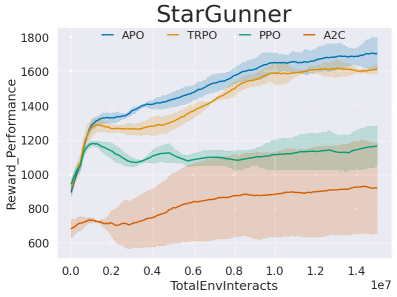}}
    \end{subfigure}
    \hfill
    \begin{subfigure}[t]{1.000\textwidth}
        \raisebox{-\height}{\includegraphics[width=\textwidth]{fig/atari/Tennis_Reward_Performance.png}}
    \end{subfigure}
    \hfill
    \begin{subfigure}[t]{1.000\textwidth}
        \raisebox{-\height}{\includegraphics[width=\textwidth]{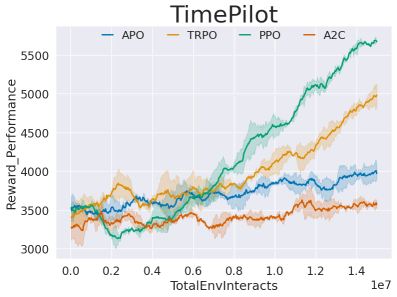}}
    \end{subfigure}
    
    \end{subfigure}
    \hfill
    \begin{subfigure}[t]{0.204\textwidth}
    \begin{subfigure}[t]{1.000\textwidth}
        \raisebox{-\height}{\includegraphics[width=\textwidth]{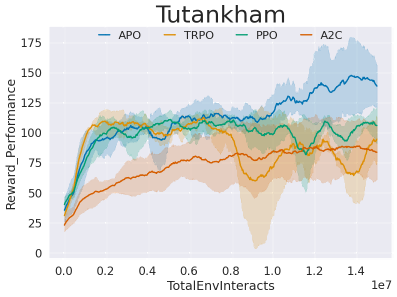}}
    \end{subfigure}
    \hfill
    \begin{subfigure}[t]{1.000\textwidth}
        \raisebox{-\height}{\includegraphics[width=\textwidth]{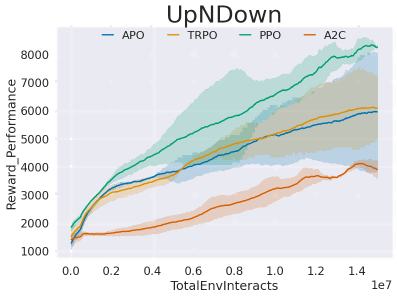}}
    \end{subfigure}
    \hfill
    \begin{subfigure}[t]{1.000\textwidth}
        \raisebox{-\height}{\includegraphics[width=\textwidth]{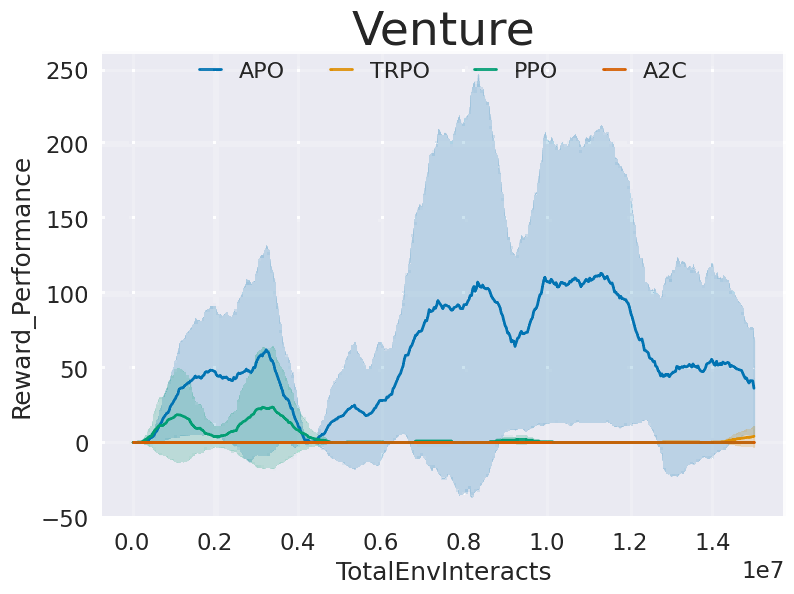}}
    \end{subfigure}
    \hfill
    \begin{subfigure}[t]{1.000\textwidth}
        \raisebox{-\height}{\includegraphics[width=\textwidth]{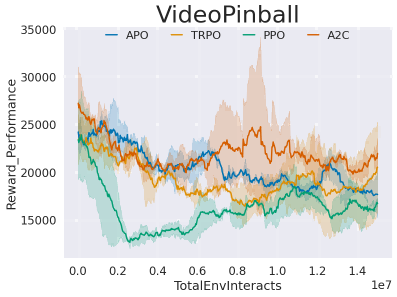}}
    \end{subfigure}
    \hfill
    \begin{subfigure}[t]{1.000\textwidth}
        \raisebox{-\height}{\includegraphics[width=\textwidth]{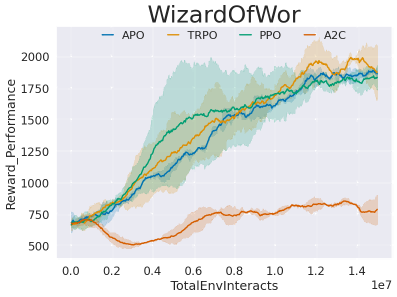}}
    \end{subfigure}
    \hfill
    \begin{subfigure}[t]{1.000\textwidth}
        \raisebox{-\height}{\includegraphics[width=\textwidth]{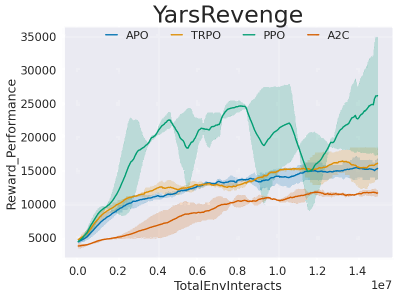}}
    \end{subfigure}
    \hfill
    \begin{subfigure}[t]{1.000\textwidth}
        \raisebox{-\height}{\includegraphics[width=\textwidth]{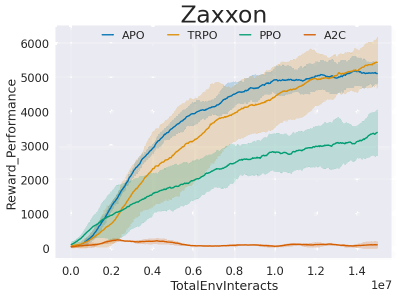}}
    \end{subfigure}
    \end{subfigure}
    \caption{Comparison of expected performance on Atari Game No.33 - No. 62}
    \label{fig: all atari figures 2}
\end{figure}

\begin{figure}[t]
    \vspace{-40pt}
    \begin{subfigure}[t]{0.204\textwidth}
    \begin{subfigure}[t]{1.000\textwidth}
        \raisebox{-\height}{\includegraphics[width=\textwidth]{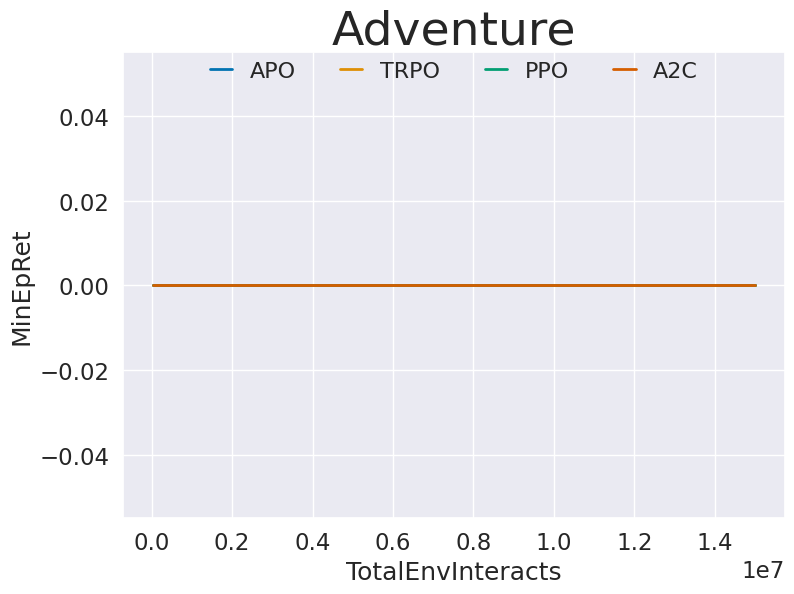}}
    \end{subfigure}
    \hfill
    \begin{subfigure}[t]{1.000\textwidth}
        \raisebox{-\height}{\includegraphics[width=\textwidth]{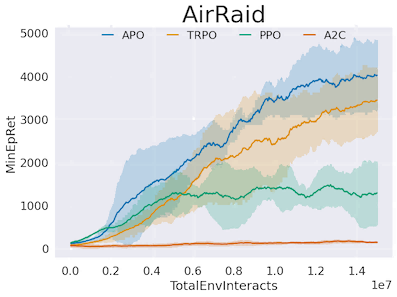}}
    \end{subfigure}
    \hfill
    \begin{subfigure}[t]{1.000\textwidth}
        \raisebox{-\height}{\includegraphics[width=\textwidth]{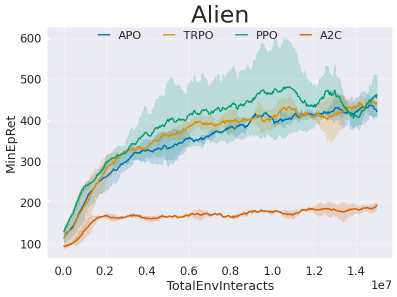}}
    \end{subfigure}
    \hfill
    \begin{subfigure}[t]{1.000\textwidth}
        \raisebox{-\height}{\includegraphics[width=\textwidth]{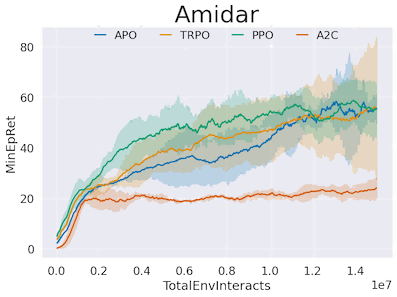}}
    \end{subfigure}
    \hfill
    \begin{subfigure}[t]{1.000\textwidth}
        \raisebox{-\height}{\includegraphics[width=\textwidth]{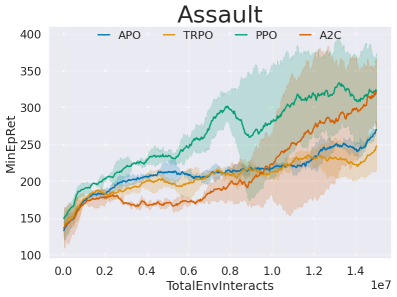}}
    \end{subfigure}
    \hfill
    \begin{subfigure}[t]{1.000\textwidth}
        \raisebox{-\height}{\includegraphics[width=\textwidth]{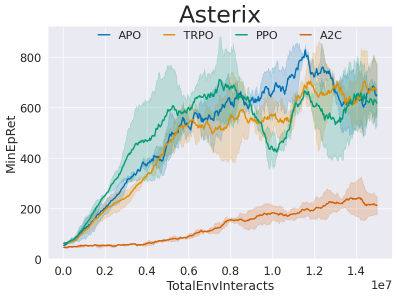}}
    \end{subfigure}
    \hfill
    \begin{subfigure}[t]{1.000\textwidth}
        \raisebox{-\height}{\includegraphics[width=\textwidth]{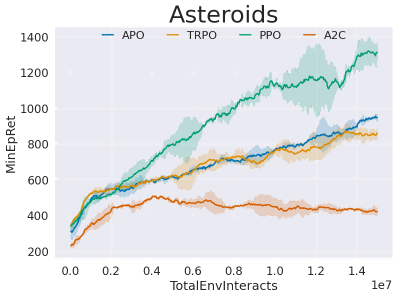}}
    \end{subfigure}
    \hfill
    \begin{subfigure}[t]{1.000\textwidth}
        \raisebox{-\height}{\includegraphics[width=\textwidth]{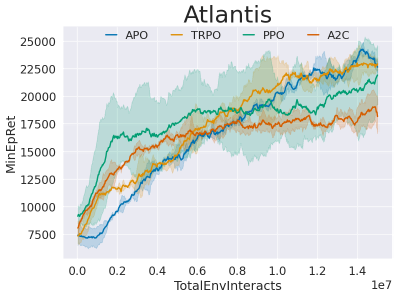}}
    \end{subfigure}
    \end{subfigure}
    \hfill
    \begin{subfigure}[t]{0.204\textwidth}
    \begin{subfigure}[t]{1.000\textwidth}
        \raisebox{-\height}{\includegraphics[width=\textwidth]{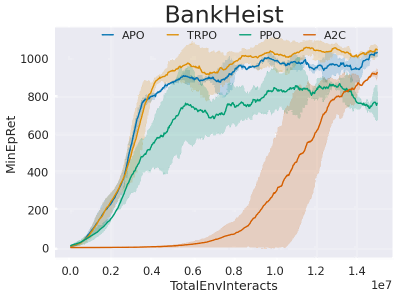}}
    \end{subfigure}
    \hfill
    \begin{subfigure}[t]{1.000\textwidth}
        \raisebox{-\height}{\includegraphics[width=\textwidth]{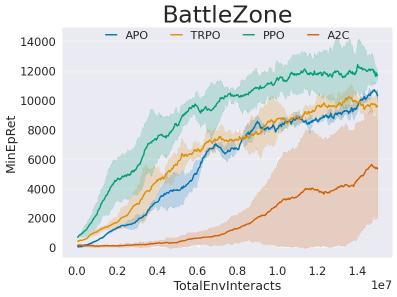}}
    \end{subfigure}
    \hfill
    \begin{subfigure}[t]{1.000\textwidth}
        \raisebox{-\height}{\includegraphics[width=\textwidth]{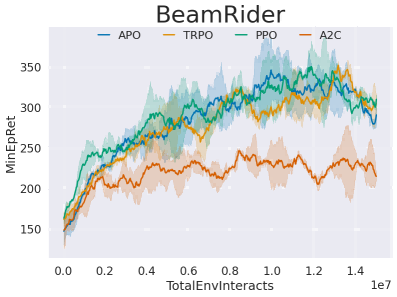}}
    \end{subfigure}
    \hfill
    \begin{subfigure}[t]{1.000\textwidth}
        \raisebox{-\height}{\includegraphics[width=\textwidth]{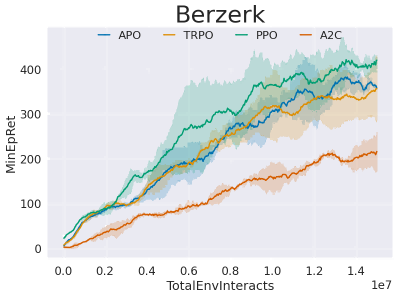}}
    \end{subfigure}
    \hfill
    \begin{subfigure}[t]{1.000\textwidth}
        \raisebox{-\height}{\includegraphics[width=\textwidth]{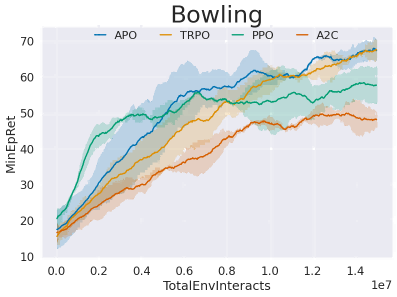}}
    \end{subfigure}
    \hfill
    \begin{subfigure}[t]{1.000\textwidth}
        \raisebox{-\height}{\includegraphics[width=\textwidth]{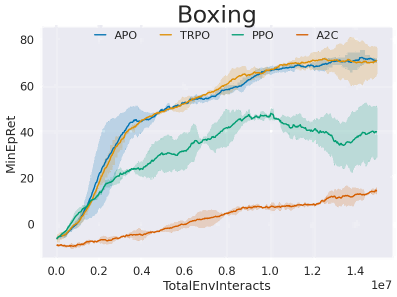}}
    \end{subfigure}
    \hfill
    \begin{subfigure}[t]
    {1.000\textwidth}
        \raisebox{-\height}{\includegraphics[width=\textwidth]{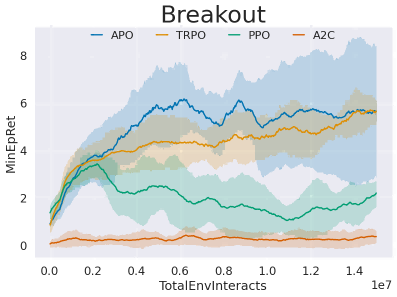}}
    \end{subfigure}
    \hfill
    \begin{subfigure}[t]{1.000\textwidth}
        \raisebox{-\height}{\includegraphics[width=\textwidth]{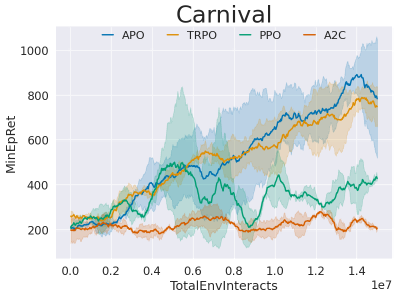}}
    \end{subfigure}
    \end{subfigure}
    \hfill
    \begin{subfigure}[t]{0.204\textwidth}
    \begin{subfigure}[t]{1.000\textwidth}
        \raisebox{-\height}{\includegraphics[width=\textwidth]{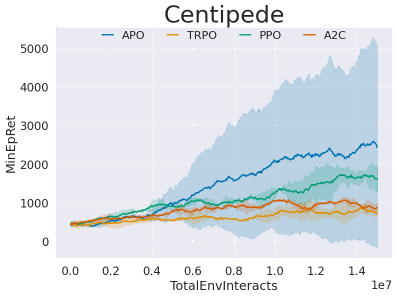}}
    \end{subfigure}
    \hfill
    \begin{subfigure}[t]{1.000\textwidth}
        \raisebox{-\height}{\includegraphics[width=\textwidth]{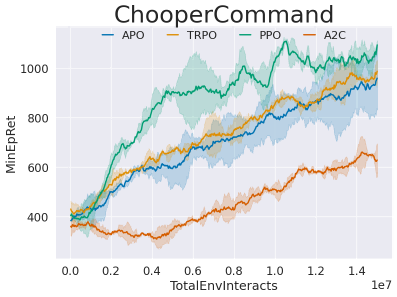}}
    \end{subfigure}
    \hfill
    \begin{subfigure}[t]{1.000\textwidth}
        \raisebox{-\height}{\includegraphics[width=\textwidth]{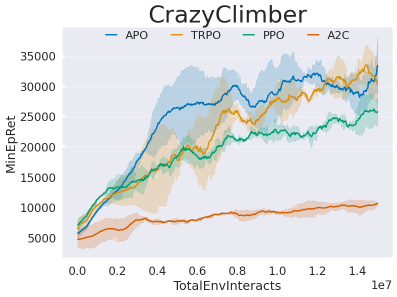}}
    \end{subfigure}
    \hfill
    \begin{subfigure}[t]{1.000\textwidth}
        \raisebox{-\height}{\includegraphics[width=\textwidth]{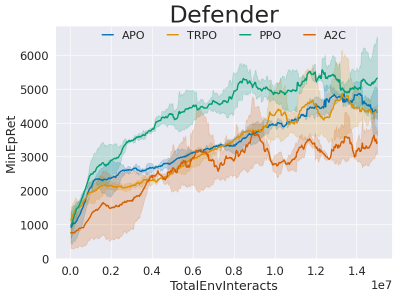}}
    \end{subfigure}
    \hfill
    \begin{subfigure}[t]{1.000\textwidth}
        \raisebox{-\height}{\includegraphics[width=\textwidth]{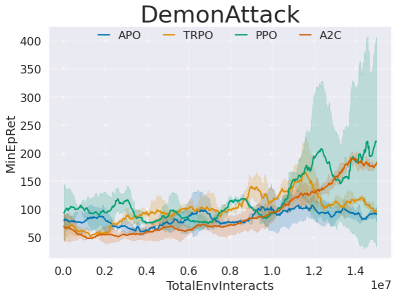}}
    \end{subfigure}
    \hfill
    \begin{subfigure}[t]{1.000\textwidth}
        \raisebox{-\height}{\includegraphics[width=\textwidth]{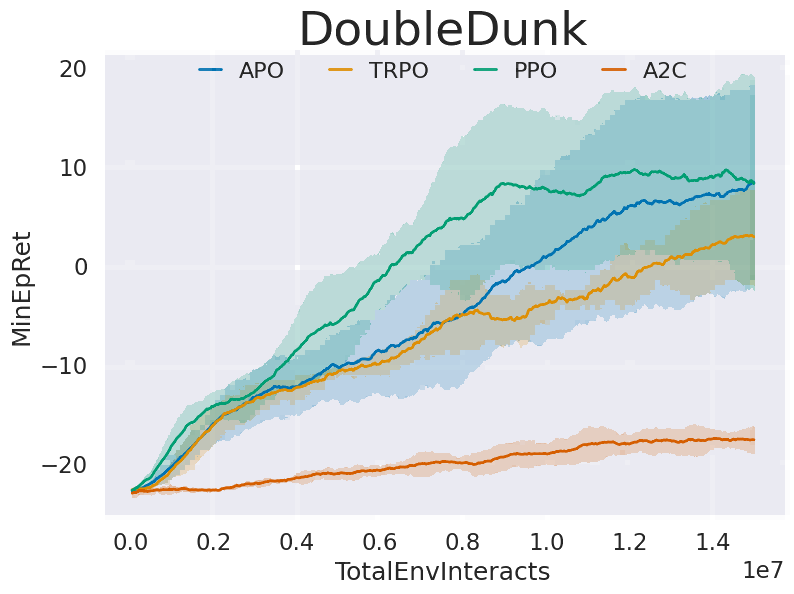}}
    \end{subfigure}
    \hfill
    \begin{subfigure}[t]{1.000\textwidth}
        \raisebox{-\height}{\includegraphics[width=\textwidth]{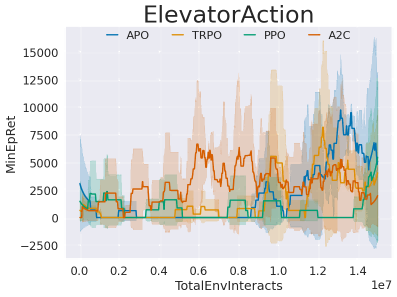}}
    \end{subfigure}
    \hfill
    \begin{subfigure}[t]{1.000\textwidth}
        \raisebox{-\height}{\includegraphics[width=\textwidth]{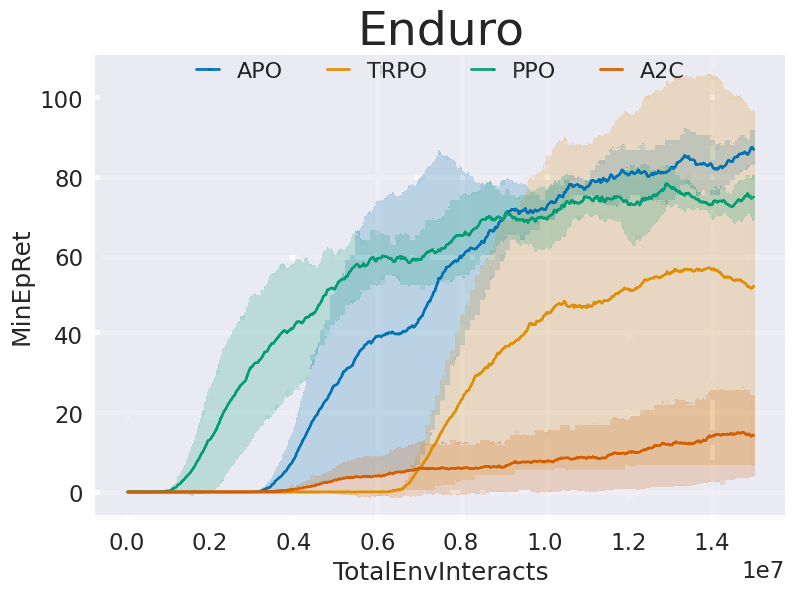}}
    \end{subfigure}
    \end{subfigure}
    \hfill
    \begin{subfigure}[t]{0.204\textwidth}
    \begin{subfigure}[t]{1.000\textwidth}
        \raisebox{-\height}{\includegraphics[width=\textwidth]{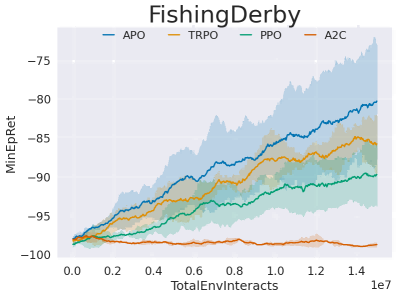}}
    \end{subfigure}
    \hfill
    \begin{subfigure}[t]{1.000\textwidth}
        \raisebox{-\height}{\includegraphics[width=\textwidth]{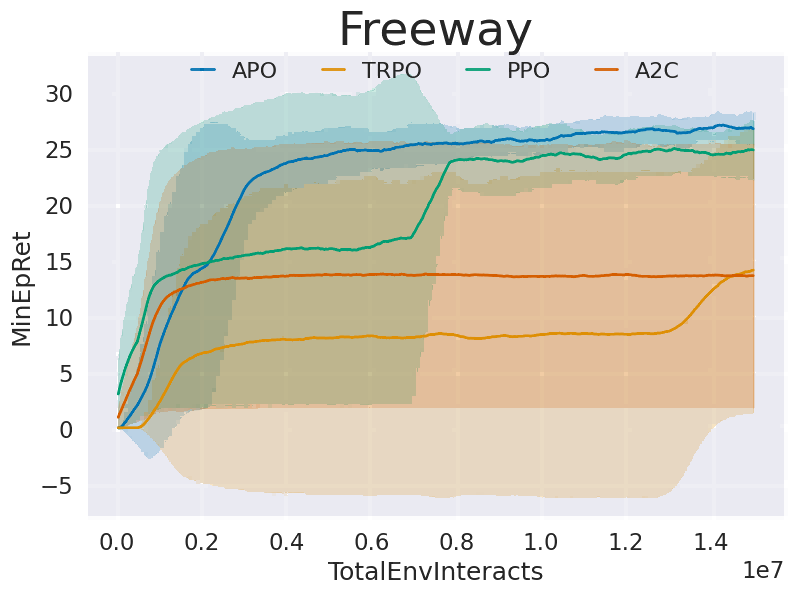}}
    \end{subfigure}
    \hfill
    \begin{subfigure}[t]{1.000\textwidth}
        \raisebox{-\height}{\includegraphics[width=\textwidth]{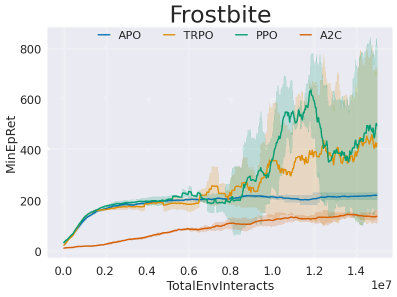}}
    \end{subfigure}
    \hfill
    \begin{subfigure}[t]{1.000\textwidth}
        \raisebox{-\height}{\includegraphics[width=\textwidth]{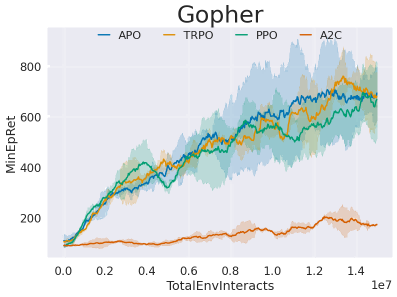}}
    \end{subfigure}
    \hfill
    \begin{subfigure}[t]{1.000\textwidth}
        \raisebox{-\height}{\includegraphics[width=\textwidth]{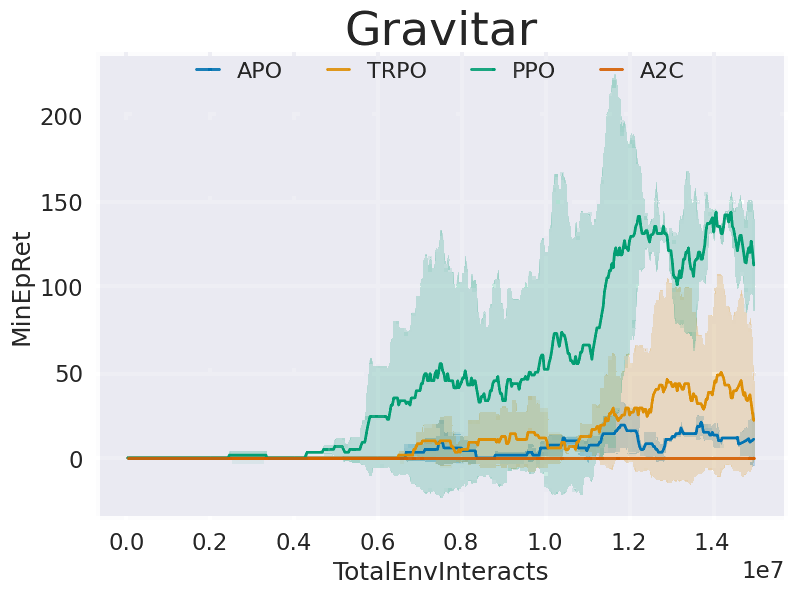}}
    \end{subfigure}
    \hfill
    \begin{subfigure}[t]{1.000\textwidth}
        \raisebox{-\height}{\includegraphics[width=\textwidth]{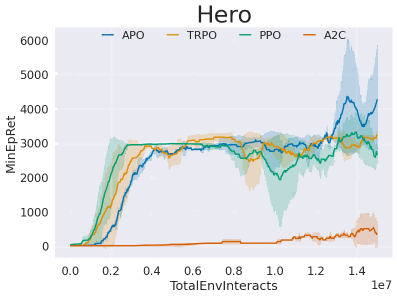}}
    \end{subfigure}
    \hfill
    \begin{subfigure}[t]{1.000\textwidth}
        \raisebox{-\height}{\includegraphics[width=\textwidth]{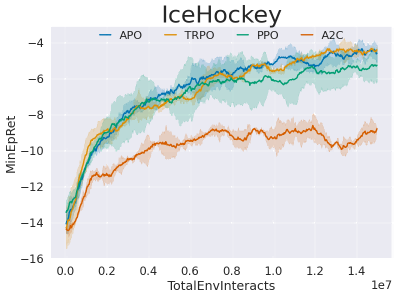}}
    \end{subfigure}
    \hfill
    \begin{subfigure}[t]{1.000\textwidth}
        \raisebox{-\height}{\includegraphics[width=\textwidth]{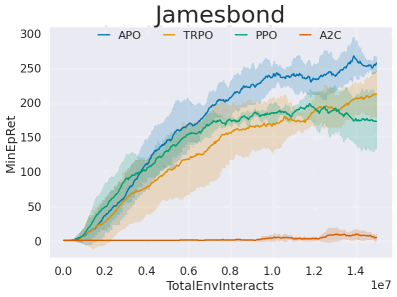}}
    \end{subfigure}
    \end{subfigure}
    \caption{Comparison of worst performance on Atari Game No.1 - No. 32}
    \label{fig: all min atari figures 1}
\end{figure}

\begin{figure}[t]
    \vspace{-40pt}
    \begin{subfigure}[t]{0.204\textwidth}
    \begin{subfigure}[t]{1.000\textwidth}
        \raisebox{-\height}{\includegraphics[width=\textwidth]{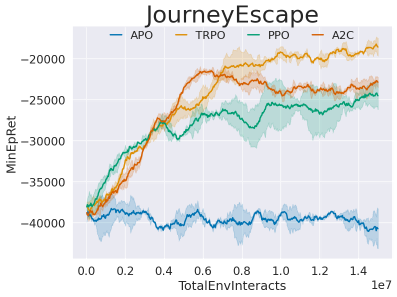}}
    \end{subfigure}
    \hfill
    \begin{subfigure}[t]{1.000\textwidth}
        \raisebox{-\height}{\includegraphics[width=\textwidth]{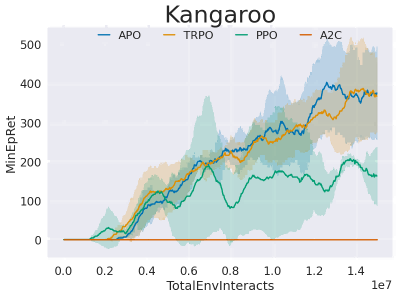}}
    \end{subfigure}
    \hfill
    \begin{subfigure}[t]{1.000\textwidth}
        \raisebox{-\height}{\includegraphics[width=\textwidth]{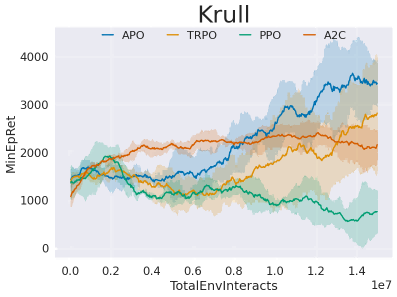}}
    \end{subfigure}
    \hfill
    \begin{subfigure}[t]{1.000\textwidth}
        \raisebox{-\height}{\includegraphics[width=\textwidth]{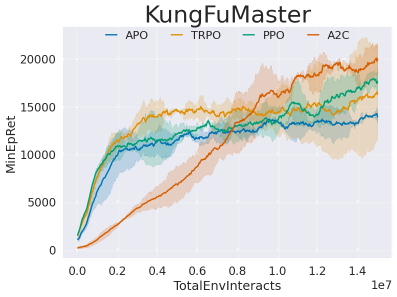}}
    \end{subfigure}
    \hfill
    \begin{subfigure}[t]{1.000\textwidth}
        \raisebox{-\height}{\includegraphics[width=\textwidth]{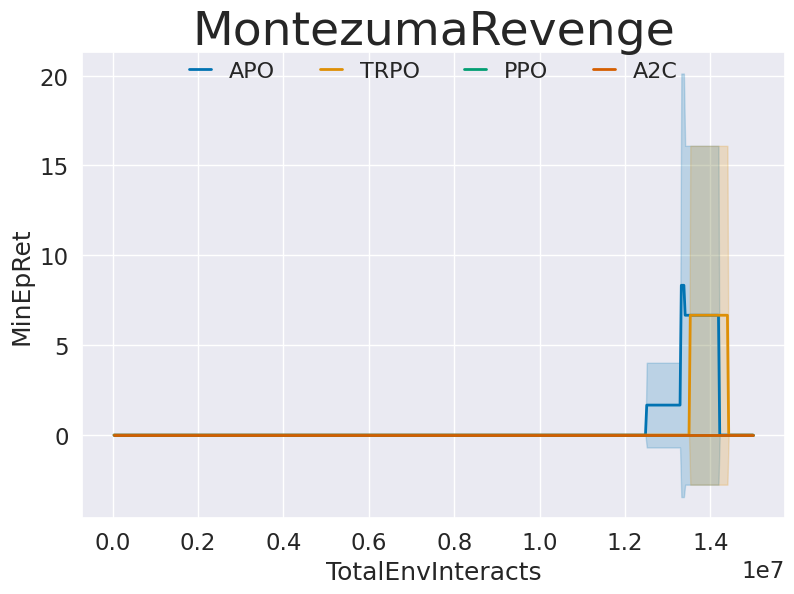}}
    \end{subfigure}
    \hfill
    \begin{subfigure}[t]{1.000\textwidth}
        \raisebox{-\height}{\includegraphics[width=\textwidth]{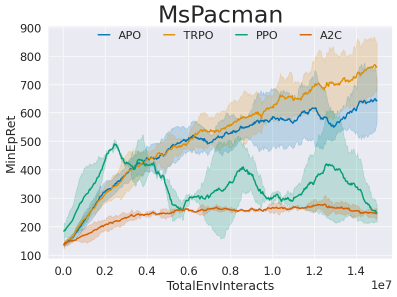}}
    \end{subfigure}
    \hfill
    \begin{subfigure}[t]{1.000\textwidth}
        \raisebox{-\height}{\includegraphics[width=\textwidth]{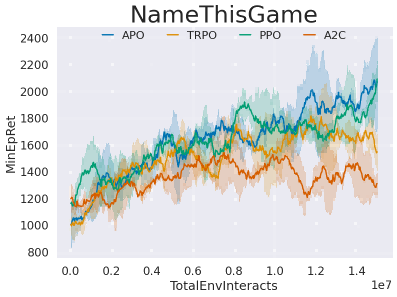}}
    \end{subfigure}
    \hfill
    \begin{subfigure}[t]{1.000\textwidth}
        \raisebox{-\height}{\includegraphics[width=\textwidth]{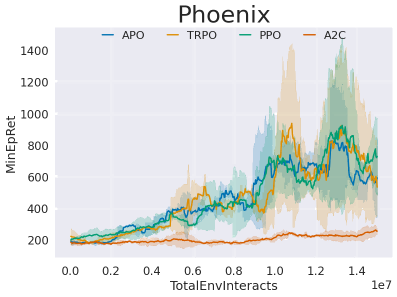}}
    \end{subfigure}
    \end{subfigure}
    \hfill
    \begin{subfigure}[t]{0.204\textwidth}
    \begin{subfigure}[t]{1.000\textwidth}
        \raisebox{-\height}{\includegraphics[width=\textwidth]{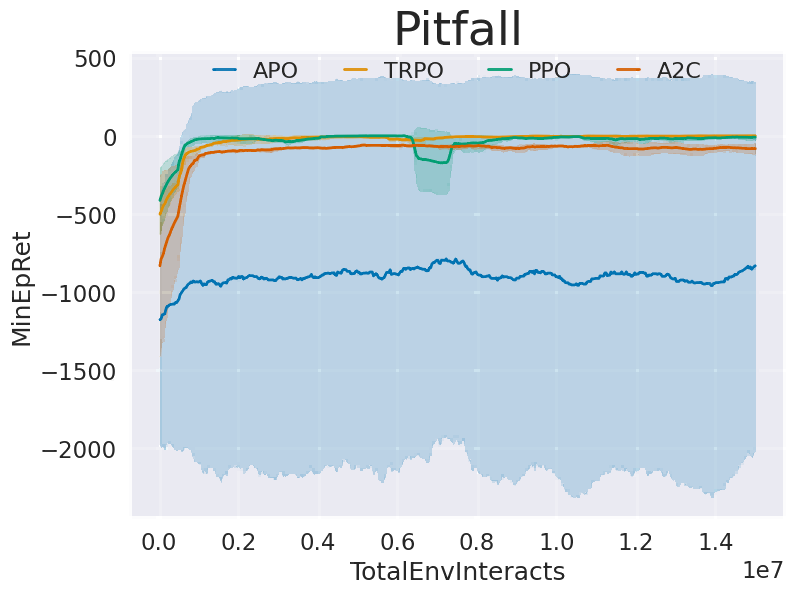}}
    \end{subfigure}
    \hfill
    \begin{subfigure}[t]{1.000\textwidth}
        \raisebox{-\height}{\includegraphics[width=\textwidth]{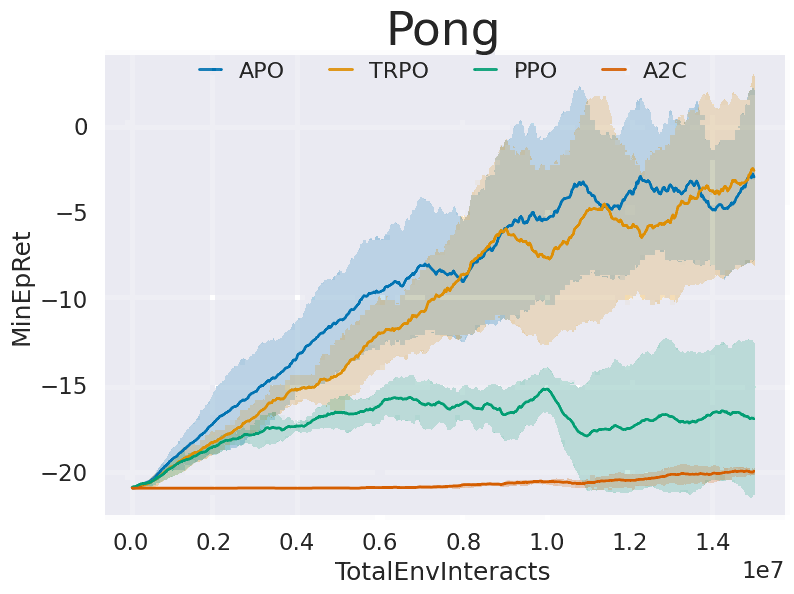}}
    \end{subfigure}
    \hfill
    \begin{subfigure}[t]{1.000\textwidth}
        \raisebox{-\height}{\includegraphics[width=\textwidth]{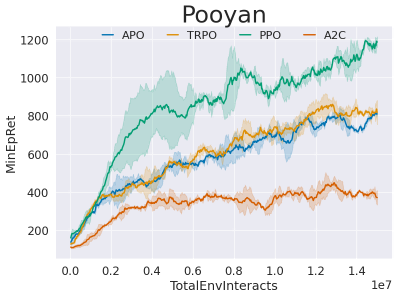}}
    \end{subfigure}
    \hfill
    \begin{subfigure}[t]{1.000\textwidth}
        \raisebox{-\height}{\includegraphics[width=\textwidth]{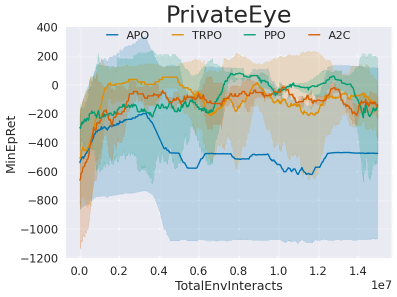}}
    \end{subfigure}
    \hfill
    \begin{subfigure}[t]{1.000\textwidth}
        \raisebox{-\height}{\includegraphics[width=\textwidth]{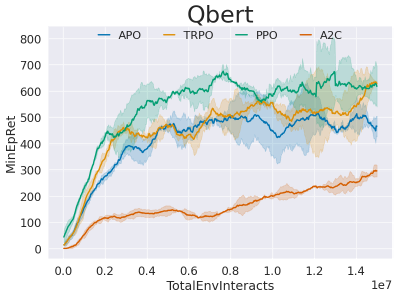}}
    \end{subfigure}
    \hfill
    \begin{subfigure}[t]{1.000\textwidth}
        \raisebox{-\height}{\includegraphics[width=\textwidth]{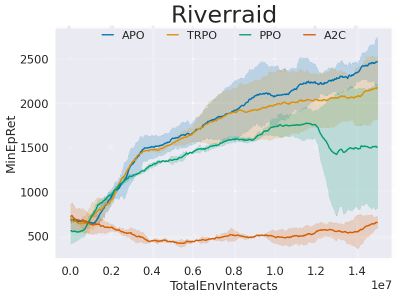}}
    \end{subfigure}
    \hfill
    \begin{subfigure}[t]{1.000\textwidth}
        \raisebox{-\height}{\includegraphics[width=\textwidth]{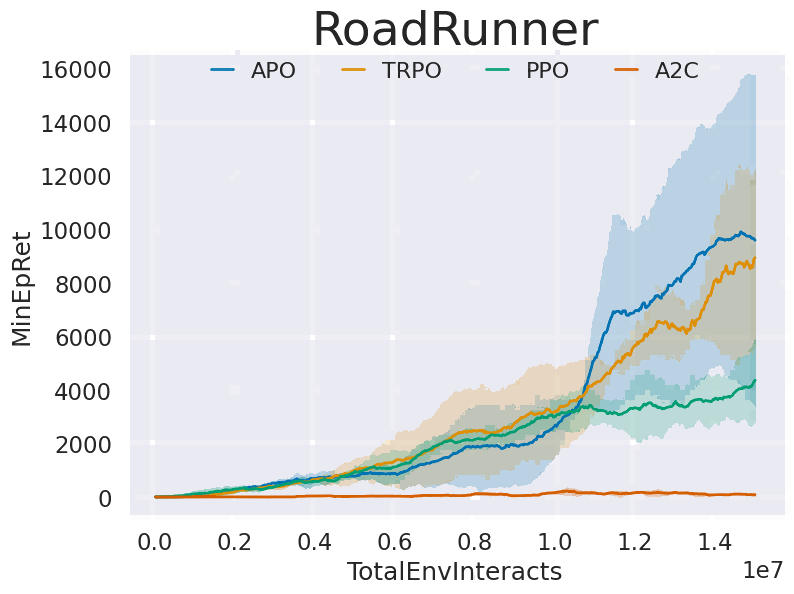}}
    \end{subfigure}
    \hfill
    \begin{subfigure}[t]{1.000\textwidth}
        \raisebox{-\height}{\includegraphics[width=\textwidth]{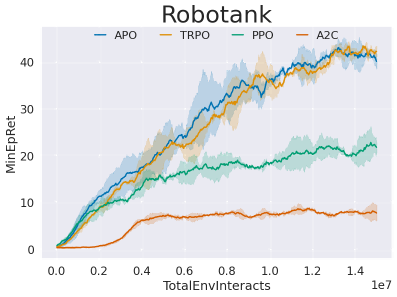}}
    \end{subfigure}
    \end{subfigure}
    \hfill
    \begin{subfigure}[t]{0.204\textwidth}
    \begin{subfigure}[t]{1.000\textwidth}
        \raisebox{-\height}{\includegraphics[width=\textwidth]{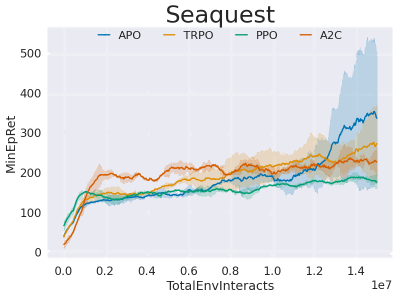}}
    \end{subfigure}
    \hfill
    \begin{subfigure}[t]{1.000\textwidth}
        \raisebox{-\height}{\includegraphics[width=\textwidth]{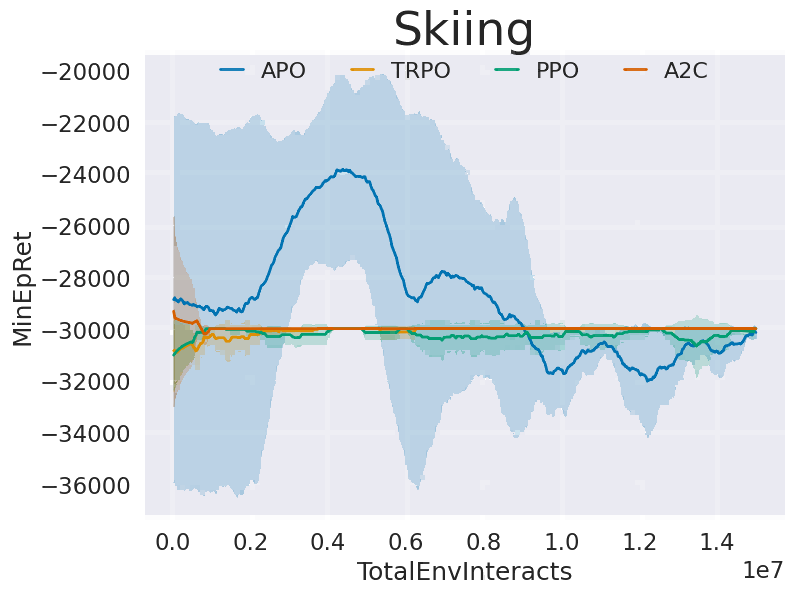}}
    \end{subfigure}
    \hfill
    \begin{subfigure}[t]{1.000\textwidth}
        \raisebox{-\height}{\includegraphics[width=\textwidth]{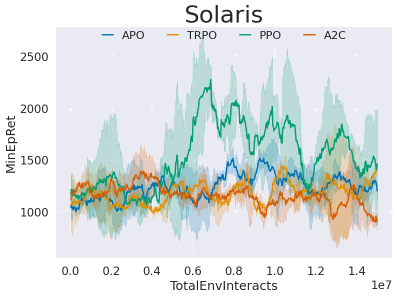}}
    \end{subfigure}
    \hfill
    \begin{subfigure}[t]{1.000\textwidth}
        \raisebox{-\height}{\includegraphics[width=\textwidth]{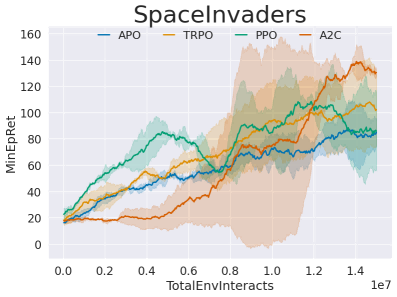}}
    \end{subfigure}
    \hfill
    \begin{subfigure}[t]{1.000\textwidth}
        \raisebox{-\height}{\includegraphics[width=\textwidth]{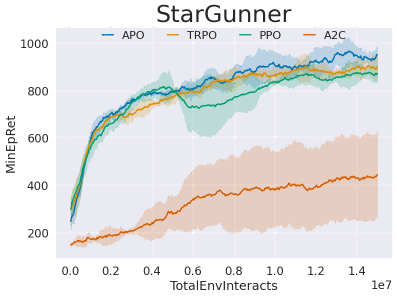}}
    \end{subfigure}
    \hfill
    \begin{subfigure}[t]{1.000\textwidth}
        \raisebox{-\height}{\includegraphics[width=\textwidth]{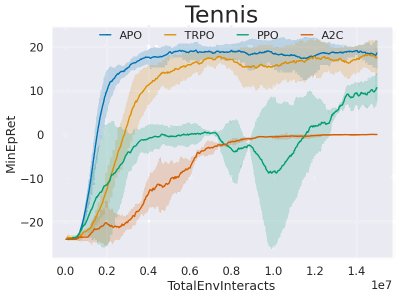}}
    \end{subfigure}
    \hfill
    \begin{subfigure}[t]{1.000\textwidth}
        \raisebox{-\height}{\includegraphics[width=\textwidth]{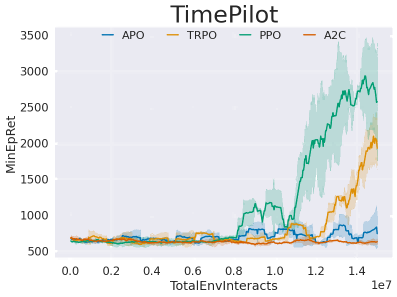}}
    \end{subfigure}
    
    \end{subfigure}
    \hfill
    \begin{subfigure}[t]{0.204\textwidth}
    \begin{subfigure}[t]{1.000\textwidth}
        \raisebox{-\height}{\includegraphics[width=\textwidth]{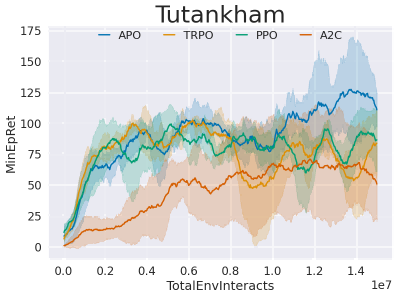}}
    \end{subfigure}
    \hfill
    \begin{subfigure}[t]{1.000\textwidth}
        \raisebox{-\height}{\includegraphics[width=\textwidth]{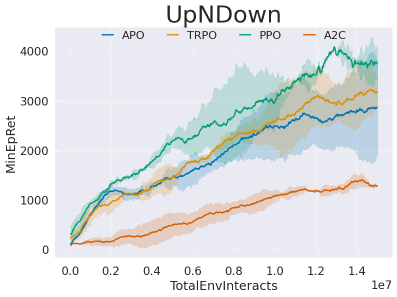}}
    \end{subfigure}
    \hfill
    \begin{subfigure}[t]{1.000\textwidth}
        \raisebox{-\height}{\includegraphics[width=\textwidth]{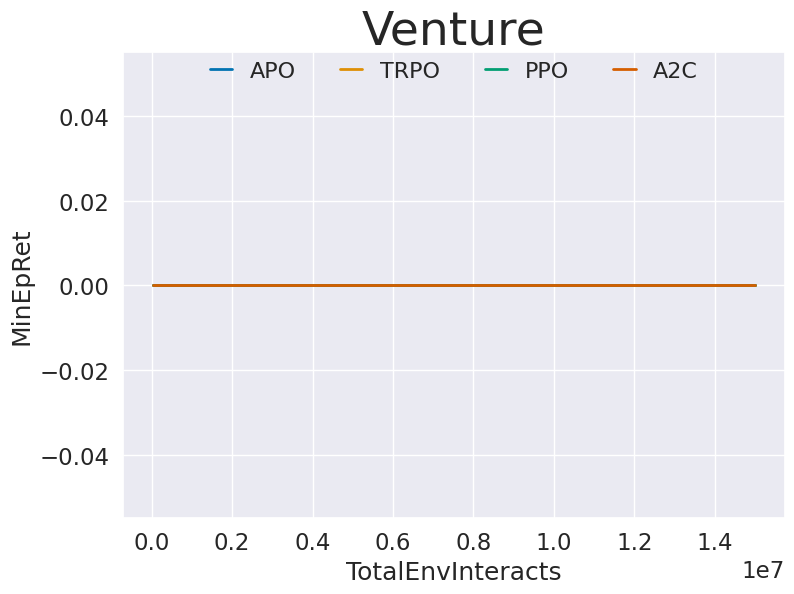}}
    \end{subfigure}
    \hfill
    \begin{subfigure}[t]{1.000\textwidth}
        \raisebox{-\height}{\includegraphics[width=\textwidth]{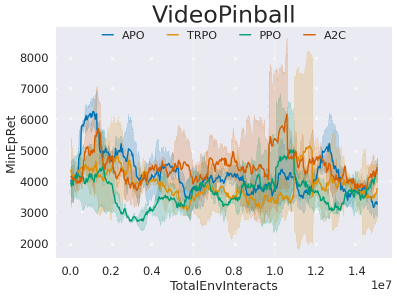}}
    \end{subfigure}
    \hfill
    \begin{subfigure}[t]{1.000\textwidth}
        \raisebox{-\height}{\includegraphics[width=\textwidth]{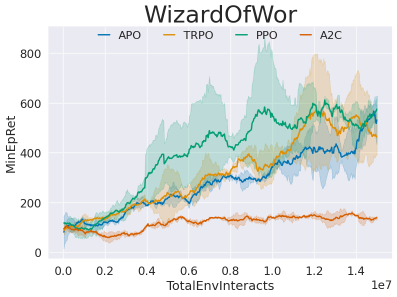}}
    \end{subfigure}
    \hfill
    \begin{subfigure}[t]{1.000\textwidth}
        \raisebox{-\height}{\includegraphics[width=\textwidth]{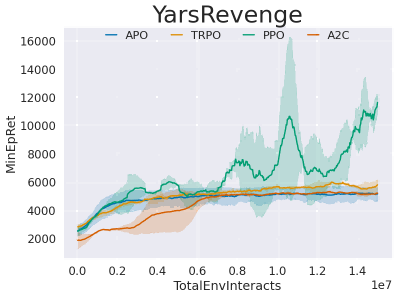}}
    \end{subfigure}
    \hfill
    \begin{subfigure}[t]{1.000\textwidth}
        \raisebox{-\height}{\includegraphics[width=\textwidth]{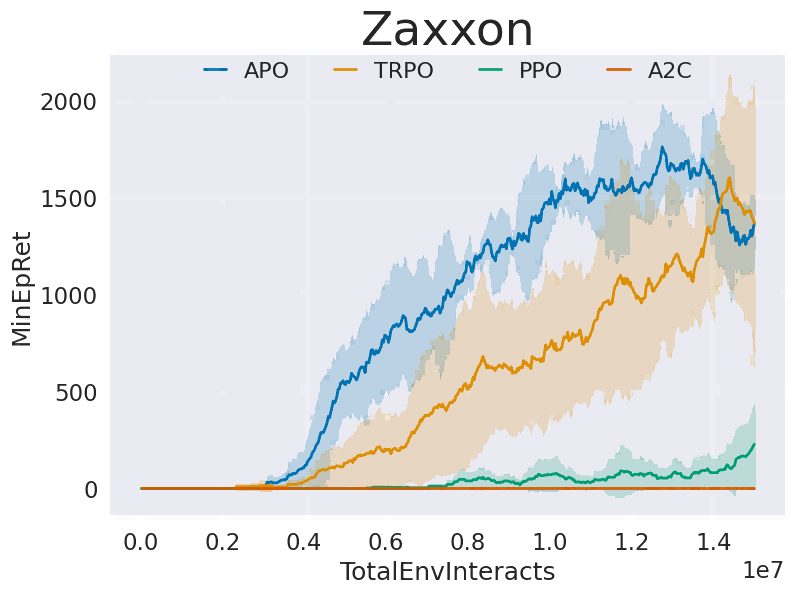}}
    \end{subfigure}
    \end{subfigure}
    \caption{Comparison of worst performance on Atari Game No.33 - No. 62}
    \label{fig: all min atari figures 2}
\end{figure}

\clearpage

\end{document}